\documentclass[11pt]{article}

\usepackage[utf8]{inputenc}
\usepackage[T1]{fontenc}
\usepackage{lmodern}
\usepackage{microtype}
\usepackage{setspace}
\usepackage{xcolor}

\newcommand{\op}{\mathrm{op}}      
\newcommand{\R}{\mathbb{R}}

\usepackage[margin=1in]{geometry}
\usepackage{amsmath,amssymb,amsthm,mathtools,bm}
\usepackage{bbm}
\usepackage{enumitem}
\usepackage{hyperref}
\usepackage{comment}
\usepackage{algorithm}
\usepackage{algpseudocode}
\algrenewcommand\algorithmicensure{\textbf{Output:}}

\newcommand{\sign}{\operatorname{sign}}
\newcommand{\One}{\mathbbm{1}}
\usepackage{thmtools}
\usepackage{booktabs}
\usepackage{natbib}
\usepackage{bbm}

\hypersetup{
  colorlinks=true,
  linkcolor=blue!60!black,
  citecolor=blue!60!black,
  urlcolor=blue!60!black
}

\newcommand{\PP}{\mathbb{P}}
\newcommand{\E}{\mathbb{E}}
\providecommand{\cL}{\mathcal{L}}
\newcommand{\Var}{\mathrm{Var}}

\newcommand{\logit}{\mathrm{logit}}
\newcommand{\Logistic}{\mathrm{logistic}}

\newtheorem{theorem}{Theorem}
\newtheorem{lemma}{Lemma}
\newtheorem{assumption}{Assumption}
\newtheorem{proposition}{Proposition}
\newtheorem{corollary}{Corollary}

\title{Calibrating an Imperfect Auxiliary Predictor for Unobserved No-Purchase Choice}
\author{Jiangkai Xiong$^\dagger$, \  Kalyan Talluri$^\ddagger$, \ Hanzhao Wang$^\diamond$}
\date{\small
$^\dagger$
Peking University,\\
$^\ddagger$
Imperial College Business School,\\
$^\diamond$ The University of Sydney Business School
}

\begin{document}

\onehalfspacing
\maketitle
\begingroup
\renewcommand{\thefootnote}{}
\footnotetext{Correspondence to xiongjk15@stu.pku.edu.cn}
\endgroup

\begin{abstract}
Firms typically are unable to observe and record some consumer actions---whether they purchase from a competitor or decide not to purchase, or even whether they considered the firm's product for purchase.   This missing data on consideration and the choice of the outside-option makes  identification and estimation of market size and choice parameters, even under simple models such as the Multinomial Logit (MNL), very challenging, and is one of the central problems in the estimation of choice models in an observed-data setting.  A number of methods have been proposed recently to estimate the choice model with unobserved no-purchases, often making use of auxiliary, related, market-share, or aggregated data.  We study a setting in which such  a black-box auxiliary predictor returns an estimate for the outside-option probability but may be biased or miscalibrated for various reasons:  it might have been trained on a different channel/market; or on market-share data from previous periods that have lost relevance; or by an externally trained machine-learning system or a foundational AI model.   In this paper we leverage a structural identity that links the outside-option log-odds to the utility of observed choices and  develop calibration methods that convert such imperfect auxiliary predictions into statistically valid estimates of unobserved no-purchase probabilities under two scenarios:   First, under an affine miscalibration model in logit space, we show that a simple regression  identifies the outside-option utility parameters and yields consistent recovery of no-purchase probabilities (without the help of any new data).  Second,  assuming the auxiliary method approximately preserves the order, we give a  rank-based calibration approach and  finite-sample error bounds that separate auxiliary-predictor quality from the  utility learning error based on the items in the choice set.  Numerical experiments illustrate the gains for no-purchase estimation and downstream assortment decisions, and we discuss extensions to aggregating multiple auxiliary predictors.
\end{abstract}

\section{Introduction}
Discrete-choice models are widely used in operations research, transportation, economics, marketing and management science for analyzing and predicting customer behavior.  Complicating their estimation and usage in practice is the fact that in many industries such as retail, hospitality, transportation and e-commerce, firms observe only purchases and not the events where customers arrive and choose not to buy (either forgoing the purchase entirely or buying from a competitor). This missing ``no-purchase''  or outside-option data makes the estimation of its weight and the market-size difficult for even simple models such as the Multinomial Logit (MNL): a low observed sales volume may reflect either weak appeal among many potential customers or just a small market size.  Overcoming this identification issue is critical in practice and has motivated a considerable body of research recently that tries to infer lost sales or unobserved choice probabilities from sales data \citep{talluri_vr_2004_ms,vanryzin_vulcano_2015_mnsc,subramanian_harsha_2021_mnsc,li_talluri_tekin_2025_msom}.

Almost all these estimation ideas, however, rely on some external data or market-share signals, such as aggregated marginal competitor information or  market-share data, or data from some other market with better observations on no-purchases, or managerial inputs and manual calibration.   Such external aggregated data, however, usually has a time-lag---market-share data for instance is usually provided by information brokers quarterly or even yearly (for instance market-share data at an airport  by origin-destination);  the competitor aggregated sales data used in \cite{li_talluri_tekin_2025_msom} is supplied with a month lag.   Other sources of data have similar relevancy concerns:  recent advances in machine learning and generative AI have made it  possible to obtain  predictions for unobserved outcomes (e.g., no-purchase events or competitors' market shares) where data from a large number of external sources is used \citep{gallegos2024survey,chen2025_msom_bias,park2023generative}. 
As a result, an estimator trained or calibrated on lagged aggregates or different data sets can create a bias in estimations of the outside-option probability even if they preserve  useful structure (e.g., correlation or approximate ranking agreement with the truth).  In this paper we assume the availability of such an \emph{auxiliary predictor} for the outside option, which may be informative yet miscalibrated in the focal environment (the context of the current period where decisions are made) and give methods for correcting certain types of bias that may be present in them.


Our proposed approach is based on a structural identity in choice models that links the outside-option log-odds to the utility of observed choices. This identity implies that if the auxiliary predictor is miscalibrated by a shift-and-scale transformation in logit space across assortments and covariates, then a simple regression on the observed information identifies the unobserved outside-utility coefficients and yields consistent recovery of outside-option probabilities. We then extend our approach beyond affine miscalibration to \emph{nearly monotone} predictors: when the auxiliary predictor preserves the ordering of outside-option logits for most samples, we can estimate the outside-option utility parameters via a maximum rank correlation (MRC) criterion \citep{han1987mrc}. We provide finite-sample error bounds for both situations, separating (i) the quality of the auxiliary predictor from (ii) the error from learning inside utilities of observed choices using purchase-only data.

In broad terms, the method requires:
\begin{itemize}
\item \textbf{Learnable utilities from purchase-only data:} there must be within-assortment variation (and implicitly, assortments of size greater than one) so that utilities of the observable choices in the assortment can be learned from the conditional likelihood.
\item \textbf{Non-trivial outside features:} the outside-option utility must depend on some observed covariates with sufficient variation. 
\item \textbf{An informative auxiliary predictor:} the auxiliary predictor must be aligned with the true outside logit---either via an affine relation in logit space (Section~\ref{sec:lin_pred}) or, more generally, by being approximately order-preserving on most samples (Section~\ref{sec:nmbs}).
\end{itemize}

\paragraph{Contributions.}
Leveraging a key structural identity of general choice models, we (i) identify the outside-utility parameters under an affine miscalibration model in logit space via a simple regression calibration, (ii) handle broadly misspecified but approximately order-preserving predictors via maximum rank correlation estimation with finite-sample guarantees, and (iii) aggregate multiple predictors in a robust, scale-free manner. Our approach is model-agnostic and requires no further data beyond purchase data; we assume (approximate) inside-option utilities from observed choices as in the literature but allow the auxiliary outside-option predictor to leverage any external market-share data or grouped data that they may possess  \citep{li_talluri_tekin_2025_msom,subramanian_harsha_2021_mnsc,abdallah2021demand}.

Methodologically, our contribution is as follows:   (i) we give a structural identification strategy that uses an auxiliary outside-option predictor together with the observed information  and provide bounds that cleanly separate auxiliary-predictor quality from utility-learning error with downstream assortment-suboptimality guarantees.
(ii) calibration estimators that remain valid under affine and nearly-monotone miscalibration, and (iii) aggregation across multiple auxiliary predictors in a scale-free way. 

On the technical side, we build on standard statistical machinery (e.g., concentration bounds and empirical-process arguments for rank-based objectives) but tailor these methods to our setting to establish our results. Operationally, we translate calibration error into decision error: we show how to plug the calibrated choice probabilities into a standard assortment optimization problem and bound the resulting optimality gap relative to using the ground-truth choice probabilities.

\section{Related Literature}

\subsection{Synthetic Data for Operations Management}
\label{subsec:synthetic-om}

Using synthetic/predicted data is becoming common in operations management to address data scarcity, costly labeling, and constraints on data sharing, while still enabling rigorous model development, benchmarking, and counterfactual experimentation. In manufacturing, \citet{buggineni2024enhancing} provide a comprehensive review and propose a four-stage ``collect, preprocess, generate, and evaluate'' framework. They show that carefully constructed observations can improve downstream tasks, and they emphasize the importance of exploring how synthetic and real data can be combined in future works. Examples of applications in operations include (i) \emph{simulation-driven} time series and event logs that instantiate production systems to create public and reusable datasets \citep[e.g.,][]{chan2022generation,pendyala2024semiconductor}, and (ii) \emph{digital twin} generators that parameterize factories and emit synthetic telemetry for bottleneck analysis and policy testing at scale \citep[e.g.,][]{lopes2024digitaltwins}.

Synthetic data is also used in supply chain management and retail analytics. \citet{long2025ijpr} synthesize various scenarios of demand, risk, inventory, and perform forecasting and network optimization. Empirically, time series generators can help demand prediction when real observations are scarce or imbalanced. For instance, \citet{chatterjee2025ev} show that sequence synthesis based on GANs strengthens electric vehicle related demand forecasting. Methodologically, a line of research studies how to build general-purpose generators for tabular and sequential data such as Conditional Tabular GANs (\textsc{CTGAN}) and related models that are widely used \citep{xu2019ctgan,patki2016sdv}; For sequential data, time-series GANs (\textsc{TimeGAN}) and their variants have been shown to perform well empirically \citep{yoon2019timegan,esteban2017rgan}.

These model families are attractive because they preserve joint dependencies and rare patterns, yet their usefulness depends on application-grounded evaluation; thus, evaluation and governance are also important. Performance is evaluated on downstream tasks, for example, with the Synthetic Train, Test on Real (TSTR) protocol, alongside fidelity checks and privacy risk assessments \citep{esteban2017rgan,giomi2023anonymeter}. 

\subsection{Unobservable Choice Prediction and Calibration}

In the current literature, there are three lines of work to tackle the unobservable  no-purchase issue for choice-model estimation:

\begin{itemize}
    \item \emph{Likelihood/EM approaches} treat no-purchases as latent counts and maximize an incomplete-data likelihood, often via EM, to jointly recover arrivals and preference parameters from purchase-only transaction streams \citep{talluri_vr_2004_ms} or use an external market-share data to calibrate and resolve the identification problem \citep{vulcano_vr_chaar_2010_ms, vanryzin_vulcano_2015_mnsc}.  
\item \emph{Two-step and moment-based estimators} to exploit identifiability structure in the multinomial-logit model with assortment sizes strictly greater than 1. Specifically, first estimate relative utilities from within-purchase shares and then \emph{calibrate} the outside-option per market size, e.g., using risk-ratio or related moment conditions \citep{talluri_2009_upf}; or decompose the likelihood into marginal and conditional components to consistently estimate with a \emph{completely censored} alternative using single-firm sales data \citep{newman_ferguson_garrow_jacobs_2014_msom}.  Relatedly, \citet{cho_ferguson_im_pekgun_2024_pom} develop a transaction-level two-step estimator that avoids time-window aggregation (accommodating frequent price/attribute changes) by first estimating normalized utilities from the conditional (purchase-only) likelihood and then identifying lost sales/outside-option parameters using either external market-share information or mild utility restrictions. They further propose a robust procedure that uses a range of plausible market shares. In addition,  \cite{li_talluri_tekin_2025_msom} propose GMM procedures when the observed purchase data have been shaped by a firm's own revenue-management controls (assortment or pricing decisions), thereby addressing optimization-induced endogeneity; they do not require external market-share data. 
    \item \emph{Optimization-based calibration/estimation} directly imputes lost/no-purchase demand and parameters by solving integrated loss-minimization or mixed-integer programs. These methods can be asymptotically consistent and perform well at operational scales to calibrate predictions of the outside-option with grouped data and a  linear MNL model \citep{subramanian_harsha_2021_mnsc}. When partial competitor aggregates are available, specialized identification and estimation procedures leverage such marginal information to calibrate preferences and market size in the presence of unobserved no-purchases \citep{talluri_tekin_2025_joom}. 
\end{itemize}

Our work departs from the above  literature as it is based on an existing (biased) black-box \emph{auxiliary predictor} of the outside option and turns it into a calibrated and statistically consistent estimator of no-purchase probabilities using only purchase-only data. In other words, our approach can be used as a plug-in calibration layer on top of existing methods, improving the estimation of no-purchase probabilities (both theoretically and empirically in our experiments).

\section{Problem Setup}\label{sec:setup}

\paragraph{Choice model.}
Let $\mathcal{I}$ denote a finite set of items. At each interaction we observe an offered subset (assortment) $\mathcal{S}\subseteq\mathcal{I}$ and an observed context $X$, and a customer (randomly) chooses one option $I\in \mathcal{S}\cup\{0\}$. Here $0$ indexes an outside-option, capturing no-purchase or purchase from a competitor.

We model choice probabilities via context-dependent utilities $u_i(X)$:
\begin{equation}\label{eq:prob}
\mathbb{P}(I=i\mid X,\mathcal{S})
=\frac{\exp(u_i(X))}{\sum_{i'\in \mathcal{S}}\exp(u_{i'}(X))+\exp(u_0(X))},
\qquad i\in \mathcal{S}\cup\{0\}.
\end{equation}

The role of $X$ is to let the model adapt across settings. For example, if we take $X=\mathcal{S}$ (the context encodes the offered assortment), then any choice rule can be represented by \eqref{eq:prob} with suitable utilities (i.e., utilities are allowed to change with $\mathcal{S}$). Indeed, if $p_i(S)$ denotes the target probability of choosing $i$ given $S$, set $u_i(S)=\log p_i(S)$ for $i\in S$ and $u_0(S)=\log p_0(S)$. The denominator then equals $\sum_{i'\in S}p_{i'}(S)+p_0(S)=1$, giving $\mathbb{P}(I=i\mid X,\mathcal{S})=p_i(S)$ \citep{manski1977structure}. More generally, $X$ may concatenate item-level features,
\[
X=\big[x_0^\top, x_1^\top,\ldots, x_{|\mathcal{S}|}^\top\big]^\top,
\]
where $x_0$ may capture market features (or be the zero vector). A concrete instance is the featurized MNL model \citep{tomlinson2021learning} with $u_i(X)=\beta^{*\top} x_i$, i.e., a linear specification with coefficients $\beta^*$.

Beyond information from the offered items, $X$ may also contain information about a competitor, such as its assortment $\mathcal{S}^c$. One simple specification takes $X=\mathcal{S}^c$, sets $u_i(X)\equiv u'_i$ for $i\neq 0$, and defines
\[
u_0(X)=\log\!\left(\sum_{i\in \mathcal{S}^c}\exp(u_i^c)+1\right).
\]
Under \eqref{eq:prob} this yields, for $i\in \mathcal{S}$,
\[
\mathbb{P}(I=i\mid X,\mathcal{S})
=\frac{\exp(u'_i)}{\sum_{i'\in \mathcal{S}}\exp(u'_{i'})+\sum_{i'\in \mathcal{S}^c}\exp(u^c_{i'})+1},
\]
and
\[
\mathbb{P}(I=0\mid X,\mathcal{S})
=\frac{\sum_{i'\in \mathcal{S}^c}\exp(u^c_{i'})+1}{\sum_{i'\in \mathcal{S}}\exp(u'_{i'})+\sum_{i'\in \mathcal{S}^c}\exp(u^c_{i'})+1}.
\]

We let $
p_0(X,\mathcal{S})=\mathbb{P}(I=0 \mid X,\mathcal{S})
$
denote the true outside-option probability.

Thus the outside-option aggregates (i) not purchasing and (ii) purchasing from the competitor; the overall rule retains an MNL form. As above, one may also incorporate feature vectors for items in $\mathcal{S}^c$ into $X$.

\paragraph{Missing data and predictor.}
In many applications we observe choices for $i\in \mathcal{I}$, but the outside-option $i=0$ is rarely recorded.  Brick-and-mortar retailers, for instance, typically log only transactions from customers who buy something and do not record customers who do not purchase, nor are they able to track the overall market share of no-purchase.   Manufacturers sell through distributors and third-party channels and do not have access to competitor sales or non-purchasers.

In this paper, we only observe transactions in which an inside item is purchased ($I\in\mathcal{S}$); equivalently, the observed samples are drawn from the conditional distribution of $(X,\mathcal{S},I)$ given $I\neq 0$. In addition, we can query an auxiliary predictor on the same inputs $(X,\mathcal{S})$ to obtain a prediction of the (unobservable) outside-option probability $\tilde p_0(X,\mathcal{S})$. We treat this predictor as a black box: we do not require access to (or knowledge of) its training data or proprietary feature engineering, only that we can evaluate its output on the observed $(X_k,\mathcal{S}_k)$ pairs.

We assume that the outside-option utility is linear in observed features $z(X)\in \mathbb{R}^d$:
\[
u_0(X)=\gamma^{*\top} z(X),
\]
for an unknown coefficient vector $\gamma^*\in \mathbb{R}^d$. In \cite{tomlinson2021learning} they take $z(X)=X$ (where $X$ can be market features or a competitor-context summary). More generally, $z(X)$ may collect embeddings from kernel basis functions \citep{yang2025reproducing} or from a neural network choice model trained on another dataset (e.g., the representation before the output layer \citep{zhang2025deep,li2025small,wang2023neural,wang2023transformer}). 

We observe \emph{purchase-only transaction data} $\{(X_k, \mathcal{S}_k, i_k)\}_{k=1}^n$ with $i_k\in \mathcal{S}_k$ denoting the purchased item, and we can query a predictor that returns an estimate  $\tilde{p}_0(X,\mathcal{S})$  for the outside-option $p_0(X,\mathcal{S})$.
The accuracy of $\tilde{p}_0(X,\mathcal{S})$ is unknown and may be poor. Our aim is to recover $p_0(X,\mathcal{S})$ without direct observations of $I=0$, using the potentially biased predictor as auxiliary information.

The predictor may be provided by a generative AI system, a consulting firm, or another external source; however, it can be biased or inaccurate because it is not trained on our data or tailored to our model. A natural source of a predictor is an \emph{auxiliary} dataset or system that records the \emph{full funnel} (arrivals/sessions) so that non-purchases are observable.
Examples include (i) a limited time window in which a platform collected detailed browsing/search logs before reverting to transaction-only storage, (ii) a partner channel (or a different market) where session outcomes, including ``no-purchase,'' are logged, or (iii) third-party tools (consultants, survey-based choice predictors, or foundation-model-based systems) trained on external data.
A predictor trained in such environments can then be queried on the $(X,\mathcal S)$ pairs observed in the focal transactions, but it may be systematically biased due to domain mismatch or covariate shift.

\section{Key Structural Insight for Bias Correction}
This section highlights a structural identity implied by the choice model \eqref{eq:prob}.  
It shows that the (unobserved) outside-option probability can be summarized by a \emph{single index} that decomposes into (i) an $X$-only term coming from the outside utility and (ii) an ``inclusive value'' term coming from the attractiveness of the inside options in the offered assortment.  
This identity is the key reason we can \emph{debias} predictor outputs without ever observing no-purchase events.

Define the \emph{logit} (log-odds) and the \emph{log-sum-exp} (inclusive value) of the offered inside utilities as follows
\[
\eta(X,\mathcal{S})
:=\logit\!\big(p_0(X,\mathcal{S})\big)
=\log\!\left(\frac{p_0(X,\mathcal{S})}{1-p_0(X,\mathcal{S})}\right),
\qquad
s(X,\mathcal{S})
:=\log\!\left(\sum_{i\in \mathcal{S}}\exp\big(u_i(X)\big)\right).
\]
The next lemma states the identity that drives all our calibration procedures.  Recall that based on our own sales, assuming we offer assortments of size larger than one and cover all the products, we can estimate $s(X,\mathcal{S})$ based on our own sales.

\begin{lemma}\label{lemma:real_index}
For any $X$ and any $\mathcal{S}\subseteq \mathcal{I}$,
\begin{equation}
\eta(X,\mathcal{S})=\gamma^{*\top} z(X)\;-\;s(X,\mathcal{S}).
\end{equation}
\end{lemma}

The proof can be found in  Appendix~\ref{appx:proof_real_index}. Lemma~\ref{lemma:real_index} has two important implications.
First, it gives an intuitive decomposition: holding $X$ fixed, a more attractive offered set (larger $s(X,\mathcal{S})$) lowers the outside log-odds, while a higher outside utility $\gamma^{*\top}z(X)$ raises it.
Second, and more crucially for identification, $\eta(X,\mathcal{S})$ is \emph{linear} in the two regressors $\big(z(X),\,s(X,\mathcal{S})\big)$ with coefficient vector $\big(\gamma^*,-1\big)$:
\[
\eta(X,\mathcal{S})
=\underbrace{\gamma^{*\top}z(X)}_{\text{$X$-only term}}
\;+\;
\underbrace{(-1)\cdot s(X,\mathcal{S})}_{\text{known coefficient}}.
\]
Because the coefficient on $s(X,\mathcal{S})$ is fixed at $-1$, any method that (nearly) recovers the coefficients on $z(X)$ and $s(X,\mathcal{S})$ up to a \emph{common scale} immediately pins down $\gamma^*$ by taking a ratio. This is exactly what happens when a predictor is miscalibrated by an affine transformation in logit space.

We next consider the two main cases of bias we tackle in this paper:
\begin{enumerate}
\item A linearly biased predictor with fixed bias in scale and a constant shift
\item A predictor that is biased but preserves the order of the logits. This setting includes the first case as a special case.
\end{enumerate}
We give the precise assumptions and the results of the two cases in the next two sections.



\section{Linearly Biased Predictor}\label{sec:lin_pred}
We begin with a simple linear-regression calibration procedure for a linearly-biased-in-logits auxiliary predictor and show it corrects for the bias and identifies the outside-option utility parameters and yields consistent recovery of the unobserved outside-option probability. Essentially Lemma~\ref{lemma:real_index} turns a miscalibrated outside-option score into a valid estimator once we exploit the shared-scale structure in logit space.  

Let $y(X,\mathcal{S})=\logit(\tilde p_0(X,\mathcal{S}))$ be the predictor's logit.  Our assumption in the section is the following:
\begin{assumption}[Affine miscalibration in logit space (linear bias)]\label{as:linear}
The auxiliary predictor's outside-option logit satisfies
\begin{equation}\label{eq:linear}
y(X,\mathcal{S}) \;=\; a^* \;+\; b^*\,\eta(X,\mathcal{S}) \;+\; \epsilon,
\end{equation}
with unknown $a^*,b^*\in\mathbb{R}$ and a noise term $\epsilon\in\mathbb{R}$. The slope $b^*$ and intercept $a^*$ do not vary with context $X$ or assortment $\mathcal{S}$.
\end{assumption}

The noise term $\epsilon$ captures idiosyncratic predictor  randomness and residual misspecification (e.g., finite-Monte-Carlo/sampling error or unmodeled heterogeneity) not explained by $\eta(X,\mathcal{S})$. In an extreme case, we can have $\epsilon\equiv0$, which means the predictor is deterministic and misspecified only through a fixed affine transformation of the true index. Working in logits (log-odds), $\log\!\big(p_0/(1-p_0)\big)$, maps common misspecifications into an intercept and/or a slope term, both captured by Assumption~\ref{as:linear}:

\begin{enumerate}[itemsep=2pt,topsep=2pt,leftmargin=16pt]
\item \emph{Intercept shifts from base-rate mismatch (class-prior shift).}
A predictor is often trained in a different environment (another market, channel, or time period) where the \emph{baseline} no-purchase propensity differs, or it is trained with reweighting/negative-sampling for computational reasons. These changes can induce an approximately constant shift in log-odds.
Concretely, if the predictor’s odds are distorted by a constant factor $\tilde a^*>0$ relative to the true odds,
\[
\frac{\tilde p_0}{1-\tilde p_0} \;=\; \tilde a^*\,\frac{p_0}{1-p_0},
\]
then taking logs yields an additive shift
\[
\logit(\tilde p_0)\;=\;\log \tilde a^*\;+\;\logit(p_0),
\]
which corresponds to the intercept parameter $a^*=\log\tilde a^*$ in \eqref{eq:linear}.

\item \emph{Uniform scaling of confidence (temperature scaling).} A predictor may produce probabilities that are uniformly too extreme or too mild. In log-odds this corresponds to multiplying every signal by the same factor, appearing as the slope $b^*$ (typically $b^*>0$):
\[
\tilde p_0(X,\mathcal{S}) \;=\; \frac{\exp\!\big(b^*\, u_0(X)\big)}{\exp(b^*s(X,\mathcal{S}))+\exp\!\big(b^*\,u_0(X)\big)}.
\]
Intuitively, $b^*$ amplifies (when $b^*>1$)/reduces (when $b^*<1$) the sensitivity of the predicted probability to the utilities.  In machine learning this uniform rescaling, i.e., learning $b^*$ and applying it to the predicted logit $\tilde u_0$, is known as \emph{temperature scaling} and is widely used to calibrate overconfident predictors with strong empirical performance.
\end{enumerate}
More broadly, many post-hoc calibration pipelines apply an affine correction in logit space to a pre-trained model’s score (e.g., \cite{guo2017calibration}); Assumption~\ref{as:linear} formalizes exactly this common ``shift-and-scale'' miscalibration.

\subsection{Linear Regression Calibration}
Combining Assumption~\ref{as:linear} with Lemma~\ref{lemma:real_index} gives
\[
y(X,\mathcal{S})
= a^* \;+\; b^*\,\gamma^{*\top} z(X) \;-\; b^*\,s(X,\mathcal{S}) \;+\; \epsilon.
\]
Thus, regressing $y$ on $[1,\,z,\,s]$ estimates coefficients $(a,\theta_z,\theta_s)$ that, at the population level, satisfy
\[
\theta_z=b^*\,\gamma^*
\qquad\text{and}\qquad
\theta_s=-\,b^*.
\]
The key observation is that the nuisance scale $b^*$ cancels in the ratio
\[
\gamma^* \;=\; \frac{\theta_z}{-\theta_s},
\]
so we recover $\gamma^*$ without knowing $a^*$ or $b^*$. Based on the estimation $\hat \gamma$, mapping the recovered index  $\hat\eta(X,\mathcal{S})=\hat\gamma^\top z(X) - s(X,\mathcal{S})$ through the logistic function then yields a calibrated $\hat p_0(X,\mathcal{S})$ aligned with the model structure.

\paragraph{Oracle and practical implementations.}
If the utilities $\{u_i(X)\}_{i\in\mathcal{I}}$ are known, we regress $y$ on $\big[1,\,z(X),\,s(X,\mathcal{S})\big]$ to estimate $(a,\theta_z,\theta_s)$ and set $\hat\gamma=\hat\theta_z/(-\hat\theta_s)$ based on our estimations $(\hat \theta_z, \hat\theta_s)$. 
In practice, we only have estimated utilities $\{\hat u_i(X)\}_{i\in\mathcal{I}}$ learned from historical data, so we replace
\[
s(X,\mathcal{S}) \quad \text{with} \quad
\hat s(X,\mathcal{S}) \;:=\; \log\!\left(\sum_{i\in \mathcal{S}}\exp\big(\hat u_i(X)\big)\right),
\]
and run the regression of $y$ on $\big[1,\,z(X),\,\hat s(X,\mathcal{S})\big]$, and again take $\hat\gamma=\hat\theta_z/(-\hat\theta_s)$. Algorithm~\ref{alg:linear-calib} summarizes the procedure. In the subsequent analysis we quantify how the deviation between $\hat u_i$ and $u_i$ propagates to the final error in estimating $p_0(X,\mathcal{S})$.

\begin{algorithm}[t]
\caption{Linear Regression Calibration}
\label{alg:linear-calib}
\begin{algorithmic}[1]
\Require Data $\{(X_k,\mathcal{S}_k)\}_{k=1}^n$; predictor $\tilde p_0$; feature map $z(\cdot)$; learned utilities $\{\hat u_i(\cdot)\}$.
\Ensure Calibrated predictor $\hat p_0(\cdot,\cdot)$.

\For{$k=1,\ldots,n$}
  \State $y_k \gets \operatorname{logit}\!\big(\tilde p_0(X_k,\mathcal{S}_k)\big)$,\quad
         $\hat{s}_k \gets \log\!\left(\sum_{i\in \mathcal{S}_k}\exp\big(\hat u_i(X_k)\big)\right)$,\quad
         $z_k \gets z(X_k)$.
\EndFor
\State Fit by least squares the linear model 
\[
y_k \;=\; a \;+\;\theta_z^\top z_k \;+\; \theta_s\, \hat{s}_k \;+\; \epsilon_k,
\quad\text{and obtain } (\hat a,\hat\theta_z,\hat\theta_s).
\]
\State Recover the outside-utility coefficient: \(\hat\gamma \gets \hat\theta_z/(-\hat\theta_s)\) \ (provided $\hat\theta_s\neq 0$).
\State For any $(X,\mathcal{S})$, set
\[
\hat s(X,\mathcal{S})=\log\!\left(\sum_{i\in\mathcal{S}}\exp\big(\hat u_i(X)\big)\right),\quad
\hat\eta(X,\mathcal{S})=\hat\gamma^\top z(X)-\hat s(X,\mathcal{S}),\quad
\hat p_0(X,\mathcal{S})=\Logistic\!\big(\hat\eta(X,\mathcal{S})\big).
\]
\end{algorithmic}
\end{algorithm}

\paragraph{Learning $\{u_i(X)\}$ without outside-option labels.}
Even if $I=0$ is never observed, purchase data can identify the utilities of inside options up to a common additive constant (equivalently, they identify utility differences and within-assortment choice shares, a fact exploited in  \cite{talluri_2009_upf, newman_ferguson_garrow_jacobs_2014_msom,li_talluri_tekin_2025_msom}). 
Conditional on $I\in\mathcal{S}$, the likelihood is
\[
\mathbb{P}(I=i\mid I\in\mathcal{S},X,\mathcal{S}) \;=\;
\frac{\exp(u_i(X))}{\sum_{i'\in\mathcal{S}}\exp(u_{i'}(X))},
\]
which depends only on observable items and not on $u_0$. Maximizing the conditional log-likelihood over a parametric or flexible model for $u_i(X)$ (e.g., linear features, embeddings, neural nets) consistently learns the \emph{shape} of $u_i(X)$ up to a common shift fixed by a standard normalization (e.g., designate a baseline item or impose $\sum_{i\in\mathcal{S}}u_i(X)=0$). Once such a normalization is chosen, the inclusive value $s(X,\mathcal{S})=\log\sum_{i\in\mathcal{S}}\exp(u_i(X))$ is well-defined and directly usable in Algorithm~\ref{alg:linear-calib}. The unknown mapping from that normalization to the true outside utility is then learned in the calibration step via the ratio $\hat\theta_z/(-\hat\theta_s)$, without requiring outside-option observations. In Appendix \ref{appx: obs_utility_learn} we provide an analysis for a linear MNL model's estimation as an example.

\subsection{Analysis of the estimation errors}
We analyze the estimation error of Algorithm~\ref{alg:linear-calib}. All proofs can be found in Appendix \ref{appx:proofs-lin-pred}. Using the notation in the algorithm, without loss of generality, for data $\{(X_k,\mathcal{S}_k)\}_{k=1}^n$ drawn from an unknown joint distribution, feature map $z(\cdot)$, and learned utilities $\{\hat u_i(\cdot)\}$, define for each $k$
\[
\hat{s}_k \;=\; \log\!\left(\sum_{i\in \mathcal{S}_k}\exp\big(\hat u_i(X_k)\big)\right),
\qquad
z_k \;=\; z(X_k),
\qquad
s_k \;=\; s(X_k,\mathcal{S}_k).
\]
Let
\[
w_k \;=\; \begin{bmatrix} z_k \\[2pt] \hat{s}_k \end{bmatrix}\in\R^{d+1}
\]
denote the regressor used in Algorithm~\ref{alg:linear-calib}.

\begin{assumption}\label{as:w_k}
\begin{enumerate}[itemsep=2pt,topsep=2pt,leftmargin=16pt,label=(\roman*)]
\item (Bounded regressors) There exists $B_w>0$ such that $\|w_k\|_2\le B_w$ almost surely for all $k$.
\item (Nondegenerate covariance) Writing $\bar w=\tfrac{1}{n}\sum_{k=1}^n w_k$ and
\[
\Sigma \;=\; \frac{1}{n}\sum_{k=1}^n (w_k-\bar w)(w_k-\bar w)^\top,
\]
there exists $\lambda_0>0$ such that $\lambda_{\min}(\Sigma)\ge \lambda_0$.
\end{enumerate}
\end{assumption}

Assumption~\ref{as:w_k}(i) rules out pathological, unbounded covariates. Assumption~\ref{as:w_k}(ii) requires sufficient variation and excludes exact collinearity among the components of $w_k$; it fails, e.g., if $z_k$ or $\mathcal{S}_k$ is constant across $k$. In practice this requirement is relatively mild. Competitors and customers respond to offered assortments, firms update assortments over time, and context shifts occur routinely, all of which create variability in both $X_k$ and $\mathcal{S}_k$.

Under Assumption~\ref{as:linear}, the predictor satisfies
\begin{equation}\label{eq:linear_decompose}
y_k \;=\; a^* \;+\; b^*\,\eta(X_k,\mathcal{S}_k) \;+\; \epsilon_k,
\end{equation}
where $y_k=\logit\!\big(\tilde p_0(X_k,\mathcal{S}_k)\big)$. We further make the following assumptions.

\begin{assumption}[Informative predictor]\label{as:non_zero}
The slope $b^*\neq 0$.
\end{assumption}

Assumption~\ref{as:non_zero} is necessary for identification: if $b^*=0$, then $y_k=a^*+\epsilon_k$ carries no information about the true logit $\eta(X_k,\mathcal{S}_k)$. In that case no method can recover $\gamma^*$ from $y$.

\begin{assumption}[Noise]\label{as:noise_sub}
Given $\{w_k\}_{k=1}^n$, the errors $\{\epsilon_k\}_{k=1}^n$ are independent, mean-zero and conditionally sub-Gaussian with proxy variance $\sigma^2$, i.e., $\E[\epsilon_k\mid w_k]=0$ and $\E[\exp(t\epsilon_k)\mid w_k]\le \exp(\sigma^2 t^2/2)$ for all $t\in\R$.
\end{assumption}

Assumption~\ref{as:noise_sub} imposes a light-tailed condition on the predictor noise; equivalently, the logit of the predictor probability $\tilde p_0$ (i.e., $y_k$) is sub-Gaussian. Such assumptions are standard in operations and statistical learning \citep[e.g.,][]{wainwright2019high,vershynin2018,keskin2014dynamic,ban2021personalized}. We do not require independence between $\epsilon_k$ and $w_k$ beyond mean-independence (heteroskedasticity is allowed): we permit the predictor's bias to correlate with its inputs, as is common for practical predictors.

\subsubsection{Identification and consistency}
\begin{theorem}[Identification and consistency]\label{thm:ident-consistency}
Suppose Assumptions~\ref{as:linear}, \ref{as:w_k}, \ref{as:non_zero}, and \ref{as:noise_sub} hold.
\begin{enumerate}[itemsep=2pt,topsep=2pt,leftmargin=16pt,label=(\roman*)]
\item \textbf{Identification (oracle case).} If $s_k$ is observed (i.e., the regression uses $[z_k,s_k]$), then the unique population least-squares minimizer has slope coefficients $(\theta_z^*,\theta_s^*)=(b^*\gamma^*,-\,b^*)$. Hence
\[
\gamma^* \;=\; -\frac{\theta_z^*}{\theta_s^*}.
\]
\item \textbf{Consistency (plug-in case).} If $\hat s_k\to s_k$ in probability (e.g., when $\hat u_i(X)\to u_i(X)$ so that the inclusive value is learned consistently), then the OLS estimator based on $[z_k,\hat s_k]$ satisfies $\hat\gamma \stackrel{p}{\longrightarrow}\gamma^*$.
\end{enumerate}
\end{theorem}
Theorem~\ref{thm:ident-consistency} shows that we can \emph{identify} and \emph{consistently estimate} the outside-option choice probability without directly observing $I=0$, using only purchase data and a biased auxiliary predictor. Part (i) establishes that, at the population level, linear regression on $[z_k,\; s_k]$ recovers the scaled parameters $b^*\gamma^*$ and $-b^*$; taking their ratio identifies $\gamma^*$ and removes the unknown scaling $b^*$. Part (ii) shows robustness to learned utilities: if the plug-in estimate $\hat u_i$ approaches $u_i$ (e.g., with abundant data on observable choices), then the sample estimator converges in probability to the population target, so the outside-option utility parameters are recovered in large samples.

\subsubsection{Finite-sample error analysis}
We now quantify how errors in learning $\{u_i\}$ and finite-sample noise propagate to the calibration stage. Let
\[
\bar{\tau} \;:=\; \frac{1}{n}\sum_{k=1}^n \big(\hat s_k - s_k\big)^2
\]
measure the average squared error in the learned inclusive value. In many parametric settings, $\bar{\tau}$ converges to $0$ (see Appendix~\ref{appx: obs_utility_learn} for a linear MNL example where $\bar{\tau}=O(1/n)$).

\begin{theorem}[Finite-sample calibration error]\label{thm:finite-sample}
Under Assumptions~\ref{as:linear}, \ref{as:w_k}, \ref{as:non_zero}, and \ref{as:noise_sub}, for any $\delta\in(0,1)$, if
\begin{equation}\label{eqn:lin_finite_condition}
\sqrt{\bar{\tau}}\;\le\;\frac{\lambda_0}{8B_w}
\qquad\text{and}\qquad
n\;\ge\;\frac{128\,B_w^2\,\sigma^2}{\lambda_0^2\,{b^*}^2}\,
\log\!\Big(\frac{2(d+1)}{\delta}\Big),
\end{equation}
then with probability at least $1-\delta$,
\begin{equation}\label{eq:gamma-fs-Bw}
\|\hat\gamma-\gamma^*\|_2
\;\le\;
\frac{2\sqrt{2}\,\big(1+\|\gamma^*\|_2\big)\,B_w}{|b^*|\,\lambda_0}\,
\sigma\sqrt{\frac{\log\!\big(2(d+1)/\delta\big)}{n}}
\;+\;
\frac{2\big(1+\|\gamma^*\|_2\big)\,B_w}{\lambda_0}\,\sqrt{\bar{\tau}}.
\end{equation}
\end{theorem}

The bounds above quantify the estimation error for the calibration parameter $\gamma$. Condition~\eqref{eqn:lin_finite_condition} requires that the learned-utility error $\bar{\tau}$ be sufficiently small (which generally holds when $n$ is large; see Appendix~\ref{appx: obs_utility_learn} for an example) and that the sample size $n$ be large enough. The next corollary controls the error in the unobserved choice probability:

\begin{corollary}[No-purchase probability error]\label{cor:p0_error}
For any $(X,\mathcal{S})$,
\begin{equation}\label{eq:p0-Lip-explicit}
\big|\hat p_0(X,\mathcal{S})-p_0(X,\mathcal{S})\big|
\;\le\;\frac{1}{4}\,\Big(\,\|\hat\gamma-\gamma^*\|_2\,\|z(X)\|_2+\big|\hat s(X,\mathcal{S})-s(X,\mathcal{S})\big|\,\Big).
\end{equation}
Consequently, under the conditions of Theorem~\ref{thm:finite-sample}, with probability at least $1-\delta$,
\begin{align}\label{eq:p0-final-bound}
\big|\hat p_0(X,\mathcal{S})-p_0(X,\mathcal{S})\big|
\;\le\;&\frac{1}{4}\Bigg[\frac{2\big(1+\|\gamma^*\|_2\big)\,B_w}{\lambda_0}
\left(
\frac{\sqrt{2}\,\sigma}{|b^*|}\,
\sqrt{\frac{\log\!\big(2(d+1)/\delta\big)}{n}}
\;+\;
\,\sqrt{\bar{\tau}}
\right)\|z(X)\|_2 \nonumber\\[4pt]
&\qquad+\big|\hat s(X,\mathcal{S})-s(X,\mathcal{S})\big|
\Bigg].
\end{align}
\end{corollary}

Bounds \eqref{eq:gamma-fs-Bw} and \eqref{eq:p0-final-bound} cleanly separate two sources of error. The first term, with its order at $O\!\left(\frac{1}{|b^*|}\sqrt{\log(d/\delta)/n}\right)$, is the usual finite-sample term for regression \citep[e.g.,][]{wainwright2019high}, scaled by $1/|b^*|$. Here, $|b^*|$ measures the alignment between the synthetic logit $y$ and the true (unobserved) logit $\eta$: recall in the discussion after Assumption~\ref{as:linear}, larger $|b^*|$ means the predictor amplifies the sensitivity of the predictor to the true utilities and thus is more aligned (i.e., higher correlation), which reduces both the constant $1/|b^*|$ in the error bound and the sample-size requirement (which scales as $1/(b^*)^2$ in \eqref{eqn:lin_finite_condition}). Conversely, when $|b^*|$ is small (a weakly aligned predictor), recovering $\gamma$ from synthetic unobserved-choice data is intrinsically harder. The second term, $O(\sqrt{\bar{\tau}})$, captures how in-set utility learning errors propagate to calibration: poorer utility learning (larger $\bar{\tau}$) directly inflates the calibration error, while improvements in learning $s(\cdot)$ (smaller $\bar{\tau}$) tighten the bound. Thus, both a more relevant predictor (larger $|b^*|$) and better in-set utility learning (smaller $\bar{\tau}$) yield sharper calibration.

\textbf{Interpreting the probability error bound.} 
The bound in Corollary~\ref{cor:p0_error} is stated in terms of the index error and therefore depends on the scale of $z(X)$ and the inclusive value through $\|z(X)\|_2$ and $|\hat s-s|$. In applications it is natural to standardize features (e.g., rescale $z$ so that $\|z(X)\|_2\le 1$ on the relevant domain), in which case the bound becomes directly comparable across datasets and reads as an $O(\|\hat\gamma-\gamma^*\|_2 + |\hat s-s|)$ control on probability error (up to the logistic Lipschitz constant). Likewise, normalizations used when learning inside utilities implicitly control the scale of $s(X,\mathcal S)=\log\sum_{i\in\mathcal S}\exp(u_i(X))$, improving interpretability of probability-level guarantees.

\textbf{Dimensional scaling.} Beyond the logarithmic factor $\sqrt{\log(d+1)}$, the ratio $B_w/\lambda_0$ may grow with the feature dimension $d$. For example, if $\Sigma\approx I_{d+1}$ so that $\lambda_0\approx 1$, and the coordinates of $w_k$ are uniformly bounded, then typically $B_w=O(\sqrt{d+1})$.

\section{Nearly Monotone Biased Predictor}\label{sec:nmbs}

Assumption~\ref{as:linear} requires a shift-and-scale relationship in logit space that holds across all contexts and assortments. In particular, the slope $b^*$ and intercept $a^*$ do not vary with $(X,\mathcal S)$.  In practice, real prediction systems often deviate from such a simple global pattern: logits may be passed through nonlinear transforms, or different regimes/models may be used in different parts of the feature space. These effects can break an affine relationship in logit space even when the predictor still contains useful information, which motivates the rank-based, nearly-monotone approach developed in this section.

We relax Assumption~\ref{as:linear} (linear bias) and consider a more general condition: the predictor is \emph{nearly monotone} with respect to the unobserved probability. Recall given data $\{(X_k,\mathcal{S}_k)\}_{k=1}^n$, we define
\[
y_k=\operatorname{logit}\!\big(\tilde p_0(X_k,\mathcal{S}_k)\big),\qquad
\eta_k=\operatorname{logit}\!\big(p_0(X_k,\mathcal{S}_k)\big),
\]
as the predicted logit and the true logit of the unobservable outside choice, respectively. We then assume

\begingroup
\renewcommand{\theassumption}{5*}
\begin{assumption}[(Informal) Nearly monotone bias]\label{as:mnbs-informal}
There exists $\delta_0\in[0,1)$ and an index set $\mathcal{G}\subseteq\{1,\ldots,n\}$ with $|\mathcal{G}|\ge (1-\delta_0)n$ such that for all $k,l\in\mathcal{G}$ with $\eta_k\neq \eta_l$,
\[
\operatorname{sign}(y_k-y_l)=\operatorname{sign}(\eta_k-\eta_l).
\]
\end{assumption}
\endgroup
\addtocounter{assumption}{-1}
Assumption~\ref{as:mnbs-informal} only requires the predictor to preserve the pairwise ranking of logits (equivalently, the outside-option probability) on a large subset of samples. It imposes no parametric structure on bias in level or scale, does not require knowledge of which samples constitute $\mathcal{G}$ and places no restrictions on observations outside $\mathcal{G}$: the term $1-\delta_0$ isolates a subset of difficult pairs (e.g., ties, clipping, or local non-monotonic artifacts) on which we impose no structure, which makes the condition robust to systematic edge cases. This assumption also subsumes the linear-bias model (Assumption~\ref{as:linear}) when $b^*>0$: then $y=a^*+b^*\,\eta+\epsilon$ is increasing in $\eta$, so rankings are preserved on any subset where noise does not invert pairwise orders (e.g., when $\epsilon\equiv 0$ we have a perfectly monotone regime with $\delta_0=0$). We postpone a formal statement to Section~\ref{sec:MRC_analysis} (with the analysis on how the performance of our approach depends on the parameters, e.g., $\delta_0$, defined in the assumption) and first introduce our algorithm.

\paragraph{Why the nearly monotone condition is non-vacuous.}
Assumption~\ref{as:mnbs-informal} is satisfied by a broad class of score-based predictors whose outputs are \emph{monotone transformations} of the true outside logit, up to bounded local distortions from clipping, numerical approximation, or domain shift.
We formalize this in Section~\ref{sec:nmbs-perturb}: if
\(
Y(X,\mathcal S)=(h_0+\Delta h)\big(\eta(X,\mathcal S)\big)
\)
with $h_0$ strictly increasing and $\Delta h$ uniformly bounded (Assumption~\ref{as:perturb}),
then the predictor preserves pairwise orderings except on near-tie pairs, implying the formal nearly-monotone condition (Proposition~\ref{prop:perturb-epsmassart}).
This captures common practice where a model produces a real-valued risk score that is passed through a (possibly miscalibrated) monotone link: rankings are typically much more stable than absolute probability levels, especially under covariate shift.

\subsection{Maximum Rank Correlation Calibration}
We estimate $\hat p_0(X,\mathcal{S})$ by Algorithm~\ref{alg:mrc-calib} via Maximum Rank Correlation (MRC) \citep{han1987mrc}. For data $\{(X_k,\mathcal{S}_k)\}_{k=1}^n$, recall
\[
\hat{s}_k=\log\!\Big(\sum_{i\in\mathcal{S}_k}\exp(\hat u_i(X_k))\Big),
\qquad
w_k=\big[\,z(X_k)^\top,\ \hat{s}_k\,\big]^\top,
\]
where as in the linearly-biased case (Algorithm~\ref{alg:linear-calib}), the observed option utilities $\{\hat u_i\}_{i\in\mathcal{I}}$ can be learned from purchase-only data. The key ingredient of the method (and its main difference from Algorithm~\ref{alg:linear-calib}) is the empirical rank-correlation \citep{han1987mrc} objective for $\theta$:
\[
\widehat{\mathrm{RC}}_n(\theta)
=\frac{2}{n(n-1)}\sum_{1\le k<\ell\le n}
\One\!\Big\{(y_k-y_\ell)\,\big(\theta^\top w_k-\theta^\top w_\ell\big)>0\Big\}.
\]
Intuitively, under Assumption~\ref{as:mnbs-informal}, a larger synthetic logit $y_k$ tends to correspond to a larger true logit $\eta(X_k,\mathcal{S}_k)$ (when $s_k$ is computed from the true utilities). We therefore choose $\hat\theta=[\hat{\theta}_z^\top,\hat{\theta}_s]^\top$ to maximize the agreement between the pairwise rankings induced by $\{y_k\}$ and by $\{\theta^\top w_k\}$, exactly as captured by $\widehat{\mathrm{RC}}_n(\theta)$. After solving over the unit sphere, we fix scale and orientation by rescaling so that the last coordinate equals $-1$ and report $\hat\gamma=-\,\hat\theta_z/\hat\theta_s$, as in Algorithm~\ref{alg:linear-calib}. For example, in the extreme case where the predictor perfectly preserves the unobserved preference ordering ($\delta_0=0$ in Assumption~\ref{as:mnbs-informal}) and $\hat s_k=s_k$, one maximizer satisfies
\[
\hat{\theta}=\frac{1}{\sqrt{\|\gamma^*\|_2^2+1}}\,[\,\gamma^{*\top},\ -1\,]^\top,
\]
so the rescaling step yields $\hat{\gamma}=\gamma^*$. 

\begin{algorithm}[t]
\caption{Maximum Rank Correlation (MRC) Calibration}
\label{alg:mrc-calib}
\begin{algorithmic}[1]
\Require Data $\{(X_k,\mathcal{S}_k)\}_{k=1}^n$; predictor $\tilde p_0$; feature map $z(\cdot)$; learned utilities $\{\hat u_i(\cdot)\}$.
\Ensure Calibrated predictor $\hat p_0(\cdot,\cdot)$.
\For{$k=1,\ldots,n$}
  \State $y_k \gets \logit\!\big(\tilde p_0(X_k,\mathcal{S}_k)\big)$;
  \State $\hat{s}_k \gets \log\!\big(\sum_{i\in\mathcal{S}_k}\exp(\hat u_i(X_k))\big)$; \quad
        $w_k \gets [\,z(X_k)^\top,\ \hat{s}_k\,]^\top$.
\EndFor
\State Compute $\hat\theta=[\hat{\theta}_z^\top,\hat{\theta}_s]^\top$ by
\begin{equation}
\label{eq:RM_solve}
    \hat\theta \in \arg\max_{\|\theta\|_2=1}\ \frac{2}{n(n-1)}\sum_{k<\ell}
\One\!\big\{ (y_k-y_\ell)\,(\theta^\top w_k-\theta^\top w_\ell) > 0\big\}.
\end{equation}
\State Recover the outside-utility coefficient: \(\hat\gamma \gets -\,\hat\theta_z/\hat\theta_s\).
\State For any $(X,\mathcal{S})$, set
\[
\hat s(X,\mathcal{S})=\log\!\Big(\sum_{i\in\mathcal{S}}\exp(\hat u_i(X))\Big),\quad
\hat\eta(X,\mathcal{S})=\hat\gamma^\top z(X)-\hat s(X,\mathcal{S}),\quad
\hat p_0(X,\mathcal{S})=\Logistic\!\big(\hat\eta(X,\mathcal{S})\big).
\]
\end{algorithmic}
\end{algorithm}

\textbf{Remark.} Although the optimization step \eqref{eq:RM_solve} is non-smooth and non-concave \citep{sherman1993mrc}, it can be handled efficiently in practice. For small to medium $n$, an exact mixed-integer reformulation with pairwise ordering constraints delivers a certified global optimum together with a solver-reported MIP gap \citep{shin2021exact}. For larger $n$, an iterative marginal optimization method monotonically increases the objective and scales well \citep{wang2007imo}. One useful alternative is the linearized MRC estimator, which has a closed-form solution and can be used either as a fast stand-alone surrogate or as a strong initializer for the non-smooth MRC solve \citep{shen2023lmrc}.

\subsection{Analysis}
\label{sec:MRC_analysis}
We first introduce statistical assumptions that capture the ``quality'' of the predictor and of the data, then the following implications show how these assumptions impact the estimation performance. All proofs can be found in Appendix \ref{app:nmbs}.

Assume the samples $\{(X_k,\mathcal{S}_k)\}_{k=1}^n$ are drawn i.i.d.\ from an unknown joint distribution. For a generic draw $(X,\mathcal{S})$, define the covariate vector with the true log-sum-exp score
\[
W^*(X,\mathcal{S}) \;:=\; \big[\,z(X)^\top,\ s(X,\mathcal{S})\,\big]^\top \in \R^{d+1}.
\]
Let $(Y,W^*)$ and $(Y',W^{*\prime})$ denote two i.i.d.\ copies, where $Y=\logit\!\big(\tilde p_0(X,\mathcal{S})\big)$ is the predictor logit  (we omit the dependence on $(X,\mathcal{S})$ and $(X',\mathcal{S}')$ for simplicity of notation). Write the target direction (defined up to scale)
\[
\theta^* \;:=\; \big[\,\gamma^{*\top},\ -1\,\big]^\top,
\quad\text{so that}\quad
\eta(X,\mathcal{S})=\theta^{*\top}W^*(X,\mathcal{S}).
\]
Let
\[
\mathcal{O}\;:=\;\big\{(w,w')\in\R^{d+1}\times\R^{d+1}:\ \theta^{*\top}(w-w')\neq 0\big\}
\]
collect pairs with no tie in the true index.

We now formalize Assumption~\ref{as:mnbs-informal}.

\begin{assumption}[Nearly monotone bias]\label{as:mnbs-formal}
There exist $\rho_0\in[0,\tfrac12)$ and $\delta_0\in[0,1)$ and a set of ``good'' pairs $\mathcal G\subseteq\mathcal O$ with $\Pr\big((W^*,W^{*\prime})\in\mathcal G\big)\ge 1-\delta_0$ such that, for every $(W^*,W^{*\prime})\in\mathcal G$,
\[
\Pr\!\Big(\operatorname{sign}(Y-Y')\neq \operatorname{sign}\big(\theta^{*\top}(W^*-W^{*\prime})\big)\ \Big|\ W^*,W^{*\prime}\Big)\ \le\ \rho_0.
\]
\end{assumption}

Assumption~\ref{as:mnbs-formal} formalizes the intended \emph{nearly monotone} behavior of the predictor: for most context/assortment pairs, the ordering of no-purchase logits (equivalently, probabilities) induced by $Y$ agrees with the ordering of the true index $\eta$, up to a flip probability $\rho_0$. Specifically: (i) the \emph{pair fraction} $1-\delta_0$ allows an unrestricted set of difficult pairs (e.g., saturation/clipping, or local non-monotonic artifacts), which makes the condition robust to edge cases; and (ii) the \emph{flip rate} $\rho_0$ captures bounded stochastic or algorithmic noise in the predictor. In a perfectly monotone regime we have $\rho_0=0$ and $\delta_0=0$ (but the predictor may still be biased in level/scale).

The next theorem justifies the choice of the objective (loss) function and how this assumption captures the ``quality'' of the predictor. Let $D^*:=W^*-W^{*\prime}$ as the difference of the pair $(W^*, W^{*\prime})$. Consider the population counterpart of the MRC objective,
\[
\mathrm{RC}(\theta):=\E\Big[\One\!\big\{(Y-Y')\,\big(\theta^\top D^*\big)>0\big\}\Big].
\]

\begin{theorem}[Population approximate optimality]\label{thm:pop-approx}
Under Assumption~\ref{as:mnbs-formal} and assume $\Pr((W^*,W^{*'})\in \mathcal{O})=1$, for any $\theta$,
\[
\mathrm{RC}(\theta^*)-\mathrm{RC}(\theta)\ \ge\
(1-2\rho_0)\,\Pr\!\Big(\operatorname{sign}(\theta^{*\top}D^*)\neq \operatorname{sign}(\theta^\top D^*) \mid (W^*,W^{*'})\in \mathcal G\Big)\ -\ \delta_0.
\]
In particular, when $\delta_0=0$ we have $\theta^*\in \arg\max_{\theta}\mathrm{RC}(\theta)$.
\end{theorem}

Theorem~\ref{thm:pop-approx} justifies maximizing the empirical rank-correlation $\widehat{\mathrm{RC}}_n(\theta)$: when the predictor is monotone over all samples ($\delta_0=0$) and $s_k$ is correctly computed, the population maximizer coincides with the target direction $\theta^*$ (up to scale) even $\rho_0>0$. More generally, the gap $(1-2\rho_0)$ quantifies how strongly the correct order is preferred on the good set $\mathcal G$: smaller $\rho_0$ yields larger separation. The term $\delta_0$ is an unavoidable bias due to arbitrarily adversarial pairs. Both terms determine the estimation performance at a population level.

We now impose regularity on the data.  The next two assumptions govern how often pairs are ``close'' along the true parameter $\theta^*$ and how dispersed pairwise differences are. 

\begin{assumption}[Difference and near-tie regularity]\label{as:design-grouped}
There exist constants $\tau_{\mathrm{nd}},p_{\mathrm{nd}},c_{\mathrm{lt}},L_{\mathrm{ac}},t_0>0$ and exponents $\alpha_{\mathrm{down}},\alpha_{\mathrm{up}}\in(0,1]$ such that:
\begin{enumerate}[leftmargin=18pt,itemsep=3pt,label=(\roman*)]
\item \textbf{Non-degenerate design.} For every unit $v\in \R^{d+1}$:\quad
$\Pr\big(|v^\top D^*|\ge \tau_{\mathrm{nd}}\mid \mathcal O\big)\ge p_{\mathrm{nd}}$.
\item \textbf{Lower tail along the truth.} For all $t\in(0,t_0]$:\quad
$\Pr\big(|\theta^{*\top}D^*|\le t\mid \mathcal O\big)\ge c_{\mathrm{lt}}\,t^{\alpha_{\mathrm{down}}}$.
\item \textbf{Uniform anti-concentration.} For all unit $v\in \R^{d+1}$ and all $t\in(0,1]$:\quad
$\Pr\big(|v^\top D^*|\le t\mid \mathcal O\big)\le L_{\mathrm{ac}}\,t^{\alpha_{\mathrm{up}}}$.
\item \textbf{Association between design and truth.} 
There exists a constant $c_{\mathrm{assoc}}\in(0,1]$ such that for every unit vector 
$v=(v_1,\dots,v_{d+1})\in\mathbb R^{d+1}$ with $v_{d+1}=0$ and every $t\in(0,t_0]$,
\[
\Pr\big(A(v)\cap B(t)\mid\mathcal O\big)
\;\ge\;
c_{\mathrm{assoc}}\,
\Pr\big(A(v)\mid\mathcal O\big)\,
\Pr\big(B(t)\mid\mathcal O\big),
\]
where $A(v)=\{v^\top D^*\le -\tau_{\mathrm{nd}}\}$ and 
$B(t)=\{0<\theta^{*\top}D^*\le t\}$ denote one-sided version of the events
appearing in Assumptions~\ref{as:design-grouped} (i) and (ii), respectively.
\end{enumerate}
\end{assumption}

These assumptions impose mild geometric regularity on pairwise differences. (i) Non-degenerate design (difference) rules out collapse by requiring a nontrivial fraction of differences to be bounded away from zero in every direction. Intuitively, any possible boundary $v$ can non-trivially separate some differences. (ii) The lower-tail along the truth ensures enough probability mass arbitrarily close to the separating hyperplane ${\theta^{*\top}D^*=0}$ to deliver curvature (and hence identification) of the population rank-correlation criterion at $\theta^*$. (iii) Uniform anti-concentration prevents pairs from piling up on any hyperplane, stabilizing the concentration of the estimation. These conditions are satisfied by many design classes, e.g., sub-Gaussian designs with nondegenerate covariance, isotropic log-concave distributions, and bounded-density product measures. From these classes we typically obtain linear lower- and upper-tail behavior (i.e., $\alpha_{\mathrm{down}}=\alpha_{\mathrm{up}}=1$).  Condition (iv) is a mild dependence restriction ruling out the pathological situation in which the pairs
that are near ties along the true direction $\theta^*$ (as
quantified by (ii)) are almost perfectly negatively associated with
other directions.  It is satisfied whenever the conditional design 
$D^*\mid\mathcal O$ does not exhibit strong negative dependence between 
$v^\top D^*$ and $\theta^{*\top}D^*$. 
It holds, for instance, under elliptical designs, 
log-concave designs, or any distribution with a jointly bounded density. Further discussion and examples are deferred to the Appendix \ref{appx:assp_discuss}.

The next two propositions show how Assumption~\ref{as:design-grouped} governs the local sensitivity of the criterion and the plug-in stability of Algorithm~\ref{alg:mrc-calib}.

\begin{proposition}[Local margin from design]\label{prop:margin}
Under Assumption~\ref{as:design-grouped}(i), (ii), (iv), there exist constants
\[
\kappa \ :=\ \frac{c_{assoc}}{2}p_{\mathrm{nd}}\,c_{\mathrm{lt}}\,\tau_{\mathrm{nd}}^{\alpha_{\mathrm{down}}},
\qquad
r_0 \ :=\ \frac{t_0}{\tau_{\mathrm{nd}}},
\]
such that for all $\theta=[\gamma^\top,-1]^\top$ with $\|\gamma-\gamma^*\|_2\le r_0$,
\[
\Pr\!\Big(\operatorname{sign}(\theta^{*\top}D^*)\neq \operatorname{sign}(\theta^\top D^*)\mid \mathcal O\Big)
\ \ge\ \kappa\,\|\gamma-\gamma^*\|_2^{\,\alpha_{\mathrm{down}}}.
\]
\end{proposition}

The term \(\Pr(\operatorname{sign}(\theta^{*\top}D^*)\neq \operatorname{sign}(\theta^\top D^*)\mid \mathcal O)\) is the pairwise \emph{misordering rate}: among pairs with no true tie, the chance that the candidate index \(\theta^\top W^*\) ranks the pair differently from the true index \(\theta^{*\top}W^*\). Assumption~\ref{as:design-grouped} implies the lower bound \(\text{Left Hand Side}\ge \kappa\,\|\gamma-\gamma^*\|_2^{\alpha_{\mathrm{down}}}\) for \(\|\gamma-\gamma^*\|_2\le r_0\). Thus, as \(\gamma\) departs from \(\gamma^*\), the misordering rate grows at least polynomially in the error; smaller \(\alpha_{\mathrm{down}}\) (more mass near the true boundary) yields greater local sensitivity of the error from candidate $\theta$ (equally, $\gamma$), whereas larger \(\alpha_{\mathrm{down}}\) makes the criterion flatter. The constant \(\kappa\) sets the scale, and \(r_0\) specifies the neighborhood where this curvature holds.

The next proposition gives the stability of Algorithm \ref{alg:mrc-calib} when facing learned utilities $\hat u_i$. Recall we use $\hat s_k=\log\!\big(\sum_{i\in\mathcal{S}_k}e^{\hat u_i(X_k)}\big)$ in Algorithm \ref{alg:mrc-calib}. Here we introduce the maximum estimation error:
\[
\tau_s\ :=\ \max_{1\le k\le n}\,|\hat s_k-s_k|.
\]
And we denote the objective function when we have perfectly learned utilities $\hat{u}_i=u_i$ as $$\widehat{\mathrm{RC}}^{*}_n(\theta):=\frac{2}{n(n-1)}\sum_{k<\ell}
\One\!\big\{ (y_k-y_\ell)\,(\theta^\top w_k^*-\theta^\top w_\ell^*) > 0\big\}$$
 with       $w^*_k = [\,z(X_k)^\top,\ s_k\,]^\top$.

\begin{proposition}[Plug-in stability]\label{prop:plugin}
Under Assumption~\ref{as:design-grouped}(iii), for any $\delta \in (0,1)$, with probability at least $1-\delta$, we have
\[
\sup_{\|\theta\|_2=1}\ \big|\widehat{\mathrm{RC}}_n(\theta)-\widehat{\mathrm{RC}}^{*}_n(\theta)\big|
\ \le\ L_{\mathrm{ac}}\,(2\tau_s)^{\alpha_{\mathrm{up}}}+C_U\sqrt{\frac{d+d\log(n/d)+\log(2/\delta)}{n}},
\]
where $C_U$ is the universal constant from Theorem \ref{thm:uprocess} in Appendix \ref{app:nmbs}.
\end{proposition}

The quantity \(\sup_{\|\theta\|_2=1}\ \big|\widehat{\mathrm{RC}}_n(\theta)-\widehat{\mathrm{RC}}^{*}_n(\theta)\big|
\ \) is the \emph{worst-case fraction of pairwise comparisons} whose contribution to the empirical rank-correlation changes when replacing the true scores \(s_k\) by the learned scores \(\hat s_k\). For any pair \((k,\ell)\) and unit \(\theta\), the margin perturbation satisfies \(|(\hat s_k-s_k)-(\hat s_\ell-s_\ell)|\le 2\tau_s\). Hence, a sign flip can occur only for pairs whose unperturbed margin lies within a \(2\tau_s\) tube of some separating hyperplane. Assumption~\ref{as:design-grouped}(iii) (uniform anti-concentration) bounds the prevalence of such near-hyperplane pairs, giving the H{\"o}lder stability \(\text{Left Hand Side}\le L_{\mathrm{ac}}(2\tau_s)^{\alpha_{\mathrm{up}}}\) when $n\rightarrow \infty$. Thus stronger anti-concentration (smaller \(L_{\mathrm{ac}}\)) and more accurate utilities (smaller \(\tau_s\)) directly improve the robustness of the rank-correlation objective.

\subsubsection{Finite-sample estimation error.}
\begin{theorem}[Finite-sample estimation error]\label{thm:finite}
Under Assumptions~\ref{as:mnbs-formal} and \ref{as:design-grouped}, 
let $\kappa$ and $r_0$ be as in Proposition~\ref{prop:margin} 
and assume $\rho_0<\frac12$.
For any $\delta\in(0,1)$, suppose that the following two conditions hold:

\begin{equation}\label{eq:feasibility-gap}
\begin{aligned}
&(1-2\rho_0)\,\kappa\, r_0^{\alpha_{\mathrm{down}}}
    > 2L_{\mathrm{ac}}\,(2\tau_s)^{\alpha_{\mathrm{up}}} + \delta_0,
\\[6pt]
&n \ge
\frac{
4\,C_U^2 \big(d + d\log(\tfrac{n}{d}) + \log \tfrac{2}{\delta}\big)
}{
\Big((1-2\rho_0)\,\kappa\, r_0^{\alpha_{\mathrm{down}}}
      - 2L_{\mathrm{ac}}(2\tau_s)^{\alpha_{\mathrm{up}}}
      - \delta_0
\Big)^2
}.
\end{aligned}
\end{equation}

Then with probability at least $1-\delta$,
\begin{equation}\label{eq:gamma-fs}
\|\hat\gamma - \gamma^*\|_2
\;\le\;
\left(
\frac{
2C_U \sqrt{\tfrac{d+d\log(n/d)+\log(2/\delta)}{n}}
\;+\;
2L_{\mathrm{ac}}(2\tau_s)^{\alpha_{\mathrm{up}}}
\;+\;
\delta_0
}{
(1-2\rho_0)\,\kappa
}
\right)^{\frac{1}{\alpha_{\mathrm{down}}}}
\;\le\; r_0.
\end{equation}

Moreover, even without the feasibility condition~\eqref{eq:feasibility-gap}, 
the bound \eqref{eq:gamma-fs} holds with probability at least $1-\delta$ 
on the event $\{\|\hat\gamma-\gamma^*\|_2\le r_0\}$.
\end{theorem}

Similar to Corollary \ref{cor:p0_error}, we have the following estimation bound for the unobserved probability:
\begin{corollary}[No-purchase probability error]\label{cor:p0_massart}
Fix any $(X,\mathcal{S})$ and assume the conditions of Theorem~\ref{thm:finite}.
Then, with probability at least $1-\delta$, we have
\begin{equation}\label{eq:p0-massart-final}
\begin{aligned}
\big|\hat p_0(X,\mathcal{S})-p_0(X,\mathcal{S})\big|
\;\le\;
&\frac{1}{4}\Bigg[
\|z(X)\|_2
\Bigg(
\frac{
2C_U\sqrt{\dfrac{d + d\log(n/d) + \log(2/\delta)}{n}}
+ 2L_{\mathrm{ac}}(2\tau_s)^{\alpha_{\mathrm{up}}}
+ \delta_0
}{
(1-2\rho_0)\,\kappa
}
\Bigg)^{\!1/\alpha_{\mathrm{down}}}
\\
&\qquad
+\; \big|\hat s(X,\mathcal{S})-s(X,\mathcal{S})\big|
\Bigg].
\end{aligned}
\end{equation}

Moreover, without the feasibility conditions, the same bound holds on the event
$\{\|\hat\gamma-\gamma^*\|_2\le r_0\}$ with probability at least $1-\delta$.
\end{corollary}

Condition~\eqref{eq:feasibility-gap} plays the role of a self-consistency (or localization) requirement: the combined approximation and plug-in errors in the numerator must be smaller than the design margin $(1-2\rho_0)\kappa r_0^{\alpha_{\mathrm{down}}}$. This margin increases with the design strength encoded by $\kappa=\frac{c_{assoc}}{2}p_{\mathrm{nd}}c_{\mathrm{lt}}\tau_{\mathrm{nd}}^{\alpha_{\mathrm{down}}}$ and with the neighborhood radius $r_0=t_0/\tau_{\mathrm{nd}}$, and it shrinks with predictor noise through $(1-2\rho_0)$. The sample-size threshold in \eqref{eq:feasibility-gap} then ensures that random fluctuations, governed by $C_U\sqrt{[d+d\log(n/d)+\log(2/\delta)]/n}$, are also dominated by the same margin. As $n$ grows, the stochastic term decays like $n^{-1/2}$ and the overall error scales as that term raised to $1/\alpha_{\mathrm{down}}$. For example, when $\alpha_{\mathrm{down}}=1$ we recover a linear transfer from stochastic error to parameter error. Better utility learning (smaller $\tau_s$) and a better aligned predictor (smaller $\delta_0$) both tighten the numerator and relax the conditions in \eqref{eq:feasibility-gap}. Conversely, if $\rho_0$ approaches $1/2$ or if $\delta_0$ and $\tau_s$ are large, the feasibility gap can close and no sample size suffices to make the unconditional guarantee hold. In that regime, the conditional version of \eqref{eq:gamma-fs} (on $\{\|\hat\gamma-\gamma^*\|_2\le r_0\}$) remains valid. 

For the bound in \eqref{eq:gamma-fs}, in an idealized limit $\tau_s=\delta_0=0$, the rate in the number of samples $n$ becomes
\[
\|\hat\gamma-\gamma^*\|_2
\;\lesssim\;
\Big(\tfrac{2C_U}{(1-2\rho_0)\kappa}\Big)^{\!1/\alpha_{\mathrm{down}}}
\Big(\tfrac{d+d\log(n/d)+\log(2/\delta)}{n}\Big)^{\!\frac{1}{2\alpha_{\mathrm{down}}}},
\]
so at $\alpha_{\mathrm{down}}=1$ we recover the $n^{-1/2}$ behavior like the linearly-biased case, while smaller $\alpha_{\mathrm{down}}$ slows the rate to $n^{-1/(2\alpha_{\mathrm{down}})}$. When either $\tau_s$ or $\delta_0$ is nonzero, there is an \emph{irreducible floor} as $n\to\infty$:
\[\
\left(\frac{2L_{\mathrm{ac}}(2\tau_s)^{\alpha_{\mathrm{up}}}+\delta_0}{(1-2\rho_0)\,\kappa}\right)^{\!1/\alpha_{\mathrm{down}}},
\]
so improvements in utility learning ($\tau_s\downarrow$) and cleaner prediction ($\delta_0\downarrow$, $\rho_0\downarrow$) reduce this asymptotic floor. 

\subsection{Example: Perturbative Monotone Link}\label{sec:nmbs-perturb}

This subsection illustrates the estimation bound above in a setting where the predicted logit $Y$ equals a strictly monotone baseline link plus a bounded, possibly non-monotone perturbation.

\begin{assumption}[Perturbative link]\label{as:perturb}
The predictor's (synthetic) outside-option logit satisfies
\[
Y(X,\mathcal{S})=\big(h_0+\Delta h\big)\!\left(\eta(X,\mathcal{S})\right),
\]
where $h_0:\R\to\R$ is strictly increasing and $\Delta h:\R\to\R$ is a bounded perturbation with $\|\Delta h\|_\infty\le \Delta$.
\end{assumption}

Assumption~\ref{as:perturb} captures a predictor with a monotone baseline $h_0$ plus a bounded deviation $\Delta h$. For instance, $\Delta h$ may represent residual noise depending on $(X,\mathcal{S})$ that $h_0$ cannot capture, or a distortion induced by training the predictor on a different data source. The scalar $\Delta$ measures the worst-case magnitude of this non-monotone distortion.

To measure the quality of the predictor, define the modulus of increase of $h_0$ as
\[
\omega_{h_0}(t):=\inf_{|u-v|\ge t}\big(h_0(u)-h_0(v)\big)\,\mathrm{sign}(u-v).
\]
The modulus $\omega_{h_0}$ quantifies how strongly $h_0$ separates inputs: larger $\omega_{h_0}(t)$ means greater separation for inputs at least $t$ apart. If $h_0'(\cdot)\ge m_0>0$ almost everywhere, then $\omega_{h_0}(t)\ge m_0 t$, i.e., $h_0$ has a linear lower slope bound $m_0$. The next result links this modulus to Assumption~\ref{as:mnbs-formal}.

\begin{proposition}[Perturbation $\Rightarrow$ Assumption \ref{as:mnbs-formal}]\label{prop:perturb-epsmassart}
Let $t_\Delta=\inf\{t>0:\ \omega_{h_0}(t)\ge 2\Delta\}$. Then for any pair with $|\theta^{*\top}D^*|>t_\Delta$,
\[
\Pr\!\Big(\operatorname{sign}(Y-Y')\neq \operatorname{sign}\big(\theta^{*\top}(W^*-W^{*\prime})\big)\ \Big|\ W^*,W^{*\prime}\Big)=0.
\]
Hence, if Assumption~\ref{as:design-grouped} and ~\ref{as:perturb} hold, Assumption~\ref{as:mnbs-formal} holds with $\rho_0=0$ and
\[
\delta_0\ =\ \Pr\big(|\theta^{*\top}D^*|\le t_\Delta\mid \mathcal O\big)\ \le\ L_{\mathrm{ac}}\,t_\Delta^{\alpha_{\mathrm{up}}}.
\]
In particular, if $h_0'(\cdot)\ge m_0>0$ almost everywhere, then $t_\Delta\le 2\Delta/m_0$.
\end{proposition}

The condition $\omega_{h_0}(t_\Delta)\ge 2\Delta$ implies that once two true logits differ by more than $t_\Delta$, the monotone component $h_0$ creates an output gap large enough to dominate any opposing perturbation $\Delta h$ of size at most $\Delta$. Signs are therefore preserved and no label flips occur ($\rho_0=0$). Specifically, the threshold $t_\Delta$ marks the smallest gap between two true logits beyond which the monotone part $h_0$ grows fast enough to dominate any perturbation $\Delta h$, ensuring that the predictor's output preserves the correct order.  Only pairs with $|\theta^{*\top}D^*|\le t_\Delta$ can be ambiguous; anti-concentration then converts this geometric window into the mass bound $\delta_0\le L_{\mathrm{ac}}\,t_\Delta^{\alpha_{\mathrm{up}}}$. A larger monotone slope $m_0$ reduces $t_\Delta$ linearly ($t_\Delta\le 2\Delta/m_0$), shrinking the ambiguous region and, when plugged into Theorem~\ref{thm:finite} and Corollary~\ref{cor:p0_massart}, tightening the resulting bounds.

\begin{corollary}[Finite-sample error under Assumption~\ref{as:perturb}]
\label{cor:perturb-fs}
Under Assumptions~\ref{as:design-grouped} and \ref{as:perturb},
Theorem~\ref{thm:finite} applies with $\rho_0=0$ and
$\delta_0\le L_{\mathrm{ac}}\,t_\Delta^{\alpha_{\mathrm{up}}}$.
Thus, for any $\delta\in(0,1)$, with probability at least $1-\delta$,
on the event $\{\|\hat\gamma-\gamma^*\|_2\le r_0\}$,

\begin{equation}
\|\hat\gamma-\gamma^*\|_2
\;\le\;
\Bigg(
\frac{
  2C_U\sqrt{\frac{d+d\log(n/d)+\log(2/\delta)}{n}}
  \;+\;
  2L_{\mathrm{ac}}\,(2\tau_s)^{\alpha_{\mathrm{up}}}
  \;+\;
  L_{\mathrm{ac}}\,t_\Delta^{\alpha_{\mathrm{up}}}
}{
  \kappa
}
\Bigg)^{\!\frac{1}{\alpha_{\mathrm{down}}}}
\label{eq:perturb-gamma-bound}
\end{equation}

\begin{equation}\label{eq:perturb-p0-bound}
\begin{aligned}
\big|\hat p_0(X,\mathcal{S})-p_0(X,\mathcal{S})\big|
\;\le\;
&\frac{1}{4}\Bigg[
\|z(X)\|_2\,
\Bigg(
\frac{
  2C_U\sqrt{\frac{d+d\log(n/d)+\log(2/\delta)}{n}}
  \;+\;
  2L_{\mathrm{ac}}\,(2\tau_s)^{\alpha_{\mathrm{up}}}
  \;+\;
  L_{\mathrm{ac}}\,t_\Delta^{\alpha_{\mathrm{up}}}
}{
  \kappa
}
\Bigg)^{\!\frac{1}{\alpha_{\mathrm{down}}}}
\\
&\qquad
+\;\big|\hat s(X,\mathcal{S})-s(X,\mathcal{S})\big|
\Bigg].
\end{aligned}
\end{equation}

Moreover, define the perturbative feasibility gap
\[
\Delta_{\mathrm{pert}}
\;:=\;
\kappa\,r_0^{\alpha_{\mathrm{down}}}
\;-\;
\Big(
   2L_{\mathrm{ac}}\,(2\tau_s)^{\alpha_{\mathrm{up}}}
   + L_{\mathrm{ac}}\,t_\Delta^{\alpha_{\mathrm{up}}}
 \Big).
\]

If $\Delta_{\mathrm{pert}}>0$ and
\begin{equation}\label{eq:nstar-perturb}
n
\;\ge\;
n_*^{\mathrm{pert}}(\delta)
\;:=\;
\frac{4\,C_U^2}{\Delta_{\mathrm{pert}}^2}\,
\Big(d +d\log(n/d)+ \log\tfrac{2}{\delta}\Big),
\end{equation}
then the event $\{\|\hat\gamma-\gamma^*\|_2\le r_0\}$ occurs automatically.
If, in addition, $h_0'(\cdot)\ge m_0>0$ almost everywhere,
replace $t_\Delta$ by $2\Delta/m_0$.
\end{corollary}

In this perturbation model, the flip probability is eliminated ($\rho_0=0$) and $\delta_0$ is replaced by the near-tie mass driven by $t_\Delta$. Feasibility requires that the combined approximation and plug-in errors, $2L_{\mathrm{ac}}(2\tau_s)^{\alpha_{\mathrm{up}}}+L_{\mathrm{ac}}t_\Delta^{\alpha_{\mathrm{up}}}$, be strictly smaller than the design margin $\kappa r_0^{\alpha_{\mathrm{down}}}$. This becomes easier as the perturbation weakens (smaller $\Delta$, or larger $m_0$ so that $t_\Delta\le 2\Delta/m_0$), which in turn relaxes the sample-size requirement in~\eqref{eq:nstar-perturb} and tightens the estimation bounds.

\section{Discussion and Extension}
\subsection{Assortment Optimization}\label{sec:assortment}

In this subsection we study how estimation error from Algorithm~\ref{alg:mrc-calib} affects the downstream decision problem of assortment optimization. Let each item $i\in\mathcal{I}$ yield revenue $r_i>0$. Fix a context $X$ and an assortment $\mathcal{S}\subseteq\mathcal{I}$. The \emph{true} and \emph{plug-in} (based on estimated choice probabilities $\hat p_i$) expected revenues are
\[
R(X,\mathcal{S}) \;=\; \sum_{i\in\mathcal{S}} r_i\,p_i(X,\mathcal{S}),
\qquad
\hat R(X,\mathcal{S}) \;=\; \sum_{i\in\mathcal{S}} r_i\,\hat p_i(X,\mathcal{S}).
\]
Let
\[
\mathcal{S}^*(X)\in\arg\max_{\mathcal{S}} R(X,\mathcal{S}),
\qquad
\hat{\mathcal{S}}(X)\in\arg\max_{\mathcal{S}} \hat R(X,\mathcal{S})
\]
denote, respectively, an optimal assortment under the true model and the assortment obtained by optimizing the plug-in objective. For notational simplicity we treat $X$ as fixed and independent of $\mathcal{S}$; the same analysis extends to the case $X=X(\mathcal{S})$ by replacing $R(X,\mathcal{S})$ with $R\!\big(X(\mathcal{S}),\mathcal{S}\big)$ throughout.

Define
\[
q_i(X,\mathcal{S}) :=\mathbb{P}(I=i\mid I\in \mathcal{S},\,X, \mathcal{S})=\frac{\exp(u_i(X))}{\sum_{j\in\mathcal{S}}\exp(u_j(X))}
\]
as the within-assortment choice share (i.e., conditional on not selecting the outside-option $0$) and $\hat q_i$ analogously from $(\hat u_0,\hat u)$. Then we have

\begin{theorem}[Assortment optimality gap]\label{thm:assortment-regret}
Assume the conditions of Corollary~\ref{cor:p0_massart} hold. For any $\delta\in(0,1)$, with probability at least $1-\delta$,
\begin{equation}\label{eq:rev-gap-bound-pert}
R\!\left(X,\mathcal{S}^*(X)\right)-R\!\left(X,\hat{\mathcal{S}}(X)\right)
\;\le\; 2R_{\max}\,\Big(C^{\mathrm{pert}}_{n}(\delta)\,\|z(X)\|_2+\varepsilon_s(X)+\varepsilon_q(X)\Big),
\end{equation}
where $R_{\max}:=\max_{i} r_i$ and
\[
\varepsilon_s(X):=\max_{\mathcal{S}}\big|\hat s(X,\mathcal{S})-s(X,\mathcal{S})\big|,
\qquad
\varepsilon_q(X):=\max_{\mathcal{S}}\sum_{i\in\mathcal{S}}\big|\hat q_i(X,\mathcal{S})-q_i(X,\mathcal{S})\big|,
\]
and $C^{\mathrm{pert}}_{n}(\delta)$ is the calibration coefficient guaranteed by Corollary~\ref{cor:p0_massart}, i.e.,\  the smallest quantity for which, uniformly over $\mathcal{S}$,
\[
\big|\hat p_0(X,\mathcal{S})-p_0(X,\mathcal{S})\big|
\ \le\ \frac{1}{4}\Big(C^{\mathrm{pert}}_{n}(\delta)\,\|z(X)\|_2+\big|\hat s(X,\mathcal{S})-s(X,\mathcal{S})\big|\Big).
\]
\end{theorem}

The bound \eqref{eq:rev-gap-bound-pert} separates: (i) the calibration term $C^{\mathrm{pert}}_{n}(\delta)\|z(X)\|_2$ supplied by the perturbation-based MRC analysis, (ii) the aggregation error $\varepsilon_s(X)$ propagated from the learned utility, and (iii) the within-assortment share error $\varepsilon_q(X)$ also dependent on the learned utility. In particular, when the conditions in Corollary~\ref{cor:p0_massart} hold, the calibration term decays at the same rate governed by the corollary.

\subsection{Multiple Predictors}\label{sec:multi-sim}
We now consider the setting where $M>1$ external predictors produce synthetic outside-option probabilities for a common context-assortment pair $(X,\mathcal{S})$. For each predictor $m=1,\ldots,M$ and observation $k=1,\ldots,n$, define the synthetic unobserved choice logit
\[
y_k^{(m)}=\logit\!\big(\tilde p_0^{(m)}(X_k,\mathcal{S}_k)\big).
\]
In the single-predictor case (Algorithm~\ref{alg:mrc-calib}), the MRC step maximizes the sample rank correlation between $\{y_k\}$ and $\{\theta^\top w_k\}$, after which we recover $\hat\gamma=-\,\hat\theta_z/\hat\theta_s$ and set $\hat p_0(X,\mathcal{S})=\Logistic(\hat\gamma^\top z(X)-\hat s(X,\mathcal{S}))$. With multiple predictors we keep the same recovery step but replace the single ranking with an aggregation of the \emph{pairwise orderings} implied by $\{y_k^{(m)}\}_{m=1}^M$. Using multiple predictors can reduce variance by averaging weakly correlated errors and can enlarge coverage when predictors are locally strong in different regions of $(X,\mathcal{S})$. However, at the same time, aggregation does not undo a bias shared by all predictors, heavy probability clipping can compress margins and increase ties, and a globally anti-monotone predictor can confound naïve pooling. These considerations motivate sign-based, scale-free methods that (i) retain MRC's invariances, (ii) avoid any separate ``orientation'' preprocessing by handling global flips \emph{within} the objective, and (iii) provide either reliability weighting or robustness to outliers. Specifically, below we provide two methods to aggregate multiple predictors. 

\textbf{Weighted average.} Let $s_{k\ell}^{(m)}=\sign\!\big(y_k^{(m)}-y_\ell^{(m)}\big)\in\{-1,0,1\}$ for $k<\ell$. A pooled, weighted objective treats each predictor as a ``judge'' with nonnegative weight $\pi_m$ on the simplex and introduces a \emph{global orientation} variable $o_m\in\{\pm1\}$ that is learned jointly with $\theta$ so that globally anti-monotone predictors are flipped automatically. We then maximize
\begin{equation}
\label{eq:multi-pool}
\widehat{\mathrm{RC}}^{\mathrm{pool}}_n(\theta;\pi)
\;=\;
\max_{o\in\{\pm1\}^M}\;
\frac{2}{n(n-1)}\sum_{k<\ell}\ \sum_{m=1}^M \pi_m\;
\One\!\Big\{\,o_m\,s_{k\ell}^{(m)}\big(\theta^\top w_k-\theta^\top w_\ell\big)>0\Big\}.
\end{equation}
A strong default takes $\pi_m=1/M$, while validation-based choices can up-weight more reliable predictors (e.g., by self-consistency or agreement with held-out outcomes). When predictor quality is heterogeneous but most are informative, this pooled formulation sharpens margins by averaging idiosyncratic pairwise ``flips'' and eliminates the need for any standalone orientation step.

\textbf{Median.} When a minority of predictors may be erratic or adversarial, we instead aggregate \emph{before} optimization via a median (or majority) sign, which confers high-breakdown robustness. Define
\[
s_{k\ell}\;=\;\sign\!\Big(\operatorname{median}_{m=1,\ldots,M}\{\,y_k^{(m)}-y_\ell^{(m)}\,\}\Big)\in\{-1,0,1\},
\qquad
\mathcal P=\{(k,\ell):k<\ell,\ s_{k\ell}\neq 0\},
\]
and maximize the consensus objective
\begin{equation}
\label{eq:multi-cons}
\widehat{\mathrm{RC}}^{\mathrm{cons}}_n(\theta)
=
\frac{2}{|\mathcal P|}\sum_{(k,\ell)\in\mathcal P}
\One\!\Big\{ s_{k\ell}\,\big(\theta^\top w_k-\theta^\top w_\ell\big)>0\Big\}.
\end{equation}
This approach down-weights outlying predictors through the median and avoids scale choices, at the cost of discarding tied pairs when the median sign is zero (a situation exacerbated by extreme clipping of probabilities).

\begin{algorithm}[t]
\caption{Multi-Predictor MRC Calibration (choose \emph{one} of A or B)}
\label{alg:multi-mrc-calib}
\begin{algorithmic}[1]
\Require Data $\{(X_k,\mathcal{S}_k)\}_{k=1}^n$; predictors $\{\tilde p_0^{(m)}\}_{m=1}^M$; feature map $z(\cdot)$; learned utilities $\{\hat u_i(\cdot)\}$; (optional) weights $\pi$.
\For{$k=1,\ldots,n$}
  \State $y_k^{(m)} \gets \logit\!\big(\tilde p_0^{(m)}(X_k,\mathcal{S}_k)\big)$ for $m=1,\ldots,M$;
  \State $\hat{s}_k \gets \log\!\big(\sum_{i\in\mathcal{S}_k}\exp(\hat u_i(X_k))\big)$;\quad
        $w_k \gets [\,z(X_k)^\top,\ \hat{s}_k\,]^\top$.
\EndFor
\State Compute pairwise signs $s_{k\ell}^{(m)} \gets \sign\!\big(y_k^{(m)}-y_\ell^{(m)}\big)$ for all $k<\ell$, $m=1,\ldots,M$.
\State \textbf{Choice A (Weighted Average):}
\Statex \quad Solve $(\hat\theta,\hat o)\in\arg\max_{\|\theta\|_2=1,\;o\in\{\pm1\}^M}\ \widehat{\mathrm{RC}}^{\mathrm{pool}}_n(\theta;\pi)$ via \eqref{eq:multi-pool} with $\pi_m\equiv 1/M$ or validation-chosen weights.
\State \textbf{Choice B (Median):}
\Statex \quad Set $s_{k\ell}\gets \sign(\operatorname{median}_m\{y_k^{(m)}-y_\ell^{(m)}\})$ and $\mathcal P=\{(k,\ell):k<\ell,\ s_{k\ell}\neq 0\}$, then compute $\hat\theta\in\arg\max_{\|\theta\|_2=1}\ \widehat{\mathrm{RC}}^{\mathrm{cons}}_n(\theta)$ via \eqref{eq:multi-cons}.
\State Recover $\hat\gamma \gets -\,\hat\theta_z/\hat\theta_s$ and, for any $(X,\mathcal{S})$, set
\[
\hat s(X,\mathcal{S})=\log\!\Big(\sum_{i\in\mathcal{S}}\exp(\hat u_i(X))\Big),\quad
\hat\eta(X,\mathcal{S})=\hat\gamma^\top z(X)-\hat s(X,\mathcal{S}),\quad
\hat p_0(X,\mathcal{S})=\Logistic\!\big(\hat\eta(X,\mathcal{S})\big).
\]
\end{algorithmic}
\end{algorithm}

Both choices succeed under Assumption~\ref{as:mnbs-informal} when predictor-induced orderings align with the true index often enough. Choice~A \eqref{eq:multi-pool} is preferable when most predictors are informative but differ in quality: implicit global orientation removes the need for a separate alignment step, and reliability weights can concentrate mass on higher-quality judges. Choice~B \eqref{eq:multi-cons} is preferable when some predictors may be erratic or adversarial: median-sign aggregation offers robustness with minimal tuning. In either case, ties due to severe clipping simply contribute no pairwise information; converting probabilities to logits expands margins but cannot fully recover lost orderings near $0$ or $1$. Finally, neither aggregation corrects biases \emph{shared} by all predictors, so modest validation or ground-truth checks remain valuable for selecting $\pi$ and stress-testing the fitted model.
\section{Experiments}

We empirically evaluate the proposed calibration methods in two settings. 
First, we run controlled synthetic experiments where the ground-truth outside-option probability $p_0(X,\mathcal{S})$ is known and we can separately vary (i) the sample size $n$, (ii) the first-stage inclusive-value error, and (iii) predictor quality (bias and noise). 
Second, we study a real-world application on the Expedia Personalized Sort dataset \citep{expedia-personalized-sort}, where we emulate a biased predictor by training a historical prediction model and deploying it under covariate shift.

Throughout, $\eta(X,\mathcal{S})=\gamma^{*\top}z(X)-s(X,\mathcal{S})$ denotes the (unobserved) outside-option logit and $y=\operatorname{logit}(\tilde p_0)$ the predictor-provided logit. 
Complete data-generation details, hyperparameters, and implementation specifics are provided in Appendix~\ref{sec:appendix_exp_setup}.

\subsection{Synthetic Data Experiments}
\label{subsec:syn_experiments}

We first validate the theoretical predictions under synthetic designs where we control the predictor link $h(\cdot)$ and the utility-learning accuracy. 
A visual overview of the synthetic pipeline is given in Figure~\ref{fig:synthetic_flow} (Appendix~\ref{subsec:synthetic_generation}).

\subsubsection{Performance Metrics}
\label{subsubsec:syn_metrics}

We report two complementary metrics.

\paragraph{Empirical $p_0$ error.}
On an independent test set of size $N$, we evaluate absolute errors $\{|p_0^{(k)}-\hat p_0^{(k)}|\}_{k=1}^N$, where $p_0^{(k)}=p_0(X_k,\mathcal{S}_k)$ and $\hat p_0^{(k)}=\hat p_0(X_k,\mathcal{S}_k)$. 
To summarize performance robustly, we report the $0.7$-quantile:
\begin{equation}
\text{Error}_{0.7}
= 
\text{Quantile}_{0.7}\!\left(\left\{\,\big|p_0^{(k)}-\hat p_0^{(k)}\big|\,\right\}_{k=1}^N\right).
\end{equation}

\paragraph{Revenue suboptimality.}
To quantify downstream impact, we solve an unconstrained assortment problem using the plug-in model and measure the relative revenue gap:
\begin{equation}
\text{Suboptimality (\%)} 
=
\frac{1}{N_{\text{test}}}\sum_{j=1}^{N_{\text{test}}}
\frac{R\!\left(X_j,\mathcal{S}^*(X_j)\right)-R\!\left(X_j,\hat{\mathcal{S}}(X_j)\right)}{R\!\left(X_j,\mathcal{S}^*(X_j)\right)}\times 100\%,
\end{equation}
where $\mathcal{S}^*(X)$ and $\hat{\mathcal{S}}(X)$ are defined in Section~\ref{sec:assortment}. 
(See Appendix~\ref{subsec:appendix_synthetic_exp} for the exact decision-instance construction.)

\subsubsection{Benchmarks and Algorithms}
\label{subsubsec:syn_benchmarks}

We include ``oracle'' baselines that remove first-stage utility error:
\begin{itemize}
\item \textbf{Linear\_oracle / MRC\_oracle:} Algorithms~\ref{alg:linear-calib} and~\ref{alg:mrc-calib} run with the \emph{true/oracle} inclusive values $s_k$ in place of $\hat s_k$.
\end{itemize}
Implementation details (numerical stabilization for OLS and the smoothed MRC optimization) are deferred to Appendix~\ref{appendix_syn_bench_algo}.

\subsubsection{Experimental Protocols (Exp~1--6)}
\label{subsubsec:syn_protocols}

Each experiment draws i.i.d.\ samples $\{(z_k,\hat s_k,y_k)\}_{k=1}^n$ under the synthetic data-generation process described in Appendix~\ref{subsec:synthetic_generation}. 
We consider two simulator/predictor families, matching the assumptions that motivate our estimators:
\begin{equation}
y_k = h(\eta_k) + \epsilon_k,\qquad \epsilon_k\sim\mathcal{N}(0,\sigma_\epsilon^2),
\end{equation}
where $\eta_k=\gamma^{*\top}z_k-s_k$ is the true outside-option logit.
In plots we refer to these settings as:
\begin{itemize}
\item \textbf{Linear Simulator (Lin Sim):} The simulator/predictor satisfies $h(\eta)=a^*+b^*\eta$ (Assumption~\ref{as:linear});
\item \textbf{Monotone Simulator (Monotone Sim):} The simulator/predictor satisfies  $h(\eta)=a^*+b^*\,\big[\tfrac{1}{20}\log(1+\exp(20\eta))-\tfrac{1}{20}\log 2\big]$ (monotone but non-linear, satisfying Assumption~\ref{as:mnbs-formal}).
\end{itemize}
Across Exp~1--6 we vary $n$, the inclusive-value error, $\sigma_\epsilon$, and the predictor bias scale $b^*$, and we also evaluate multi-predictor aggregation (Algorithm~\ref{alg:multi-mrc-calib}). 
All curves report averages over multiple random seeds; the exact grids and seed counts are listed in Appendix~\ref{subsec:appendix_synthetic_exp}.

\subsubsection{Results}
\label{subsubsec:syn_results}

\paragraph{Exp~1: Sample complexity.}
Figures~\ref{fig:linear_convergence_rate}--\ref{fig:mrc_convergence_rate} report $\text{Error}_{0.7}$ as a function of $n$.
Under the correctly specified linear predictor, Algorithm~\ref{alg:linear-calib} exhibits a clear decrease in error with $n$ and tracks its oracle closely (the remaining gap is driven by utility learning through $\hat s$). 
Under the non-linear monotone predictor, linear calibration plateaus, reflecting misspecification of the affine link.
In contrast, Algorithm~\ref{alg:mrc-calib} improves with $n$ under \emph{both} predictors: it is consistent under the monotone setting (its target regime) and remains low error under the linear predictor.

\begin{figure}[htbp]
    \centering
    \begin{minipage}{0.48\textwidth}
        \centering
        \includegraphics[width=\linewidth]{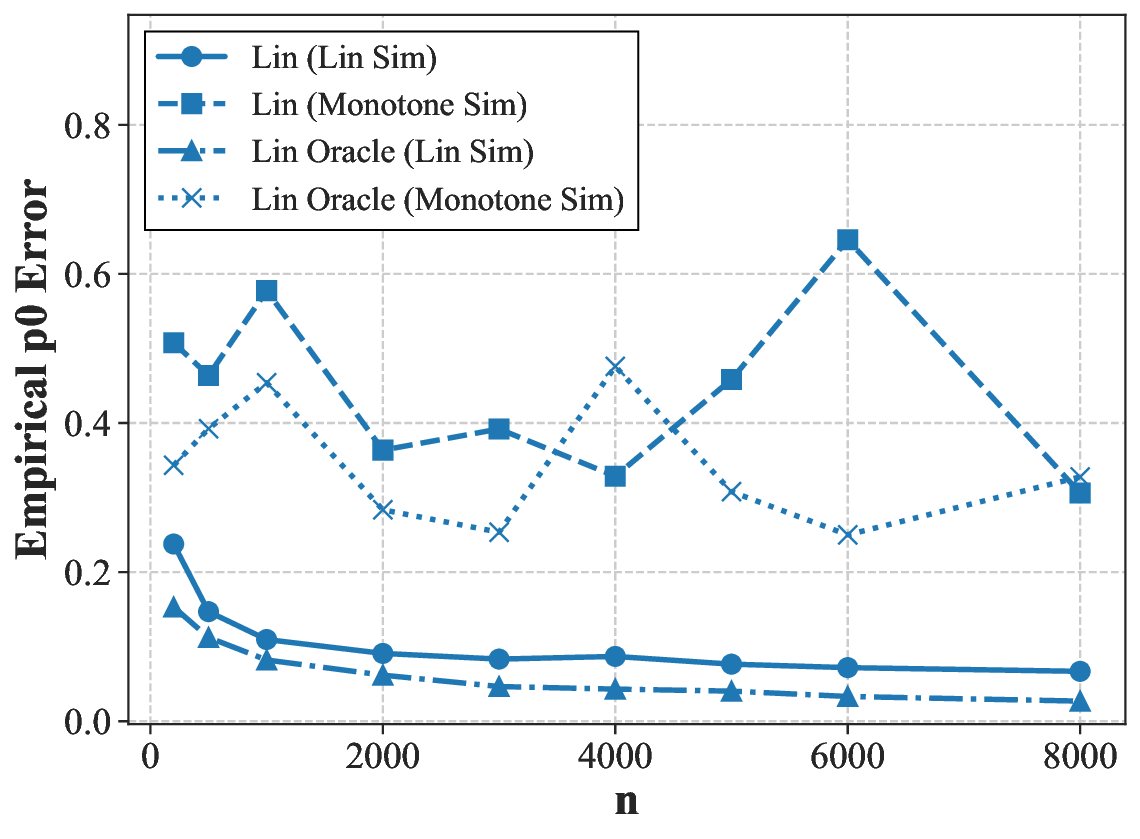}
        \caption{Exp 1, Convergence of Linear Calibration (Alg~\ref{alg:linear-calib}).}
        \label{fig:linear_convergence_rate}
    \end{minipage}
    \hfill
    \begin{minipage}{0.48\textwidth}
        \centering
        \includegraphics[width=\linewidth]{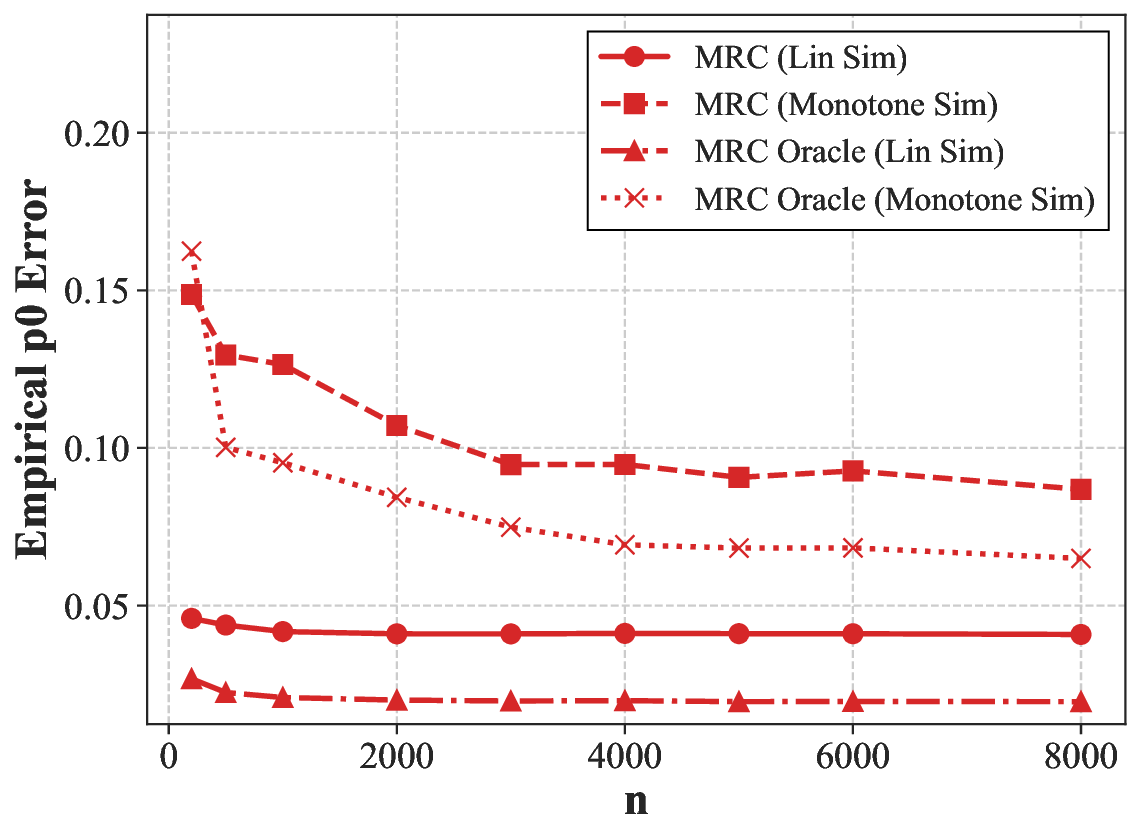}
        \caption{Exp 1, Convergence of MRC Calibration (Alg~\ref{alg:mrc-calib}).}
        \label{fig:mrc_convergence_rate}
    \end{minipage}
\end{figure}

\paragraph{Exp~2: Sensitivity to inclusive-value error.}
Figures~\ref{fig:linear_utility_noise_robustness}--\ref{fig:mrc_utility_noise_robustness} vary the first-stage utility noise and report performance against the error quantities that enter our theory.
For linear calibration we plot against $\sqrt{\bar\tau}$, where $\bar\tau=\frac{1}{n}\sum_{k=1}^n(\hat s_k-s_k)^2$; for MRC we plot against the sup-norm error $\tau_s=\max_k|\hat s_k-s_k|$ (cf.\ the MRC analysis).
In both cases the empirical curves are approximately monotone in these quantities, consistent with the plug-in error terms in our finite-sample upper bounds.

\begin{figure}[htbp]
    \centering
    \begin{minipage}{0.48\textwidth}
        \centering
        \includegraphics[width=\linewidth]{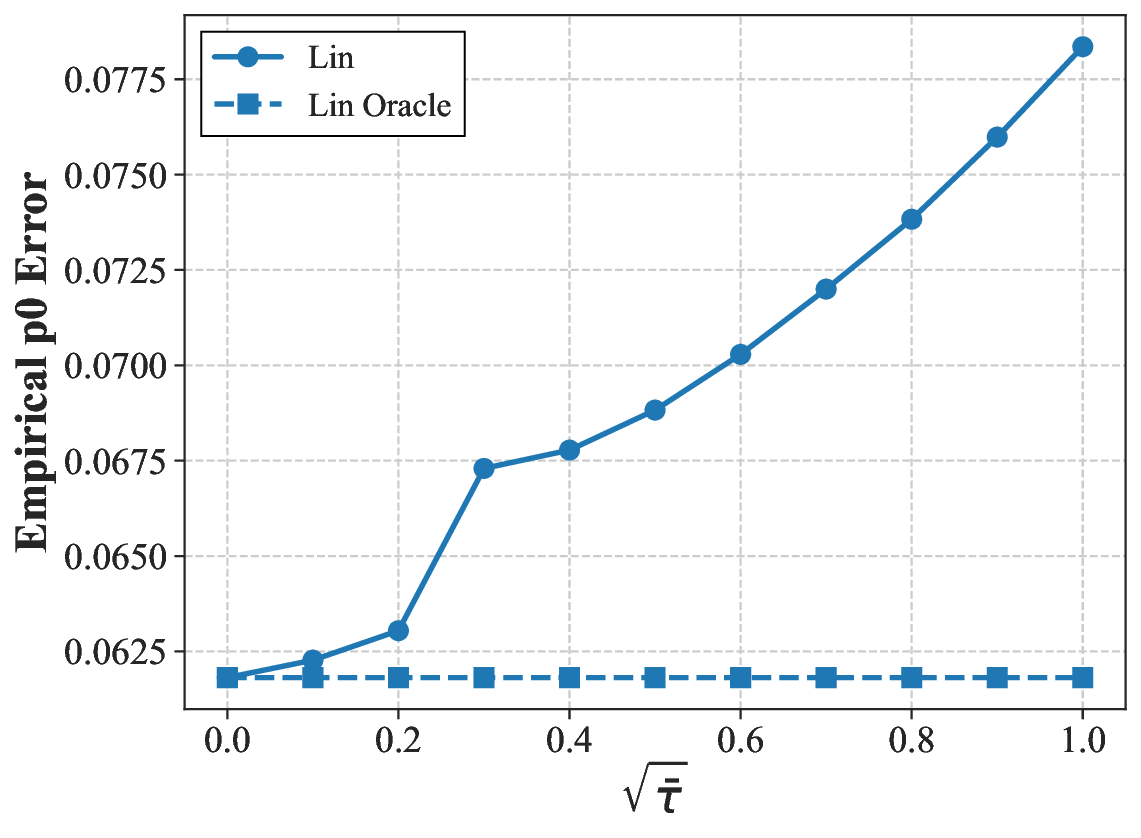}
        \caption{Exp 2, Linear Alg vs. Utility Estimation Error $\sqrt{\bar{\tau}}$.}
        \label{fig:linear_utility_noise_robustness}
    \end{minipage}
    \hfill
    \begin{minipage}{0.48\textwidth}
        \centering
        \includegraphics[width=\linewidth]{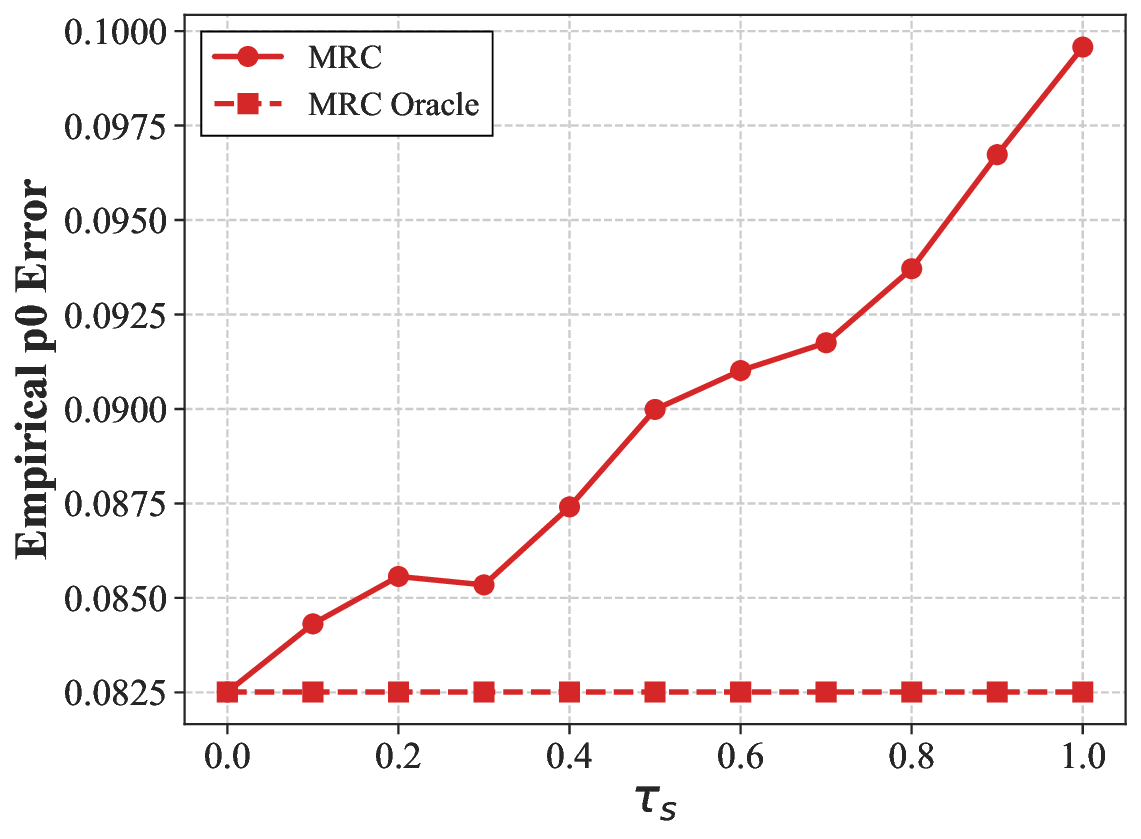}
        \caption{Exp 2, MRC Alg vs. Utility Estimation Error $\tau_s$.}
        \label{fig:mrc_utility_noise_robustness}
    \end{minipage}
\end{figure}

\paragraph{Exp~3: Robustness to predictor noise variance.}
Figures~\ref{fig:linear_robustness_to_simulator_noise}--\ref{fig:mrc_robustness_to_simulator_noise} vary $\sigma_\epsilon$.
Under the correctly specified linear predictor, linear calibration is stable as $\sigma_\epsilon$ increases. Under the non-linear monotone predictor, the same procedure is more sensitive and does not converge to the oracle because the affine model cannot represent $h(\cdot)$ (i.e., model misspecification).
MRC, which only requires approximate monotone ordering, degrades much more gently as $\sigma_\epsilon$ increases in the monotone setting; large $\sigma_\epsilon$ impacts performance primarily through pairwise rank flips.

\begin{figure}[htbp]
    \centering
    \begin{minipage}{0.48\textwidth}
        \centering
        \includegraphics[width=\linewidth]{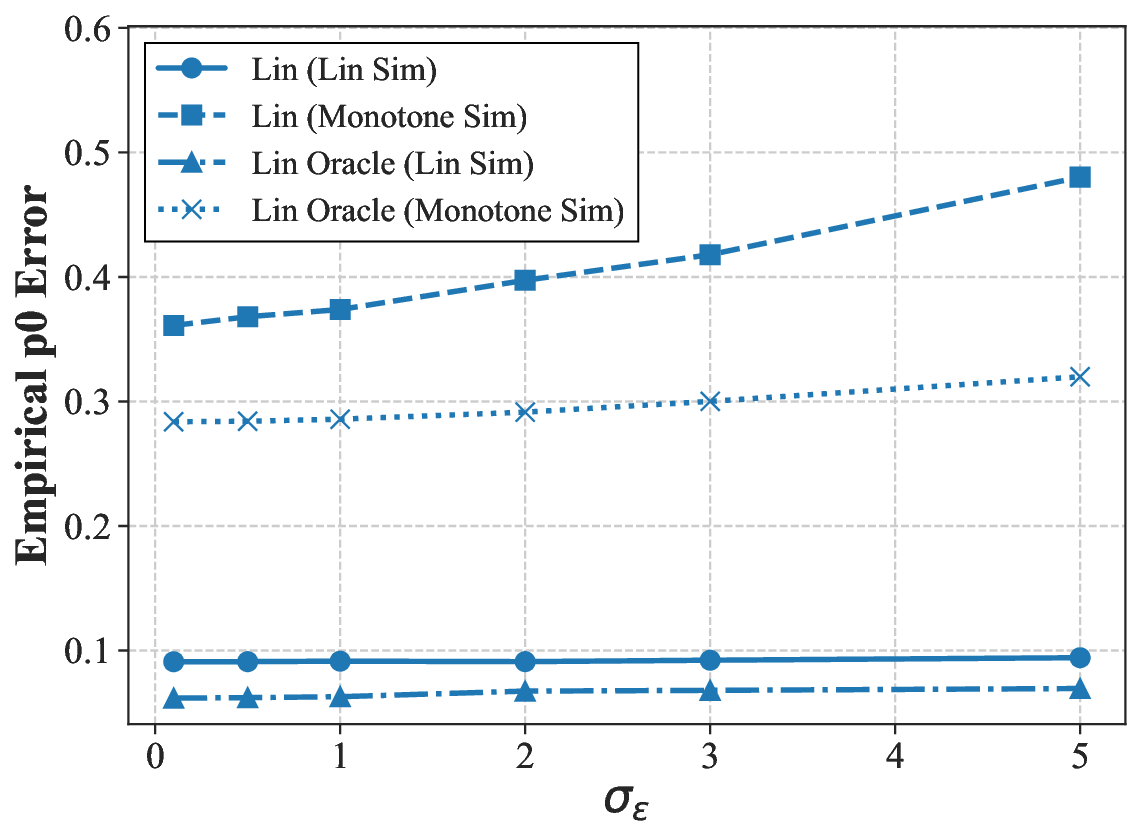}
        \caption{Exp 3, Linear Alg vs. Predictor Noise  $\sigma_{\epsilon}$.}
        \label{fig:linear_robustness_to_simulator_noise}
    \end{minipage}
    \hfill
    \begin{minipage}{0.48\textwidth}
        \centering
        \includegraphics[width=\linewidth]{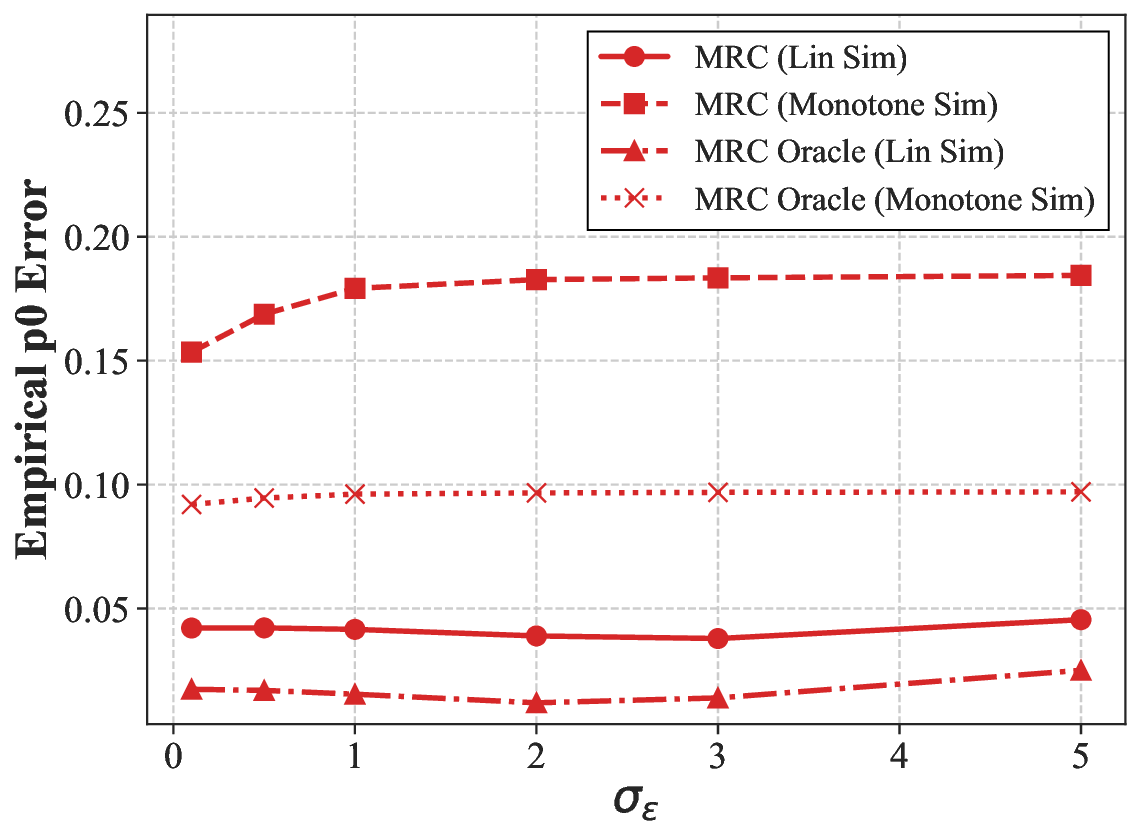}
        \caption{Exp 3, MRC Alg vs. Predictor Noise $\sigma_{\epsilon}$.}
        \label{fig:mrc_robustness_to_simulator_noise}
    \end{minipage}
\end{figure}

\paragraph{Exp~4: Effect of predictor scale $b^*$.}
Figures~\ref{fig:linear_b_sensitivity}--\ref{fig:mrc_b_sensitivity} vary the predictor scale $b^*$.
For OLS-based linear calibration under the linear predictor, increasing $b^*$ improves signal-to-noise and reduces estimation error. 
Under the monotone predictor, changing $b^*$ cannot remove misspecification bias, so only modest gains are observed.
For MRC, the linear-predictor curves are essentially scale-invariant (rank information is unchanged by affine rescaling), while under the monotone predictor larger $b^*$ reduces the effective flip probability and yields a mild improvement.

\begin{figure}[htbp]
    \centering
    \begin{minipage}{0.48\textwidth}
        \centering
        \includegraphics[width=\linewidth]{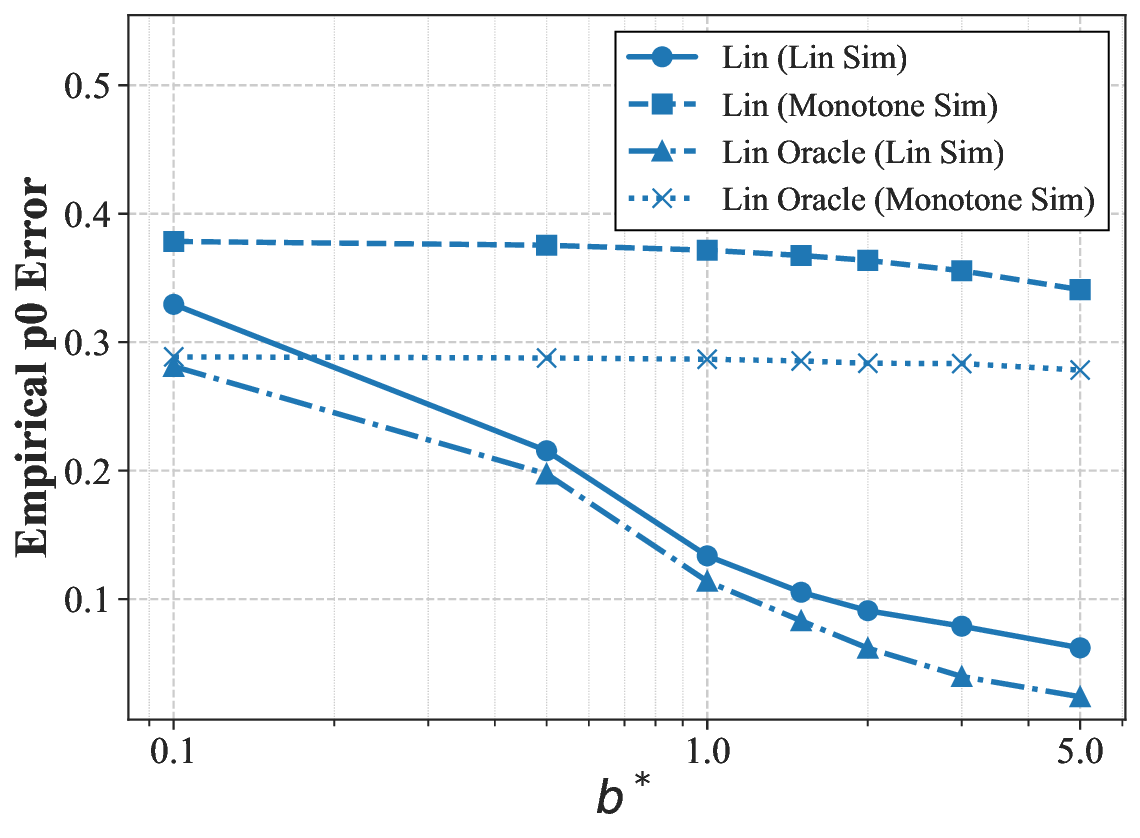}
        \caption{Exp 4: Linear Alg vs. Bias Scale $b^*$.}
        \label{fig:linear_b_sensitivity}
    \end{minipage}
    \hfill
    \begin{minipage}{0.48\textwidth}
        \centering
        \includegraphics[width=\linewidth]{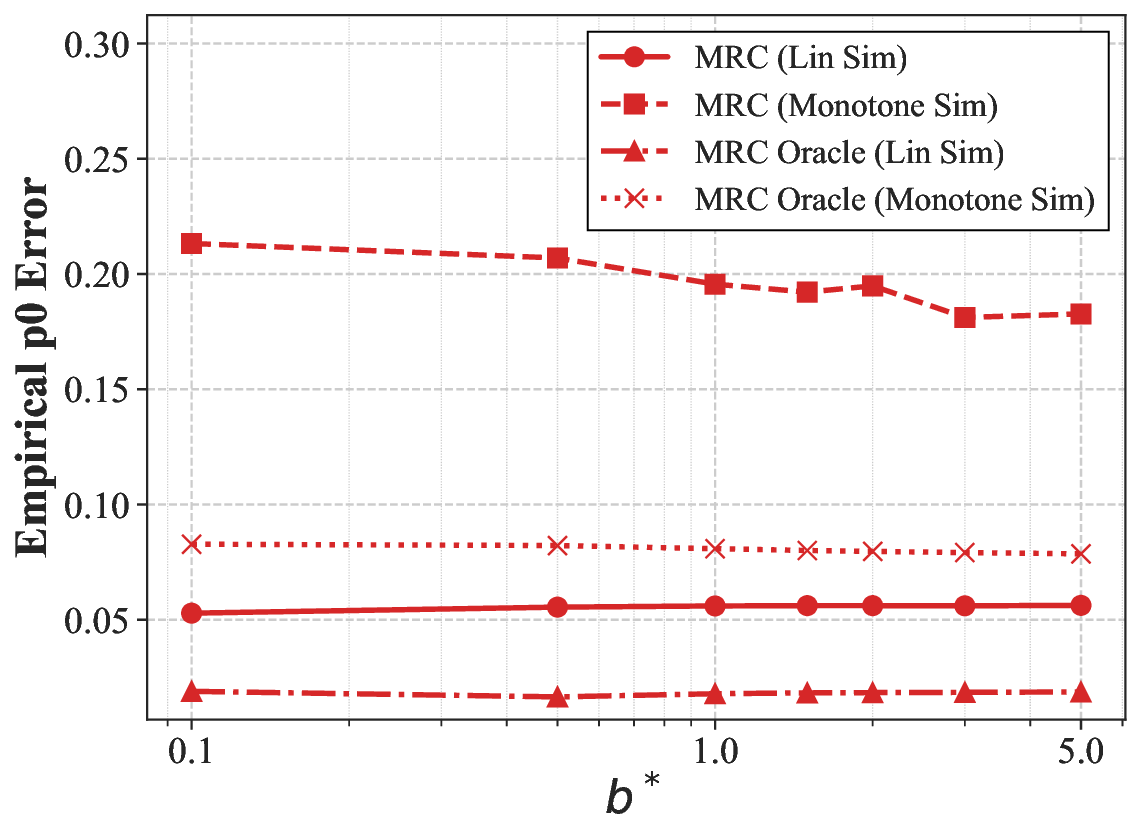}
        \caption{Exp 4: MRC Alg vs. Bias Scale $b^*$.}
        \label{fig:mrc_b_sensitivity}
    \end{minipage}
\end{figure}

\paragraph{Exp~5: Downstream revenue impact.}
Figures~\ref{fig:linear_revenue_regret}--\ref{fig:mrc_revenue_regret} report revenue suboptimality as $n$ grows.
When calibration is correctly specified (linear calibration under Lin Sim, or MRC under Monotone Sim), improved estimation accuracy translates into lower decision loss, consistent with the separation in Theorem~\ref{thm:assortment-regret}.
Under the monotone predictor, linear calibration yields persistently high suboptimality, mirroring its plateau in $\hat p_0$ accuracy.

\begin{figure}[htbp]
    \centering
    \begin{minipage}{0.48\textwidth}
        \centering
        \includegraphics[width=\linewidth]{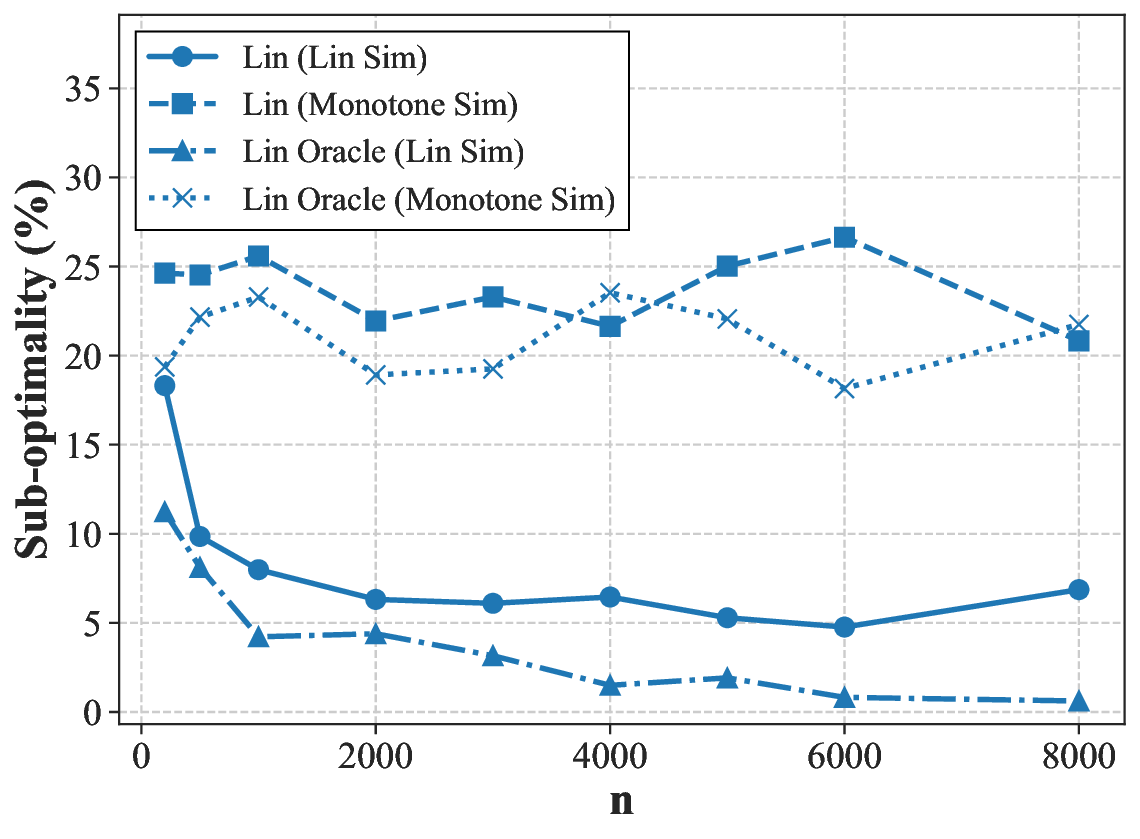}
        \caption{Exp 5, Revenue Suboptimality (Linear Alg).}
        \label{fig:linear_revenue_regret}
    \end{minipage}
    \hfill
    \begin{minipage}{0.48\textwidth}
        \centering
        \includegraphics[width=\linewidth]{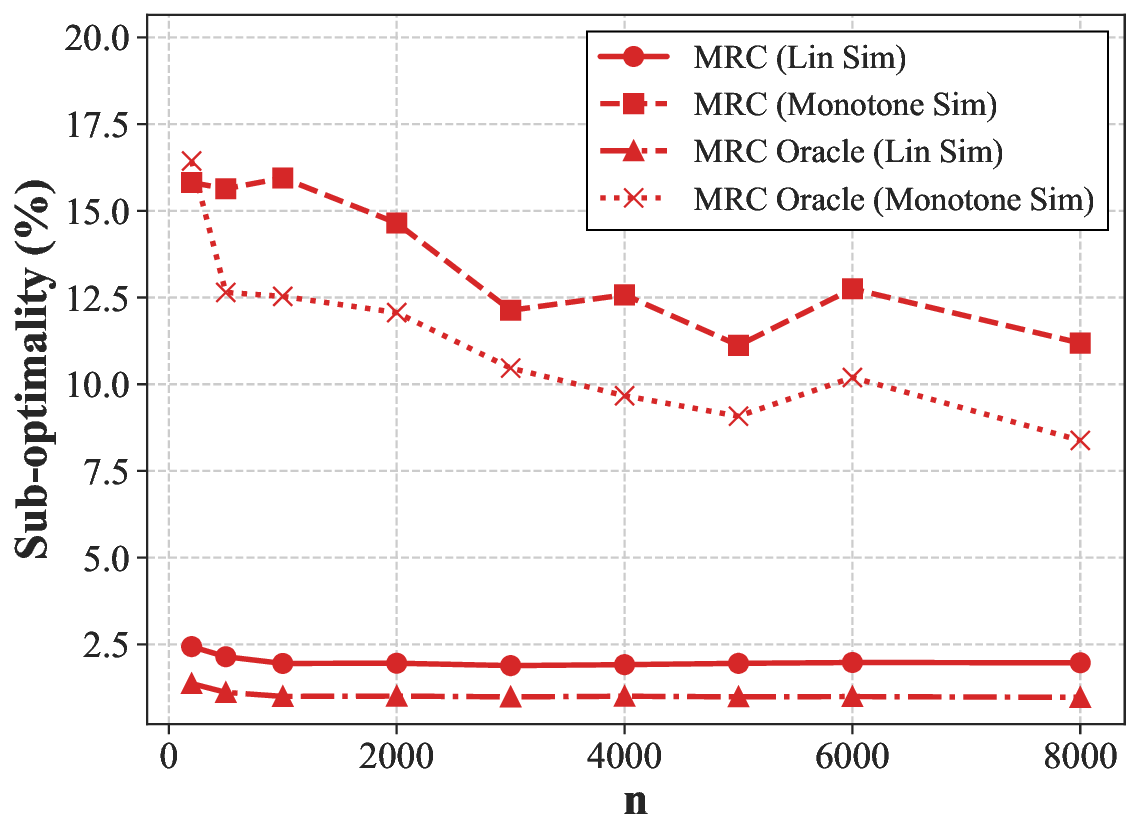}
        \caption{Exp 5, Revenue Suboptimality (MRC Alg).}
        \label{fig:mrc_revenue_regret}
    \end{minipage}
\end{figure}

\paragraph{Exp~6: Multi-Predictor Aggregation.}
Figure~\ref{fig:multi-sim_aggregation} evaluates Algorithm~\ref{alg:multi-mrc-calib} in a challenging environment with five heterogeneous predictors, including one adversarial agent (anti-monotone with high noise). We compare our methods against two baselines:
\begin{itemize}
    \item \textbf{Logit Mean:} A naive approach that computes the simple average of the predicted probabilities from all predictors, then applies Algorithm~\ref{alg:mrc-calib} using the estimated $\hat{s}$.
    \item \textbf{MRC Oracle:} Applies Algorithm~\ref{alg:mrc-calib} on the simple average of probabilities but uses the true inclusive values $s$.
\end{itemize}

\textbf{Results and Analysis:}
The naive Logit Mean performs worst across all sample sizes, confirming that simple averaging is easily corrupted by adversarial inputs. More interestingly, we observe a distinct \textbf{performance crossover} between our methods and the MRC Oracle:
\begin{enumerate}
    \item \textbf{Small Sample Regime ($n < 3000$):} Our aggregation methods (\textbf{Weighted Mean} and \textbf{Median}) significantly outperform the MRC Oracle. This indicates that in small samples, the dominant error source is the high-variance input noise from the adversarial predictor. Algorithms~\ref{alg:multi-mrc-calib}-A/B effectively filter this noise by learning to down-weight or ignore the bad predictor (robustness), whereas the Oracle, despite having perfect utility info, blindly averages the contaminated inputs.
    \item \textbf{Large Sample Regime ($n > 3000$):} The MRC Oracle eventually overtakes our methods. As $n$ grows, the predictor noise is smoothed out, and the error floor is determined by the utility estimation error $\tau$. Since the Oracle has access to the true $s$ ($\tau=0$), it is not bounded by this floor, achieving lower asymptotic error.
\end{enumerate}

Comparing our two strategies, \textbf{Median} (Alg~\ref{alg:multi-mrc-calib}-B) exhibits higher volatility than \textbf{Weighted Mean} (Alg~\ref{alg:multi-mrc-calib}-A) at small $n$. This is because the Median relies on a single order statistic, having lower statistical efficiency and discarding magnitude information. In contrast, Weighted Mean leverages ensemble variance reduction by averaging multiple reliable predictors, leading to faster initial convergence. However, both converge to similar performance at $n=4000$.

\begin{figure}[htbp]
    \centering
    \includegraphics[width=0.6\textwidth]{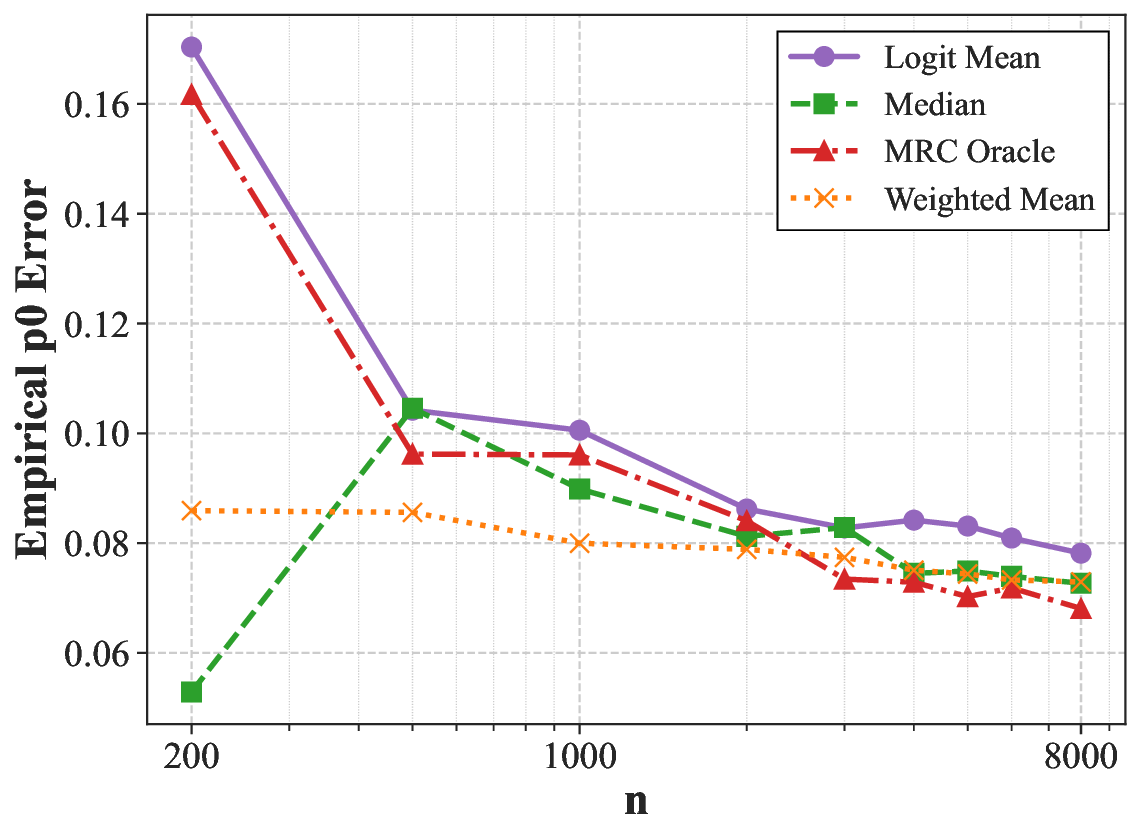}
    \caption{Exp 6, Multi-Predictor Aggregation Performance.}
    \label{fig:multi-sim_aggregation}
\end{figure}
\subsection{Real Data Application}
\label{subsec:real_data_app}

We next evaluate our framework on the \textbf{Expedia Personalized Sort Dataset} \citep{expedia-personalized-sort}. 
Each search session presents a set of hotels (the assortment $\mathcal{S}$) and results in either a booking of one hotel (inside option) or no booking (outside-option). 
The experimental design crosses two choices: (i) the utility model used to learn inside-item utilities (linear vs.\ neural network) and (ii) the calibration method used for the outside-option (Algorithm~\ref{alg:linear-calib} vs.\ Algorithm~\ref{alg:mrc-calib}).

\subsubsection{Data splitting and covariate shift}

We use a strict chronological split to induce realistic covariate shift and to mimic a common ``train-on-full-funnel / deploy-on-transactions'' pipeline:
\begin{itemize}
\item \textbf{Historical set (auxiliary, train):} the first 70\% of searches (Nov 2012--Apr 2013), used only to train a predictor that outputs a no-purchase probability $\tilde p_0(X,\mathcal S)$ from session features and the displayed set.
\item \textbf{Current set (deployment, test):} the remaining 30\% (Apr 2013--Jun 2013). To match our purchase-only setting, we keep \emph{only the booked sessions} from this period for (i) inside-utility learning and (ii) outside-option calibration via Algorithms~\ref{alg:linear-calib}--\ref{alg:mrc-calib}. Sessions with no booking are treated as unobserved during estimation and are used only for out-of-sample evaluation.
\end{itemize}

Although the public dataset records no-booking sessions, the above protocol emulates a practical situation in which detailed session-level logs are available only in an earlier instrumentation window (or via a partner channel) and are used to train a predictive model, while ongoing operation relies on a transaction database where only completed bookings are reliably recorded. 

Table~\ref{tab:data_stats} summarizes the two splits. 
While the no-purchase frequency is similar across periods, covariates shift substantially (e.g., the average listed price increases from \$225.95 to \$279.48), making the historical predictor become systematically miscalibrated on the current period, creating a realistic need for structural calibration. Intuitively, on average, the same input price leads to a lower no-purchase probability in the Historical Set (which trains the predictor) than the Current Set (which provides the real data without no-purchase). 

\begin{table}[htbp] 
\centering 
\caption{Summary Statistics of the Expedia Dataset Splits.} 
\label{tab:data_stats} 
\resizebox{0.9\textwidth}{!}{%
\begin{tabular}{lrr} 
\toprule 
\textbf{Statistic} & \textbf{Historical Set (Train)} & \textbf{Current Set (Test)} \\ 
\midrule 
Time Period & 2012-11-01 -- 2013-04-29 & 2013-04-29 -- 2013-06-30 \\ 
Duration & 179 Days & 62 Days \\ 
Total Impressions (Rows) & 6,984,474 & 2,933,056 \\ 
Unique Search Sessions & 279,540 & 119,804 \\ 
Unique Properties & 130,140 & 118,724 \\ 
\midrule 
\textit{Assortment Characteristics} & & \\ 
Avg. Assortment Size ($|\mathcal{S}_k|$) & 24.99 (Std: 9.08) & 24.48 (Std: 9.18) \\ 
Max Assortment Size & 36 & 38 \\ 
Avg. Price (USD) & \$225.95 & \$279.48 \\ 
\midrule 
\textit{Choice Outcomes} & & \\ 
Total Bookings (Inside Option) & 192,764 (68.96\%) & 83,829 (69.97\%) \\ 
No-Purchase Sessions (Outside Option) & 86,776 (31.04\%) & 35,975 (30.03\%) \\ 
\bottomrule 
\end{tabular}%
} 
\end{table}

\subsubsection{Features and models}

We partition covariates into (i) \emph{context} features $z(X)$ that drive the outside utility and (ii) \emph{item} features used to learn inside utilities $u_i(X)$. 
To avoid leakage, hotel-level historical aggregates are computed using only the Historical set and then merged into the Current set. 
Full preprocessing and feature engineering details are provided in Appendix~\ref{subsec:appendix_real_exp}.

For the biased predictor we train a CatBoost classifier \citep{prokhorenkova2018catboost} on the Historical set to predict item-level booking probabilities and then aggregate them within a session to obtain $\tilde p_0(X,\mathcal{S})$ (Appendix~\ref{subsubsec:appendix_real_workflow}). 
On the Current set we learn inside utilities $u_i(X)$ using only sessions with an observed booking (conditional MNL), and compute the inclusive value $\hat s(X,\mathcal{S})$. 
To match the purchase-only observation regime in our theory, we estimate the outside-utility parameter $\gamma$ using \emph{only the booked sessions} in the Current set when running Algorithm~\ref{alg:linear-calib} or Algorithm~\ref{alg:mrc-calib}. 
After obtaining $\hat\gamma$, we form predictions for \emph{all} Current-period sessions (booked and unbooked) via
\[
\hat p_0(X,\mathcal{S}) \;=\; \Logistic\!\big(\hat\gamma^\top z(X)-\hat s(X,\mathcal{S})\big).
\]
Hyperparameters for CatBoost and the utility learners are listed in Appendix~\ref{subsec:appendix_real_hyperparams}.

\subsubsection{Evaluation and baselines}

Calibration is performed \emph{without} using the current-period outside-option labels; those labels are used only for evaluation. 
We report Negative Log-Likelihood (NLL) and Expected Calibration Error (ECE) and also inspect reliability diagrams (definitions in Appendix~\ref{subsubsec:appendix_real_workflow}).

We compare five methods:
\begin{itemize}
\item \textbf{Predictor (baseline):} raw historical predictor predictions aggregated within-session;
\item \textbf{Linear (Lin-Util)} and \textbf{Linear (NN-Util):} Algorithm~\ref{alg:linear-calib} paired with a linear or neural inside-utility model;
\item \textbf{MRC (Lin-Util)} and \textbf{MRC (NN-Util):} Algorithm~\ref{alg:mrc-calib} paired with a linear or neural inside-utility model.
\end{itemize}
See Appendix~\ref{subsubsec:appendix_real_methods} for precise definitions.

\subsubsection{Results}
\label{subsubsec:real_results}

Figures~\ref{fig:real_nll}--\ref{fig:real_calibration} summarize performance on the Current set.
The uncalibrated predictor is highly miscalibrated under covariate shift (NLL $=1.7278$, ECE $=0.2730$). 
Linear calibration improves both metrics (NLL $=0.8330/0.8139$ for Lin-Util/NN-Util; ECE $\approx 0.08$), but leaves visible systematic deviations in the reliability diagram, consistent with non-linear bias that cannot be removed by an affine map.
MRC calibration provides a substantial additional improvement (NLL $=0.3850/0.3347$ and ECE $=0.0676/0.0464$ for Lin-Util/NN-Util), and the reliability curves in Figure \ref{fig:real_calibration} become close to the diagonal (perfect no-purchase prediction). 
Finally, within each calibrator, using a more accurate inside-utility model (NN-Util) consistently improves performance, in line with the synthetic findings that errors in $\hat s$ propagate into outside-option calibration.

\begin{figure}[htbp]
    \centering
    \begin{minipage}{0.48\textwidth}
        \centering
        \includegraphics[width=\linewidth]{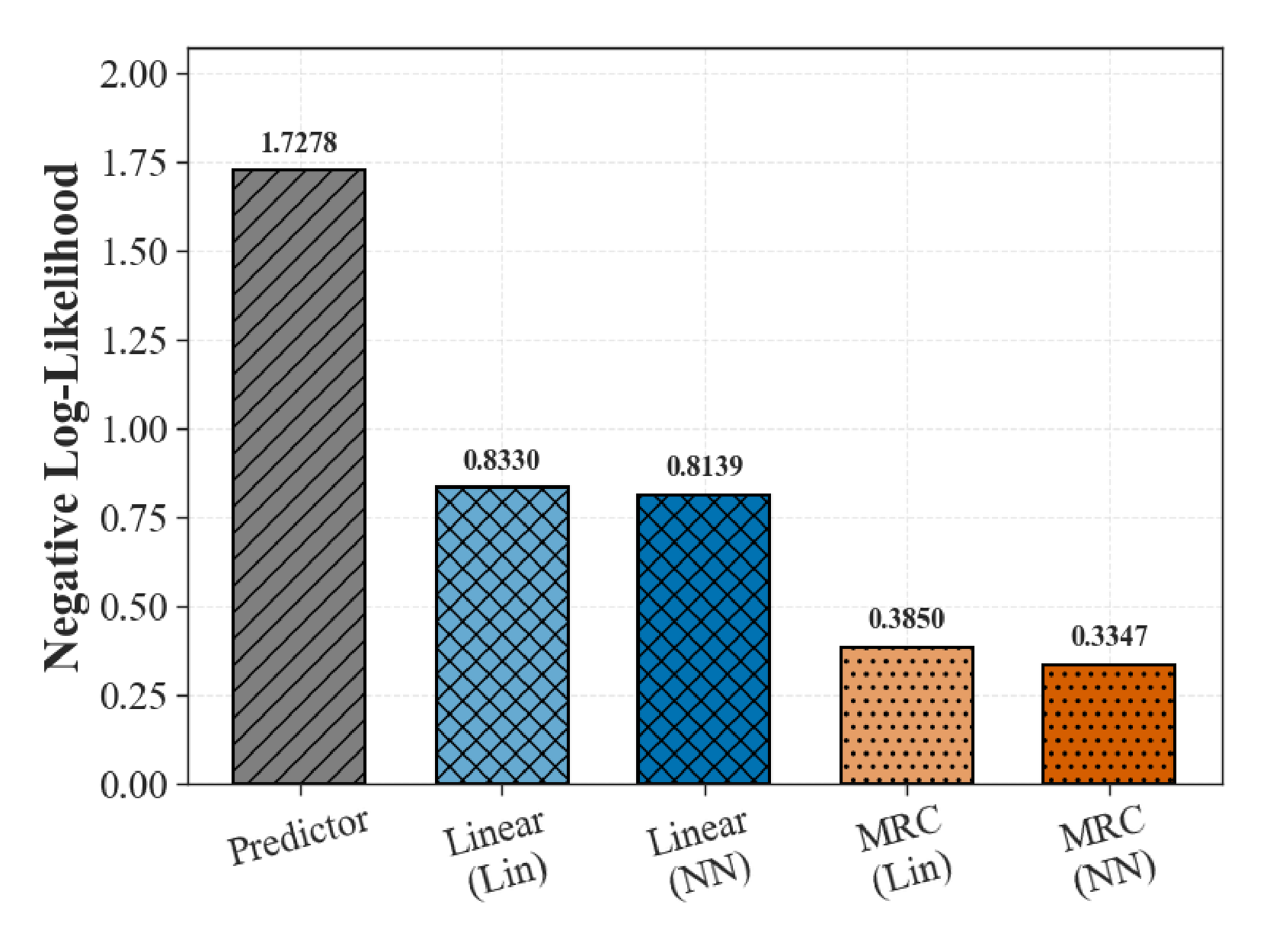}
        \caption{NLL Comparison (lower is better).}
        \label{fig:real_nll}
    \end{minipage}
    \hfill
    \begin{minipage}{0.48\textwidth}
        \centering
        \includegraphics[width=\linewidth]{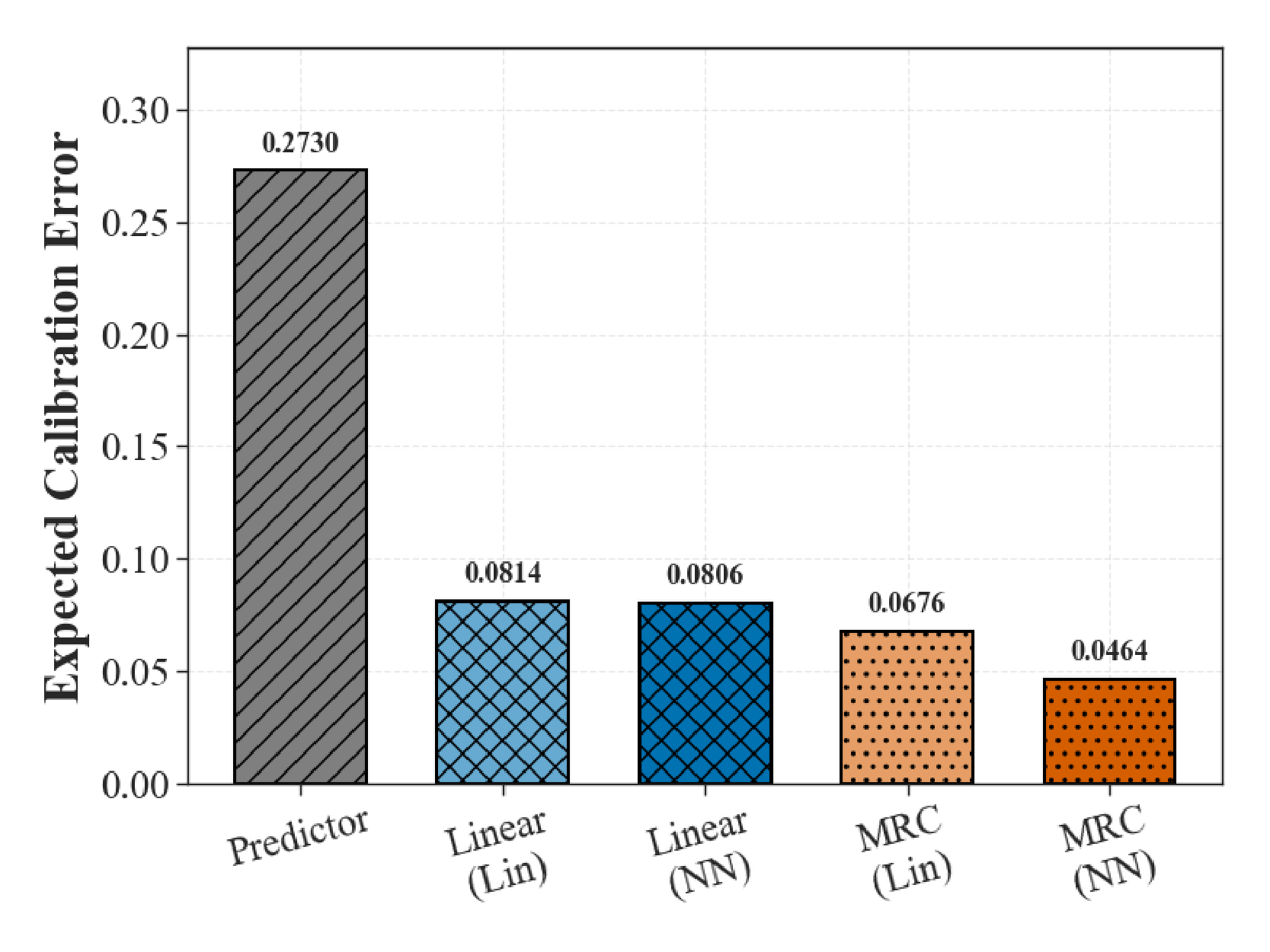}
        \caption{ECE Comparison (lower is better).}
        \label{fig:real_ece}
    \end{minipage}
\end{figure}

\begin{figure}[htbp]
    \centering
    \includegraphics[width=0.55\textwidth]{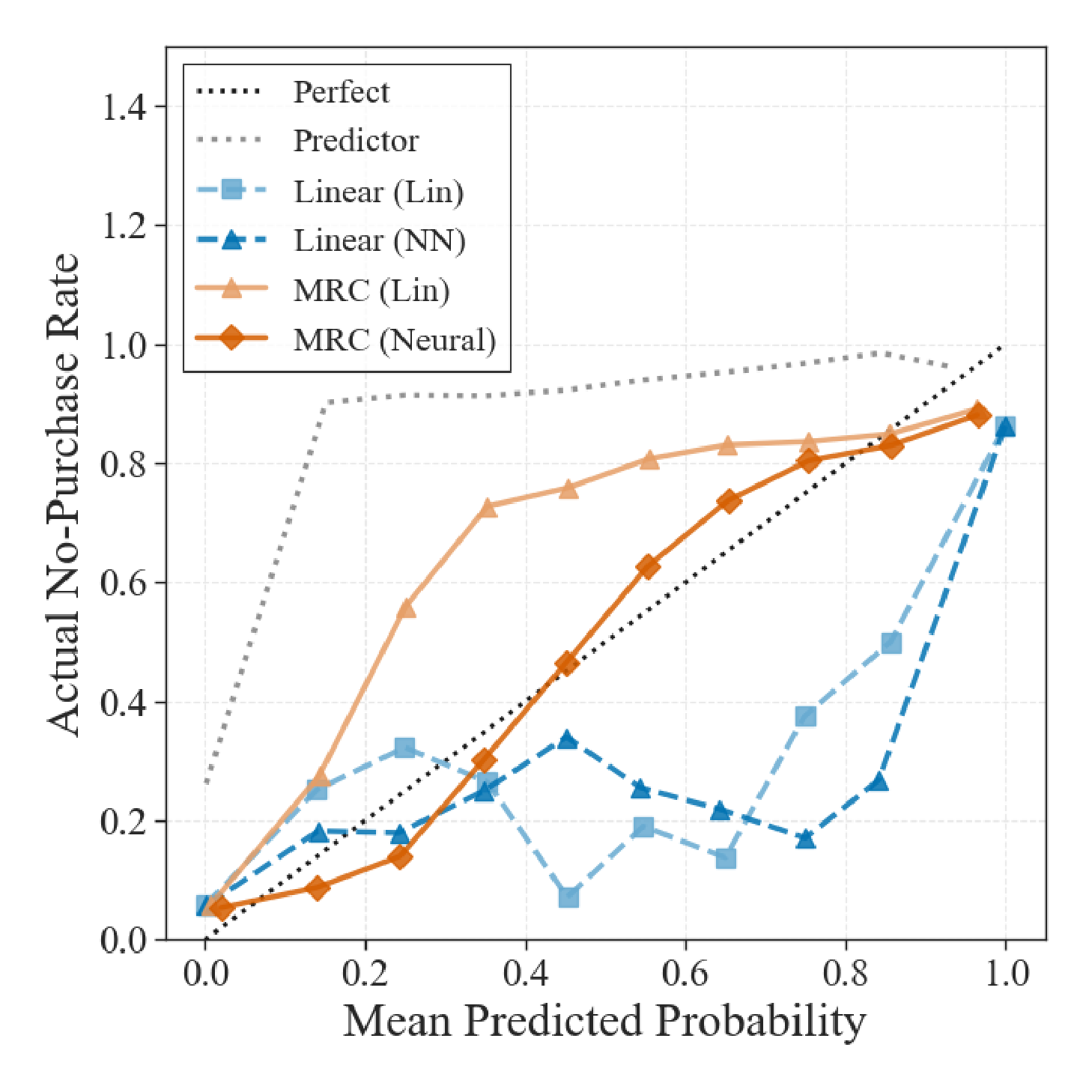}
    \caption{Reliability diagrams for no-purchase prediction on the Current set. The diagonal dashed line indicates perfect calibration; MRC (especially with NN utilities) is closest to the diagonal.}
    \label{fig:real_calibration}
\end{figure}

\section{Conclusion}
\label{sec:conclusion}

This paper studies a longstanding identification problem in choice modeling: firms often observe only realized purchases, while the ``no-purchase'' (outside-option) outcome is unrecorded, confounding market size with preference strength and obstructing demand estimation. We show that recent advances in simulation and generative prediction can be leveraged to resolve this challenge even when the available predictor for no-purchase behavior is biased or miscalibrated. The key is to use the predictor not as ground truth, but as an auxiliary signal that can be structurally corrected.

Our approach rests on a simple but powerful identity implied by general logit-style choice models: the outside-option logit equals the outside utility index minus the inclusive value of the offered items. This decomposition reveals that many common predictor distortions can be undone using only purchase data, provided we can learn the within-assortment utilities of observed items and query the predictor on the same contexts and assortments. Building on this insight, we develop two calibration procedures. In the linearly biased case, a regression-based calibration recovers the true outside-option parameters exactly (in population) and consistently (in sample) without observing any outside-option realizations. In the more general case of nearly monotone predictors, we propose a maximum-rank-correlation calibration that is invariant to unknown nonlinear transformations of the predictor output and remains valid as long as the predictor preserves the ordering of outside-option logits for most samples. For both methods, we provide finite-sample error guarantees that cleanly separate the effect of predictor quality from the effect of first-stage utility learning error, and we show how calibration accuracy translates into downstream decision quality through assortment-optimization suboptimality bounds. We also extend the framework to settings with multiple predictors and discuss robust, scale-free aggregation strategies that can mitigate idiosyncratic predictor noise and guard against adversarial or anti-monotone sources.

Our numerical results support the theory. In synthetic experiments, the proposed estimators exhibit the predicted convergence behavior under their respective bias regimes and remain robust to utility-learning error and predictor noise. In a real-data application using Expedia search logs with a deliberately biased historical predictor, structural calibration substantially improves probabilistic fit and calibration relative to using the predictor outputs directly, and the rank-based approach is especially effective when predictor bias is nonlinear.

Overall, the central message is that imperfect predictors can still enable \emph{accurate inference of unobserved choice} when used in a structure-aware way: by exploiting model identities and appropriate calibration, we can recover missing outside-option probabilities from purchase-only data, turning biased synthetic predictions into reliable demand estimates and decision inputs.

\bibliographystyle{plainnat}
\bibliography{references}

@article{manski1977structure,
  title={The structure of random utility models},
  author={Manski, Charles F},
  journal={Theory and decision},
  volume={8},
  number={3},
  pages={229},
  year={1977},
  publisher={Kluwer Academic Publishers}
}

@article{wang2023transformer,
  title={Transformer choice net: A transformer neural network for choice prediction},
  author={Wang, Hanzhao and Li, Xiaocheng and Talluri, Kalyan},
  journal={arXiv preprint arXiv:2310.08716},
  year={2023}
}

@article{wang2023neural,
  title={A neural network based choice model for assortment optimization},
  author={Wang, Hanzhao and Cai, Zhongze and Li, Xiaocheng and Talluri, Kalyan},
  journal={arXiv preprint arXiv:2308.05617},
  year={2023}
}

@article{li2025small,
  title={From Small to Large: A Graph Convolutional Network Approach for Solving Assortment Optimization Problems},
  author={Li, Guokai and Gao, Pin and Jasin, Stefanus and Wang, Zizhuo},
  journal={arXiv preprint arXiv:2507.10834},
  year={2025}
}

@inproceedings{
zhang2025deep,
title={Deep Context-Dependent Choice Model},
author={Shuhan Zhang and Zhi Wang and Rui Gao and Shuang Li},
booktitle={2nd Workshop on Models of Human Feedback for AI Alignment},
year={2025},
url={https://openreview.net/forum?id=bXTBtUjb0c}
}

@article{yang2025reproducing,
  title={Reproducing Kernel Hilbert Space Choice Model},
  author={Yang, Yiqi and Wang, Zhi and Gao, Rui and Li, Shuang},
  journal={Available at SSRN 5267975},
  year={2025}
}

@inproceedings{tomlinson2021learning,
  title={Learning interpretable feature context effects in discrete choice},
  author={Tomlinson, Kiran and Benson, Austin R},
  booktitle={Proceedings of the 27th ACM SIGKDD conference on knowledge discovery \& data mining},
  pages={1582--1592},
  year={2021}
}

@book{wainwright2019high,
  title={High-dimensional statistics: A non-asymptotic viewpoint},
  author={Wainwright, Martin J},
  volume={48},
  year={2019},
  publisher={Cambridge university press}
}

@book{vershynin2018,
  title={High-dimensional probability: An introduction with applications in data science},
  author={Vershynin, Roman},
  volume={47},
  year={2018},
  publisher={Cambridge university press}
}

@article{keskin2014dynamic,
  title={Dynamic pricing with an unknown demand model: Asymptotically optimal semi-myopic policies},
  author={Keskin, N Bora and Zeevi, Assaf},
  journal={Operations research},
  volume={62},
  number={5},
  pages={1142--1167},
  year={2014},
  publisher={INFORMS}
}

@article{ban2021personalized,
  title={Personalized dynamic pricing with machine learning: High-dimensional features and heterogeneous elasticity},
  author={Ban, Gah-Yi and Keskin, N Bora},
  journal={Management Science},
  volume={67},
  number={9},
  pages={5549--5568},
  year={2021},
  publisher={INFORMS}
}

@inproceedings{guo2017calibration,
  author    = {Guo, Chuan and Pleiss, Geoff and Sun, Yu and Weinberger, Kilian Q.},
  title     = {On calibration of modern neural networks},
  booktitle = {Proceedings of the 34th International Conference on Machine Learning},
  pages     = {1321--1330},
  year      = {2017},
  publisher = {PMLR}
}

@article{abdallah2021demand,
  title     = {Demand Estimation under the Multinomial Logit Model from Sales Transaction Data},
  author    = {Abdallah, Tarek and Vulcano, Gustavo},
  journal   = {Manufacturing \& Service Operations Management},
  year      = {2021},
  volume    = {23},
  number    = {5},
  pages     = {1196--1216},
  doi       = {10.1287/msom.2020.0925}
}

@article{Mackey_2014,
   title={Matrix concentration inequalities via the method of exchangeable pairs},
   volume={42},
   ISSN={0091-1798},
   url={http://dx.doi.org/10.1214/13-AOP892},
   DOI={10.1214/13-aop892},
   number={3},
   journal={The Annals of Probability},
   publisher={Institute of Mathematical Statistics},
   author={Mackey, Lester and Jordan, Michael I. and Chen, Richard Y. and Farrell, Brendan and Tropp, Joel A.},
   year={2014},
   month=may }

@article{han1987mrc,
  title={Non-parametric analysis of a generalized regression model: the maximum rank correlation estimator},
  author={Han, Aaron K},
  journal={Journal of Econometrics},
  volume={35},
  number={2-3},
  pages={303--316},
  year={1987},
  publisher={Elsevier}
}

@article{sherman1993mrc,
  title={The limiting distribution of the maximum rank correlation estimator},
  author={Sherman, Robert P},
  journal={Econometrica: Journal of the Econometric Society},
  pages={123--137},
  year={1993},
  publisher={JSTOR}
}

@inproceedings{park2023generative,
  title={Generative agents: Interactive simulacra of human behavior},
  author={Park, Joon Sung and O'Brien, Joseph and Cai, Carrie Jun and Morris, Meredith Ringel and Liang, Percy and Bernstein, Michael S},
  booktitle={Proceedings of the 36th annual acm symposium on user interface software and technology},
  pages={1--22},
  year={2023}
}

@article{lovasz2007geometry,
  title={The geometry of logconcave functions and sampling algorithms},
  author={Lov{\'a}sz, L{\'a}szl{\'o} and Vempala, Santosh},
  journal={Random Structures \& Algorithms},
  volume={30},
  number={3},
  pages={307--358},
  year={2007},
  publisher={Wiley Online Library}
}

@article{wang2007imo,
  author       = {Hao Wang},
  title        = {A note on iterative marginal optimization: a simple algorithm for maximum rank correlation estimation},
  journal      = {Computational Statistics and Data Analysis},
  year         = {2007},
  volume       = {51},
  number       = {5},
  pages        = {2803--2812}
}

@article{shin2021exact,
  author       = {Yongjin Shin and Valentin Todorov},
  title        = {Exact computation of maximum rank correlation estimator},
  journal      = {The Econometrics Journal},
  year         = {2021},
  volume       = {24},
  number       = {3},
  pages        = {589--607},
}

@article{shen2023lmrc,
  author       = {Guang Shen and Ke Chen and Jian Huang and Yao Lin},
  title        = {Linearized maximum rank correlation estimation},
  journal      = {Biometrika},
  year         = {2023},
  volume       = {110},
  number       = {1},
  pages        = {187--203}
}

@article{talluri_vr_2004_ms,
  title={Revenue management under a general discrete choice model of consumer behavior},
  author={Talluri, Kalyan and Van Ryzin, Garrett},
  journal={Management Science},
  volume={50},
  number={1},
  pages={15--33},
  year={2004},
  publisher={INFORMS}
}

@article{vulcano_vr_chaar_2010_ms,
  title={Om practice—choice-based revenue management: An empirical study of estimation and optimization},
  author={Vulcano, Gustavo and Van Ryzin, Garrett and Chaar, Wassim},
  journal={Manufacturing \& Service Operations Management},
  volume={12},
  number={3},
  pages={371--392},
  year={2010},
  publisher={INFORMS}
}

@article{newman_ferguson_garrow_jacobs_2014_msom,
  title={Estimation of choice-based models using sales data from a single firm},
  author={Newman, Jeffrey P and Ferguson, Mark E and Garrow, Laurie A and Jacobs, Timothy L},
  journal={Manufacturing \& Service Operations Management},
  volume={16},
  number={2},
  pages={184--197},
  year={2014},
  publisher={INFORMS}
}

@techreport{talluri_2009_upf,
  title={A finite-population revenue management model and a risk-ratio procedure for the joint estimation of population size and parameters},
  author={Talluri, Kalyan},
  journal={Available at SSRN 1374853},
  year={2009}
}

@article{vanryzin_vulcano_2015_mnsc,
  author       = {Garrett van Ryzin and Gustavo Vulcano},
  title        = {A Market Discovery Algorithm to Estimate a General Class of Nonparametric Choice Models},
  journal      = {Management Science},
  year         = {2015},
  volume       = {61},
  number       = {2},
  pages        = {281--300},
  doi          = {10.1287/mnsc.2014.2040}
}

@article{subramanian_harsha_2021_mnsc,
  title={Demand modeling in the presence of unobserved lost sales},
  author={Subramanian, Shivaram and Harsha, Pavithra},
  journal={Management Science},
  volume={67},
  number={6},
  pages={3803--3833},
  year={2021},
  publisher={INFORMS}
}

@article{li_talluri_tekin_2025_msom,
  author       = {Anran Li and Kalyan Talluri and M{\"u}ge Tekin},
  title        = {Estimating Demand with Unobserved No-Purchases on Revenue-Managed Data},
  journal      = {Manufacturing \& Service Operations Management},
  year         = {2025},
  volume       = {27},
  number       = {1},
  pages        = {161--180},
  doi          = {10.1287/msom.2021.0291}
}

@article{talluri_tekin_2025_joom,
  author       = {Kalyan Talluri and M{\"u}ge Tekin},
  title        = {Estimation Using Marginal Competitor Sales Information},
  journal      = {Journal of Operations Management},
  year         = {2025},
  note         = {Early View},
  doi          = {10.1002/joom.1359}
}

@article{buggineni2024enhancing,
  title={Enhancing manufacturing operations with synthetic data: a systematic framework for data generation, accuracy, and utility},
  author={Buggineni, Vishnupriya and Chen, Cheng and Camelio, Jaime},
  journal={Frontiers in Manufacturing Technology},
  volume={4},
  pages={1320166},
  year={2024},
  publisher={Frontiers Media SA}
}

@article{chan2022generation,
  title={Generation of synthetic manufacturing datasets for machine learning using discrete-event simulation},
  author={Chan, KC and Rabaev, Marsel and Pratama, Handy},
  journal={Production \& Manufacturing Research},
  volume={10},
  number={1},
  pages={337--353},
  year={2022},
  publisher={Taylor \& Francis}
}

@article{lopes2024digitaltwins,
  title={Synthetic data generation for digital twins: enabling production systems analysis in the absence of data},
  author={Lopes, Paulo Victor and Silveira, Leonardo and Guimaraes Aquino, Roberto Douglas and Ribeiro, Carlos Henrique and Skoogh, Anders and Verri, Filipe Alves Neto},
  journal={International Journal of Computer Integrated Manufacturing},
  volume={37},
  number={10-11},
  pages={1252--1269},
  year={2024},
  publisher={Taylor \& Francis}
}

@misc{pendyala2024semiconductor,
  title={A Benchmark Time Series Dataset for Semiconductor Fabrication Manufacturing Constructed using Component-based Discrete-Event Simulation Models},
  author={Pendyala, Vamsi Krishna and Sarjoughian, Hessam S and Potineni, Bala and Yellig, Edward J},
  journal={arXiv preprint arXiv:2408.09307},
  year={2024}
}

@article{long2025ijpr,
  title={Leveraging synthetic data to tackle machine learning challenges in supply chains: challenges, methods, applications, and research opportunities},
  author={Long, Yunbo and Kroeger, Sebastian and Zaeh, Michael F and Brintrup, Alexandra},
  journal={International Journal of Production Research},
  pages={1--22},
  year={2025},
  publisher={Taylor \& Francis}
}

@article{chatterjee2025ev,
  title={GAN-based synthetic time-series data generation for improving prediction of demand for electric vehicles},
  author={Chatterjee, Subhajit and Hazra, Debapriya and Byun, Yung-Cheol},
  journal={Expert Systems with Applications},
  volume={264},
  pages={125838},
  year={2025},
  publisher={Elsevier}
}

@inproceedings{xu2019ctgan,
    author = {Xu, Lei and Skoularidou, Maria and Cuesta{-}Infante, Alfredo and Veeramachaneni, Kalyan},
  title = {Modeling Tabular data using Conditional {GAN}},
  journal = {CoRR},
  volume = {abs/1907.00503},
  year = {2019},
  url = {http://arxiv.org/abs/1907.00503},
  archiveprefix = {arXiv},
  eprint = {1907.00503},
  bibsource = {dblp computer science bibliography, https://dblp.org}
}

@inproceedings{yoon2019timegan,
  title={Time-series generative adversarial networks},
  author={Yoon, Jinsung and Jarrett, Daniel and Van der Schaar, Mihaela},
  journal={Advances in neural information processing systems},
  volume={32},
  year={2019}
}

@misc{esteban2017rgan,
  title={Real-valued (medical) time series generation with recurrent conditional gans},
  author={Esteban, Crist{\'o}bal and Hyland, Stephanie L and R{\"a}tsch, Gunnar},
  journal={arXiv preprint arXiv:1706.02633},
  year={2017}
}

@inproceedings{patki2016sdv,
  title={The synthetic data vault},
  author={Patki, Neha and Wedge, Roy and Veeramachaneni, Kalyan},
  booktitle={2016 IEEE international conference on data science and advanced analytics (DSAA)},
  pages={399--410},
  year={2016},
  organization={IEEE}
}

@inproceedings{giomi2023anonymeter,
  title={A unified framework for quantifying privacy risk in synthetic data},
  author={Giomi, Matteo and Boenisch, Franziska and Wehmeyer, Christoph and Tasn{\'a}di, Borb{\'a}la},
  journal={arXiv preprint arXiv:2211.10459},
  year={2022}
}

@article{gallegos2024survey,
  author  = {Isabel O. Gallegos and Ryan A. Rossi and Joe Barrow and et al.},
  title   = {Bias and Fairness in Large Language Models: A Survey},
  journal = {Computational Linguistics},
  volume  = {50},
  number  = {3},
  pages   = {1097--1162},
  year    = {2024},
  url     = {https://aclanthology.org/2024.cl-3.8/}
}

@article{chen2025_msom_bias,
  title={A manager and an AI walk into a bar: does ChatGPT make biased decisions like we do?},
  author={Chen, Yang and Kirshner, Samuel N and Ovchinnikov, Anton and Andiappan, Meena and Jenkin, Tracy},
  journal={Manufacturing \& Service Operations Management},
  volume={27},
  number={2},
  pages={354--368},
  year={2025},
  publisher={INFORMS}
}

@article{hanson1971bound,
  title={A bound on tail probabilities for quadratic forms in independent random variables},
  author={Hanson, David Lee and Wright, Farroll Tim},
  journal={The Annals of Mathematical Statistics},
  volume={42},
  number={3},
  pages={1079--1083},
  year={1971},
  publisher={JSTOR}
}

@article{hoeffding1963probability,
  title={Probability inequalities for sums of bounded random variables},
  author={Hoeffding, Wassily},
  journal={Journal of the American statistical association},
  volume={58},
  number={301},
  pages={13--30},
  year={1963},
  publisher={Taylor \& Francis}
}

@misc{expedia-personalized-sort,
  title        = {Expedia Personalized Sort},
  author       = {Expedia / Kaggle},
  year         = {2013},
  howpublished = {\url{https://www.kaggle.com/competitions/expedia-personalized-sort/data}},
  note         = {Accessed: 2025-11-20}
}

@article{prokhorenkova2018catboost,
  title={CatBoost: unbiased boosting with categorical features},
  author={Prokhorenkova, Liudmila and Gusev, Gleb and Vorobev, Aleksandr and Dorogush, Anna Veronika and Gulin, Andrey},
  journal={Advances in neural information processing systems},
  volume={31},
  year={2018}
}

@article{kingma2014adam,
  title={Adam: A method for stochastic optimization},
  author={Kingma, Diederik P},
  journal={arXiv preprint arXiv:1412.6980},
  year={2014}
}

@article{liu1989limited,
  title={On the limited memory BFGS method for large scale optimization},
  author={Liu, Dong C and Nocedal, Jorge},
  journal={Mathematical programming},
  volume={45},
  number={1},
  pages={503--528},
  year={1989},
  publisher={Springer}
}

\appendix

\section{More Discussions}
\subsection{Discussion for Assumption~\ref{as:design-grouped}}
\label{appx:assp_discuss}
These conditions hold under broad sampling schemes for $(X,\mathcal{S})$:

\begin{itemize}[leftmargin=18pt,itemsep=2pt,topsep=2pt]
\item[(i)] \emph{Non-degenerate design} guarantees a nontrivial fraction of pairwise differences must be \emph{bounded away from zero} in every direction. It is implied by many common designs, including sub-Gaussian, log-concave, or bounded-density designs through the small-ball probability (e.g., see the small-ball method and related results in \citet{vershynin2018}). In practice, diverse assortment generation induces spread in both the $s$ and $z(X)$, further supporting this condition.

\item[(ii)] The \emph{lower-tail} bound controls how often pairs are near ties \emph{along the true direction}. It ensures that there is enough probability mass in a shrinking neighborhood of the separating hyperplane $\{w:\theta^{*\top}w=0\}$ to deliver curvature of the (population) rank-correlation risk around $\theta^*$; otherwise the criterion can be locally flat and identification becomes weak. A primitive condition that implies this is that the one-dimensional projection $V:=\theta^{*\top}D^*$ has a density $f_V$ that is continuous at $0$ with $f_V(0)\ge c_0>0$. Then there exists $t_0>0$ such that for all $t\in(0,t_0]$,
\[
\Pr\big(|\theta^{*\top}D^*|\le t\mid \mathcal O\big)
=\Pr(|V|\le t\mid \mathcal O)
\ \ge\ \int_{-t}^t \tfrac{c_0}{2}\,du \ =\ c_0\,t,
\]
so Assumption~\ref{as:design-grouped}(ii) holds with $\alpha_{\mathrm{down}}=1$ and $c_{\mathrm{lt}}=c_0$. This type of ``mass near zero'' requirement is standard in analyses of MRC (e.g., \cite{sherman1993mrc}).

\item[(iii)] \emph{Uniform anti-concentration} limits the fraction of pairs that are extremely close along \emph{any} direction. Intuitively, this condition prevents the design from placing too many pairwise differences almost on a single hyperplane.  It holds, for example, when $D^*$ has an absolutely continuous distribution whose one-dimensional marginals $f_{v^\top D^*}$ are uniformly bounded:
\[
\sup_{v\in\mathbb S^{d}}\ \big\|f_{v^\top D^*}\big\|_\infty \ \le\ M \ <\ \infty.
\]
In that case, for all unit $v$ and all $t\in(0,1]$,
\[
\Pr\big(|v^\top D^*|\le t\mid \mathcal O\big)
\ \le\ \int_{-t}^{t} M\,du \ =\ 2Mt,
\]
so Assumption~\ref{as:design-grouped}(iii) holds with $\alpha_{\mathrm{up}}=1$ and $L_{\mathrm{ac}}=2M$. This condition is also satisfied by several common distribution classes:
\begin{enumerate}[leftmargin=14pt,itemsep=1pt]
\item \textbf{Gaussian (or sub-Gaussian) designs:} If $D^*\sim\mathcal N(0,\Sigma)$ with $\lambda_{\min}(\Sigma)\ge \sigma_0^2>0$, then $v^\top D^*\sim\mathcal N(0, v^\top\Sigma v)$ has density bounded by $(\sqrt{2\pi}\,\sigma_0)^{-1}$ for all unit $v$, hence $L_{\mathrm{ac}}=\sqrt{\tfrac{2}{\pi}}\,\sigma_0^{-1}$ and $\alpha_{\mathrm{up}}=1$. The same linear anti-concentration (with different constants) extends to general sub-Gaussian designs with non-degenerate covariance \citep{vershynin2018}.
\item \textbf{Log-concave designs:} If $D^*$ has a log-concave density and $\lambda_{\min}(\Var(D^*))\ge \sigma_0^2>0$, then one-dimensional marginals are log-concave with uniformly bounded densities (after isotropic normalization), yielding the same linear bound with constants depending only on $\sigma_0$ \citep{lovasz2007geometry}.
\end{enumerate}

\item[(iv)] \emph{Association between design and truth.} This condition acts as a crucial link between the global signal from Assumption~\ref{as:design-grouped} (i) and the local curvature from Assumption~\ref{as:design-grouped} (ii). Intuitively, it prevents a pathological scenario where the set of ``near-tie'' pairs (event $B(t)$), which are critical for local identification of $\theta^*$, and the set of ``well-spread'' pairs (event $A(v)$), which are essential for global identifiability, are disjoint.

Notice that, by the symmetry of $D^*$, these two one-sided events $A'(v)=\{v^{\top}D^*\ge\tau_{nd}\}$ and $B'(t)=\{t\le\theta^{*\top}D^*<0\}$, which also satisfy the lower-bound property: 
$$\text{Pr}\big(A'(v)\cap B'(t)\mid \mathcal{O} \big)\ge c_{assoc}\cdot \text{Pr}\big(A'(v)\mid\mathcal{O}\big)\cdot \text{Pr}\big(B'(t)\mid\mathcal{O}\big)$$. 

Also notice that, we cannot induce that two-sided events $A(v)\cup A'(v)$ and $B(t)\cup B'(t)$ directly from Assumption ~\ref{as:design-grouped} (iv).

This assumption provides the essential link between the \textbf{global repulsive force} established by Assumption~\ref{as:design-grouped} (i) and the \textbf{local attractive force} from Assumption~\ref{as:design-grouped} (ii) on both $\{A(v),B(t)\}$ and $\{A'(v),B'(t)\}$ separately. Assumption~\ref{as:design-grouped} (ii) ensures that a pool of ``near-tie'' pairs exists to create a sharp optimum at $\theta^*$ (attraction). Assumption~\ref{as:design-grouped} (i) ensures that for any wrong direction $v$, a large set of ``well-spread'' pairs generates a steep loss landscape (repulsion). Assumption~\ref{as:design-grouped} (iv) guarantees that these two roles are not performed by mutually exclusive sets of data points. It requires that the critical ``near-tie'' pairs contributing to the local attraction at $\theta^*$ on the one side (positive or negative side) must also, on average, contribute effectively to the repulsive forces in other directions $v$ on the opposite side, rather than being silent or degenerate. This synergy, as formalized in Proposition~\ref{prop:margin}, is what ensures a quantifiable local margin and thus a robust convergence rate.

This type of association is mild and satisfied by many standard continuous distributions. For instance, if $D^*$ follows a non-degenerate multivariate Gaussian distribution $\mathcal{N}(0, \Sigma)$, then for any two unit vectors $v$ and $\theta^*$, the projections $W = v^\top D^*$ and $U = \theta^{*\top} D^*$ are jointly Gaussian. Unless $v$ and $\theta^*$ are perfectly collinear or $\Sigma$ is degenerate, their correlation is strictly less than 1. The joint probability can be analyzed via conditioning:
\[
\Pr\big(A(v) \cap B(t)\big) = \int_{0}^{t} \Pr\big(W \ge \tau_{\mathrm{nd}} \mid U=u\big) f_U(u) du.
\]
Since the conditional variance of $W$ given $U$, the conditional probability $\Pr(W \ge \tau_{\mathrm{nd}} \mid U=u)$ is uniformly bounded away from zero for $u$ in a small interval $[0, t]$, and the density $f_U(u)$ is positive around $u=0$ by a primitive for Assumption~\ref{as:design-grouped} (ii). Consequently, the integral is positive and of the same order as $\Pr(B(t))$. This ensures a valid $c_{\mathrm{assoc}} > 0$. The argument extends to other well-behaved distributions like log-concave densities where strong negative dependence between projections is excluded.

The necessity of this assumption is underscored by considering a counterexample where the data geometry decouples global and local signals. Consider a distribution for $D^*$ composed of two disjoint components in a high-dimensional space. Let the true direction be $\theta=e_1$ and a test direction be $v=e_2$.

The first component satisfies the global repulsion condition $A(v)$: it is a probability mass concentrated far along the negative $e_2$ axis (e.g., around $-Ke_2$ for large $K$), ensuring $v^\top D^*\le -\tau_{nd}$, but positioned far from the decision boundary such that $\theta^{*\top}D^*>t$.

The second component satisfies the local attraction condition $B(t)$: it is a ``flat'' distribution (e.g., a disk) located strictly within the slice $0 < x_1 \le t$, but lying entirely in the subspace orthogonal to $e_2$ (i.e., $x_2=0$). In this scenario, points in the second component satisfy the ``near-tie'' condition $B(t)$, but since $v^{\top}D^*=0>-\tau_{nd}$, they contribute zero probability to $A(v)$. Conversely, points in the first component satisfy $A(v)$ but lie outside $B(t)$. Consequently, the intersection $A(v)\cap B(t)$ is empty, and the left-hand side of the assumption is zero, while the right-hand side is positive, violating Assumption~\ref{as:design-grouped} (iv). By simply reversing the signs of the coordinates in this construction (placing mass on the positive $e_2$ axis and in the negative $x_1$ strip), the same pathological decoupling can be demonstrated for the symmetric events $A'(v)$ and $B'(t)$.
\end{itemize}

\subsection{Utility Learning from Observed Choices under a Linear Utility Model}
\label{appx: obs_utility_learn}

This subsection quantifies how utility-learning error contributes to the calibration error used in Theorems~\ref{thm:finite-sample} and~\ref{thm:assortment-regret}. Recall the quantity
\[
\bar{\tau} \;=\; \frac{1}{n}\sum_{k=1}^n \big(\hat s_k - s_k\big)^2,
\qquad
s_k \;=\; s(X_k,\mathcal{S}_k)\;=\;\log\!\sum_{i\in\mathcal{S}_k}\exp\!\big(u_i(X_k)\big),
\]
where $s_k$ is computed from the (unknown) true utilities $\{u_i(\cdot)\}$ and $\hat s_k$ from learned utilities $\{\hat u_i(\cdot)\}$. We study the case in which item utilities are linear in features and are learned from observed choices via conditional MNL.

We remark that the estimation rate derived in this subsection uses standard M-estimation arguments
(covering/union bounds, matrix concentration, and local strong convexity); see, e.g.,
\citet{wainwright2019high,vershynin2018} for textbook treatments of these tools. We include the derivation here only to connect purchase-only utility learning to the inclusive-value error
$\bar\tau=\frac1n\sum_{k=1}^n(\hat s_k-s_k)^2$ that appears in our calibration guarantees.

\paragraph{Setup.}
For each sample $k\in\{1,\dots,n\}$ we observe a choice set $\mathcal{S}_k$, an item-feature collection $X_k=\{x_{k,i}\in\R^p:i\in\mathcal{S}_k\}$, and the chosen item $i_k\in\mathcal{S}_k$. Assume item-level utilities are linear,
\begin{equation}\label{eq:lin-u}
u_i(X_k) \;=\; \beta^{*\top}x_{k,i},\qquad i\in\mathcal{S}_k,
\end{equation}
for an unknown parameter $\beta^*\in\R^p$. Let $\hat\beta$ be the MLE for the conditional MNL likelihood
\[
\cL_n(\beta)\;:=\;-\frac{1}{n}\sum_{k=1}^n\Big\{\beta^\top x_{k,i_k}-\log\!\sum_{j\in\mathcal{S}_k}\exp\big(\beta^\top x_{k,j}\big)\Big\},
\]
and define $\hat u_i(X_k)=\hat\beta^\top x_{k,i}$ and $\hat s_k=\log\!\sum_{i\in\mathcal{S}_k}\exp(\hat u_i(X_k))$.

\paragraph{Main result (bound on $\bar{\tau}$).}
Under standard regularity conditions for MNL (bounded features and set sizes, i.i.d.\ sampling, and positive curvature of the population Fisher information; see the auxiliary subsection below for details and proofs), the inclusive-value error satisfies the following high-probability bound:
\begin{equation}\label{eq:bartau-linear}
\bar{\tau}
\;\le\;
\frac{C\,B_x^4}{\lambda_x^2}\,\frac{p+\log(1/\delta)}{n}
\qquad \text{with probability at least } 1-\delta,
\end{equation}
for a universal constant $C>0$, where $B_x$ bounds feature norms and $\lambda_x$ lower-bounds the population curvature at $\beta^*$. In particular,
\[
\E[\bar{\tau}] \;\lesssim\; \frac{B_x^4}{\lambda_x^2}\,\frac{p}{n}
\quad\text{and}\quad
\bar{\tau}\;=\;O\!\left(\frac{p}{n}\right).
\]
Equation~\eqref{eq:bartau-linear} directly provides the scaling for the $\bar{\tau}$ term that enters Theorems~\ref{thm:finite-sample} and~\ref{thm:assortment-regret}: faster utility learning (larger $n$ or better-conditioned design, i.e., larger $\lambda_x$ and smaller $B_x$) yields sharper calibration. The formal result is below:

\begin{theorem}[Inclusive-value error bound]\label{thm:bartau}
Under Assumptions~\ref{as:design}--\ref{as:sub-gaussian-gradient}, for any $\delta\in(0,1)$, with probability at least $1-\delta$,
\[
\bar{\tau}
\;=\;\frac{1}{n}\sum_{k=1}^n(\hat s_k-s_k)^2
\;\le\;
B_x^2\,\|\hat\beta-\beta^*\|_2^2
\;\le\;
\frac{C\,B_x^4}{\lambda_x^2}\,\frac{p+\log(1/\delta)}{n},
\]
for a universal constant $C>0$.
\end{theorem}

\paragraph{Auxiliary assumptions and proofs.}
For convenience, define 
$$l(\beta;(X,Y))=-\beta^\top x_Y + \log\!\Big(\sum_{k\in S}\exp(\beta^\top x_k)\Big)$$ be the loss of a data point.
\begin{assumption}[Bounded design and finite set size]\label{as:design}
There exists $B_x<\infty$ such that $\|x_{k,i}\|_2\le B_x$ for all $k$ and $i\in\mathcal{S}_k$, and $\sup_k|\mathcal{S}_k|=m_{\max}<\infty$ (MNL is normalized in the usual way).
\end{assumption}

\begin{assumption}[Population curvature]\label{as:curv}
Let $H(\beta):=\E[\nabla^2\cL_n(\beta)]$ denote the population Hessian. There exists $\lambda_x>0$ such that $\lambda_{\min}\!\big(H(\beta^*)\big)\ge \lambda_x$.
\end{assumption}

\begin{assumption}[i.i.d.\,and correct specification]\label{as:iid}
The tuples $\{(X_k,\mathcal{S}_k,i_k)\}_{k=1}^n$ are i.i.d., and~\eqref{eq:lin-u} with MNL choice probabilities is correctly specified.
\end{assumption}

\begin{assumption}[Parameter domain]\label{as:domain}
The estimate $\hat\beta$ is computed over a compact, convex set $\mathcal{B}\subset\R^p$ that contains $\beta^*$, such that $\|\beta\|_2 \le R_B$ for all $\beta \in \mathcal{B}$.
\end{assumption}

\begin{assumption}[Sub-Gaussianity of gradient]\label{as:sub-gaussian-gradient}
Let $g(\beta;Z) = \nabla_{\beta}l(\beta;Z)$ be the single-sample gradient of the conditional negative log-likelihood $l(\beta;Z)$. For any unit vector $u\in\R^p$, the random variable $u^\top\big(g(\beta^*;Z) - \mathbb{E}\,g(\beta^*;Z)\big)$ is $\sigma^2$-sub-Gaussian.
\end{assumption}

\begin{lemma}[Softmax Lipschitz property]\label{lem:lipschitz-s}
For any fixed $(X_k,\mathcal{S}_k)$ and any $\beta_1,\beta_2\in\R^p$,
\[
\Big|\,\log\!\sum_{i\in\mathcal{S}_k}\!\exp(\beta_1^\top x_{k,i})
-\log\!\sum_{i\in\mathcal{S}_k}\!\exp(\beta_2^\top x_{k,i})\,\Big|
\,\le\,
B_x\,\|\beta_1-\beta_2\|_2.
\]
\end{lemma}
\emph{Proof.} The gradient of $\beta\mapsto \log\sum_{i\in\mathcal{S}_k}\exp(\beta^\top x_{k,i})$ is $\sum_{i\in\mathcal{S}_k}\pi_{k,i}(\beta)\,x_{k,i}$ with $\pi_{k,i}$ the softmax weights. Thus, its norm is at most $B_x$ by Assumption~\ref{as:design}. Then apply the mean-value theorem. 

\begin{lemma}[Gradient boundedness and Lipschitz continuity]\label{lem:grad-lip}
Under Assumption~\ref{as:design}, for the single-sample loss 
\[
l(\beta;(S,Y)) = -\beta^\top x_Y + \log\!\Big(\sum_{k\in S}\exp(\beta^\top x_k)\Big),
\]
its gradient $g(\beta;(S,Y)) := \nabla_\beta l(\beta;(S,Y))$ satisfies:
\begin{enumerate}[label=(\roman*)]
    \item \textbf{Bounded gradient:} $\|g(\beta;(S,Y))\|_2 \le 2B_x$ for all $\beta \in \mathcal{B}$.
    \item \textbf{Lipschitz gradient:} $\|H(\beta;(S,Y))\|_{\mathrm{op}} \le B_x^2$, 
    where $H(\beta;(S,Y)) := \nabla_\beta g(\beta;(S,Y))$ denotes the Hessian.
\end{enumerate}
Consequently, the gradient mapping $\beta \mapsto g(\beta;(S,Y))$ is $B_x^2$-Lipschitz.
\end{lemma}

\begin{proof}
From Assumption~\ref{as:design}, all feature vectors are bounded by $\|x_k\|_2 \le B_x$.

\textbf{Step 1: Gradient expression and boundedness.}
Differentiating the loss gives
\[
\nabla_\beta l(\beta;(S,Y))
= -\!\left[x_Y - \frac{\sum_{k\in S}\exp(\beta^\top x_k)\,x_k}
{\sum_{k'\in S}\exp(\beta^\top x_{k'})}\right]
= -\!\left(x_Y - \sum_{k\in S}\pi_k(S,\beta)x_k\right),
\]
where $\pi_k(S,\beta) := 
\frac{\exp(\beta^\top x_k)}{\sum_{k'\in S}\exp(\beta^\top x_{k'})}$ satisfies $\sum_{k\in S}\pi_k(S,\beta)=1$.
Hence,
\[
\|g(\beta;(S,Y))\|_2 
= \big\|x_Y - \sum_{k\in S}\pi_k(S,\beta)x_k\big\|_2
\le \|x_Y\|_2 + \sum_{k\in S}\pi_k(S,\beta)\|x_k\|_2
\le 2B_x.
\]

\textbf{Step 2: Hessian expression and operator-norm bound.}
Taking one more derivative,
\[
H(\beta;(S,Y))
= \nabla_\beta g(\beta;(S,Y))
= \nabla_\beta\!\big(-x_Y + \mathbb{E}_{\pi(S,\beta)}[x]\big)
= \nabla_\beta \mathbb{E}_{\pi(S,\beta)}[x].
\]
The derivative of the softmax expectation $\mu(\beta):=\mathbb{E}_{\pi(S,\beta)}[x]$
satisfies
\[
D(\mu(\beta))[h]
= \sum_{k\in S} D\pi_k(\beta)[h]\,x_k
= \sum_{k\in S} \pi_k(\beta)\,h^\top(x_k - \mu(\beta))\,x_k
= h^\top\!\Big(\sum_{k\in S}\pi_k(\beta)x_kx_k^\top - \mu(\beta)\mu(\beta)^\top\Big)
\]
By the formula of Hessian, we get
\[
D(\mu(\beta))[h] = h^\top\! H(\beta;(S,Y)).
\]
Thus,
\[
H(\beta;(S,Y))
= \sum_{k\in S}\pi_k(S,\beta)x_kx_k^\top
- \mu(\beta)\mu(\beta)^\top,
\quad \text{where } \mu(\beta)=\sum_{k\in S}\pi_k(S,\beta)x_k.
\]
Both terms are positive semi-definite, hence $H(\beta;(S,Y))$ is the covariance matrix of $x_k$ under $\pi_k(S,\beta)$.  
Therefore,
\[
\|H(\beta;(S,Y))\|_{\mathrm{op}}
\le \Big\|\sum_{k\in S}\pi_k(S,\beta)x_kx_k^\top\Big\|_{\mathrm{op}}
\le \sum_{k\in S}\pi_k(S,\beta)\|x_kx_k^\top\|_{\mathrm{op}}
\le B_x^2,
\]
since $\|x_kx_k^\top\|_{\mathrm{op}}=\|x_k\|_2^2\le B_x^2$.

\textbf{Step 3: Lipschitz continuity.}
By the mean-value form of the gradient difference,
\[
\|g(\beta)-g(\beta')\|_2 
\le \sup_{\tilde\beta\in[\beta,\beta']}\|H(\tilde\beta)\|_{\mathrm{op}}\,\|\beta-\beta'\|_2
\le B_x^2\|\beta-\beta'\|_2,
\]
proving $B_x^2$-Lipschitz continuity.

\end{proof}

\begin{lemma}[Hessian Lipschitz continuity]\label{lem:hessian-lip}
Under Assumption~\ref{as:design}, the single-sample Hessian
\[
H(\beta;(S,Y)) \;=\; \nabla^2_\beta l(\beta;(S,Y))
\]
is Lipschitz in $\beta$ with operator-norm constant $8B_x^3$.  That is, for all $\beta,\beta'\in\mathcal{B}$,
\[
\big\|H(\beta;(S,Y)) - H(\beta';(S,Y))\big\|_{\mathrm{op}}
\le 8 B_x^3 \,\|\beta-\beta'\|_2.
\]
\end{lemma}

\begin{proof}
We follow a directional-derivative (mean-value) approach.  Fix a sample $(S,Y)$ and write
\[
H(\beta) := H(\beta;(S,Y)),\qquad \pi_k(\beta):=\pi_k(S,\beta),\qquad \mu(\beta):=\sum_{k\in S}\pi_k(\beta)x_k.
\]

By the softmax identities,
\[
H(\beta)=\sum_{k\in S}\pi_k(\beta)\,(x_k-\mu(\beta))(x_k-\mu(\beta))^\top.
\]

For any direction $h\in\R^p$, compute the directional derivative $D(H(\beta))[h]$.
Using the product rule and the derivative of $\pi_k(\beta)$ (softmax),
\[
D\pi_k(\beta)[h] = \pi_k(\beta)\,h^\top(x_k-\mu(\beta)),
\]
and
\[
D\mu(\beta)[h] \;=\; \sum_{k\in S} D\pi_k(\beta)[h]\,x_k
= \sum_{k\in S}\pi_k(\beta)\,h^\top(x_k-\mu(\beta))\,x_k.
\]
One can verify (by rearrangement) that
\[
D\mu(\beta)[h] \;=\; \big[h^\top(\sum_{k\in S}\pi_k(\beta)x_kx_k^\top - \mu(\beta)\mu(\beta)^\top)\big]^\top
= H(\beta)\,h.
\]

Now apply the product rule to $H(\beta)=\sum_k\pi_k(\beta)(x_k-\mu)(x_k-\mu)^\top$.
We obtain two contributions,
\[
D(H(\beta))[h]
= \underbrace{\sum_{k\in S} D\pi_k(\beta)[h]\,(x_k-\mu)(x_k-\mu)^\top}_{=:T_1}
+ \underbrace{\sum_{k\in S}\pi_k(\beta)\,D\big((x_k-\mu)(x_k-\mu)^\top\big)[h]}_{=:T_2}.
\]

Compute each term. First,
\[
T_1 = \sum_{k\in S} \pi_k(\beta)\,(h^\top(x_k-\mu))\,(x_k-\mu)(x_k-\mu)^\top.
\]

Second, since $D(x_k-\mu)[h] = -D\mu(\beta)[h] = -H(\beta)h$, we have
\[
D\big((x_k-\mu)(x_k-\mu)^\top\big)[h]
= -(H(\beta)h)(x_k-\mu)^\top - (x_k-\mu)(H(\beta)h)^\top.
\]
Therefore
\[
T_2 = -\sum_{k\in S}\pi_k(\beta)\big( H(\beta)h\,(x_k-\mu)^\top + (x_k-\mu)\,(H(\beta)h)^\top\big).
\]
But by the definition of $\mu(\beta)$,
\[
\sum_{k\in S}\pi_k(\beta)(x_k-\mu(\beta)) = 0,
\]
so both summands vanish and hence $T_2=0$.

Thus the directional derivative reduces to
\[
D(H(\beta))[h] \;=\; T_1
= \sum_{k\in S}\pi_k(\beta)\,(h^\top(x_k-\mu))\,(x_k-\mu)(x_k-\mu)^\top.
\]

Now bound the operator norm of $D(H(\beta))[h]$. Using the operator norm bound
$\|uv^\top\|_{\mathrm{op}}=\|u\|_2\|v\|_2$ for rank-one matrices and Cauchy--Schwarz,
\begin{align*}
\big\|D(H(\beta))[h]\big\|_{\mathrm{op}}
&\le \sum_{k\in S}\pi_k(\beta)\,|h^\top(x_k-\mu)|\,\|(x_k-\mu)(x_k-\mu)^\top\|_{\mathrm{op}} \\
&= \sum_{k\in S}\pi_k(\beta)\,|h^\top(x_k-\mu)|\,\|x_k-\mu\|_2^2 \\
&\le \|h\|_2 \sum_{k\in S}\pi_k(\beta)\,\|x_k-\mu\|_2^3 \\
&\le \|h\|_2\; \big(\max_{k\in S}\|x_k-\mu\|_2\big)\; \sum_{k\in S}\pi_k(\beta)\,\|x_k-\mu\|_2^2.
\end{align*}
Using $\|x_k\|_2\le B_x$ and $\|\mu\|_2\le B_x$, we have $\|x_k-\mu\|_2\le 2B_x$, so
\[
\big\|D(H(\beta))[h]\big\|_{\mathrm{op}} \le \|h\|_2 \cdot (2B_x)\cdot(2B_x)^2 = 8 B_x^3 \|h\|_2.
\]
Hence for every $\beta$ and direction $h$,
\[
\|D(H(\beta))[h]\|_{\mathrm{op}} \le 8 B_x^3 \|h\|_2.
\]

Finally, represent the Hessian difference by integrating directional derivatives along the line segment from $\beta'$ to $\beta$:
\[
H(\beta)-H(\beta') \;=\; \int_0^1 D\big(H(\beta' + t(\beta-\beta'))\big)[\beta-\beta']\,dt,
\]
and take operator norms to obtain
\[
\|H(\beta)-H(\beta')\|_{\mathrm{op}}
\le \int_0^1 \big\|D(H(\beta' + t(\beta-\beta')))[\beta-\beta']\big\|_{\mathrm{op}} \,dt
\le \int_0^1 8B_x^3 \|\beta-\beta'\|_2 \,dt
= 8B_x^3 \|\beta-\beta'\|_2,
\]
which proves the claimed Lipschitz bound.
\end{proof}

\begin{lemma}[Uniform concentration of empirical risk]\label{lem:uniform-empirical-risk}
Suppose Assumptions~\ref{as:design}--\ref{as:sub-gaussian-gradient} hold.
Let \( \mathcal{L}(\beta)=\E l(\beta;(S,Y)) \) and \( \mathcal{L}_n(\beta)=\frac1n\sum_{i=1}^n l(\beta;(S_i,Y_i)) \).
Then with probability at least \(1-\delta/3\),
\[
\sup_{\beta\in \mathcal{B}} |\mathcal{L}_n(\beta)-\mathcal{L}(\beta)|
\;\le\;
M_l\sqrt{\frac{2\!\left(p\log\!\frac{3R_B}{\varepsilon_1}+\log\frac{6}{\delta}\right)}{n}}
	+ 4B_x\varepsilon_1,
\]
where \( M_l=2R_BB_x+\log m_{\max} \).
\end{lemma}

\begin{proof}

\textbf{Step 1: Construct an \(\varepsilon_1\)-net.}  
Let \(\mathcal{B} \subset \R^p\) with radius \(R_B\). There exists an \(\varepsilon_1\)-net \(\mathcal{N}_{\varepsilon_1}\) of \(\mathcal{B}\) with cardinality
\[
|\mathcal{N}_{\varepsilon_1}| \le \Big(\frac{3 R_B}{\varepsilon_1}\Big)^p.
\]

\textbf{Step 2: Hoeffding bound on the net points.}  
For any \(\beta \in \mathcal{N}_{\varepsilon_1}\), by boundedness of the loss \(l(\beta;(S,Y))\in [0, M_l]\) (Assumption~\ref{as:design} + Lipschitz),
\[
\Pr\Big(|\mathcal{L}_n(\beta)-\mathcal{L}(\beta)| \ge t \Big) \le 2 \exp\Big(- \frac{2 n t^2}{M_l^2}\Big).
\]

\textbf{Step 3: Union bound over the net.}  
Applying the union bound over \(|\mathcal{N}_{\varepsilon_1}|\) points,
\[
\Pr\Big(\exists \beta\in \mathcal{N}_{\varepsilon_1}: |\mathcal{L}_n(\beta)-\mathcal{L}(\beta)| \ge t \Big) 
\le 2 \Big(\frac{3 R_B}{\varepsilon_1}\Big)^p \exp\Big(- \frac{2 n t^2}{M_l^2}\Big).
\]
Choose \(t = M_l \sqrt{\frac{2\big(p \log \frac{3 R_B}{\varepsilon_1} + \log \frac{6}{\delta}\big)}{n}}\) to make the right-hand side at most \(\delta/3\).

\textbf{Step 4: Extend to all \(\beta \in \mathcal{B}\).}  
For any \(\beta \in \mathcal{B}\), let \(\tilde\beta \in \mathcal{N}_{\varepsilon_1}\) be the nearest net point. By the Lipschitz property of the loss,
\[
|\mathcal{L}_n(\beta)-\mathcal{L}_n(\tilde\beta)| \le L \|\beta - \tilde\beta\|_2 \le L \varepsilon_1,
\quad
|\mathcal{L}(\beta)-\mathcal{L}(\tilde\beta)| \le L \varepsilon_1.
\]
Hence,
\[
|\mathcal{L}_n(\beta)-\mathcal{L}(\beta)| \le |\mathcal{L}_n(\tilde\beta)-\mathcal{L}(\tilde\beta)| + 2 L \varepsilon_1 \le t + 2 L \varepsilon_1.
\]

\textbf{Step 5: Combine constants.}  
Set \(L = 2 B_x\) (from Lemma~\ref{lem:grad-lip}) and absorb into \(4B_x \varepsilon_1\) term. Then with probability at least \(1-\delta/3\),
\[
\sup_{\beta\in \mathcal{B}} |\mathcal{L}_n(\beta)-\mathcal{L}(\beta)| 
\le t + 4 B_x \varepsilon_1,
\]
which proves the claimed bound.
\end{proof}

\begin{theorem}[Crude estimation error bound for $\hat\beta$]\label{thm:crude-beta-error}
Under Assumptions~\ref{as:design}--\ref{as:sub-gaussian-gradient}, with probability at least $1-\delta/3$,
\[
\|\hat\beta-\beta^*\|_2
\;\le\;
\sqrt{\frac{4}{\lambda_x}}\,
\Bigg(
M_l\sqrt{\frac{2\Big(p\log\frac{3R_B}{\varepsilon_1}+\log\frac{6}{\delta}\Big)}{n}}
	+ 4B_x\varepsilon_1
\Bigg)^{1/2}.
\]
In particular, 
\[
\|\hat\beta-\beta^*\|_2 = O_p\Big((p \log n / n)^{1/4}\Big).
\]
\end{theorem}

\begin{proof}
By the definition of the MLE $\hat\beta$, we have $\mathcal{L}_n(\hat\beta) \le \mathcal{L}_n(\beta^*)$, where $\mathcal{L}_n(\beta) = \frac{1}{n} \sum_{k=1}^n l(\beta; Z_k)$.  
Using the second-order Taylor expansion and the strong convexity of the population risk (Assumption~\ref{as:curv}), for some $\tilde\beta$ on the line segment between $\hat\beta$ and $\beta^*$,
\[
\mathcal{L}_n(\hat\beta) - \mathcal{L}_n(\beta^*)
= (\hat\beta - \beta^*)^\top \nabla \mathcal{L}_n(\beta^*)
+ \frac{1}{2} (\hat\beta - \beta^*)^\top \nabla^2 \mathcal{L}_n(\tilde\beta) (\hat\beta - \beta^*).
\]

By Lemma~\ref{lem:uniform-empirical-risk}, with probability at least $1-\delta/3$,
\[
\sup_{\beta \in \mathcal{B}} \big| \mathcal{L}_n(\beta) - \mathcal{L}(\beta) \big|
\le
M_l \sqrt{\frac{2(p \log \frac{3R_B}{\varepsilon_1} + \log \frac{6}{\delta})}{n}}
+ 4 B_x \varepsilon_1.
\]

Since $\mathcal{L}(\beta)$ is $\lambda_x$-strongly convex, we have
\[
\mathcal{L}(\hat\beta) - \mathcal{L}(\beta^*) \ge \frac{\lambda_x}{2} \|\hat\beta - \beta^*\|_2^2.
\]

Combining the above with the concentration bound, we obtain
\[
\frac{\lambda_x}{2} \|\hat\beta - \beta^*\|_2^2
\le 2 \sup_{\beta \in \mathcal{B}} \big| \mathcal{L}_n(\beta) - \mathcal{L}(\beta) \big|.
\]

Taking square roots gives
\[
\|\hat\beta - \beta^*\|_2
\le \sqrt{\frac{4}{\lambda_x}} \,
\Bigg(
M_l \sqrt{\frac{2(p \log \frac{3R_B}{\varepsilon_1} + \log \frac{6}{\delta})}{n}}
+ 4 B_x \varepsilon_1
\Bigg)^{1/2},
\]
which establishes the crude estimation error bound.
\end{proof}

\begin{lemma}[Local uniform control of empirical Hessian]\label{lem:local-hessian}
Let 
\[
H(\beta;Z_i) = \nabla_\beta^2 l(\beta; Z_i), \quad 
H(\beta) = \mathbb{E}[H(\beta; Z)], \quad 
H_n(\beta) = \frac{1}{n} \sum_{i=1}^n H(\beta; Z_i).
\]
Under Assumptions~\ref{as:design}--\ref{as:iid}, 
define the radius 
\[
r = r_n(\varepsilon_1) = \sqrt{\frac{4}{\lambda_x} \left( M_l \sqrt{\frac{2\big(p \log (3 R_B / \varepsilon_1) + \log (6/\delta)\big)}{n}} + 4 B_x \varepsilon_1 \right)}.
\]
Let $B(\beta_\theta, r)$ be the Euclidean ball centered at $\beta_\theta$ with radius $r$.  
Then there exists $\varepsilon_2 = \frac{\lambda_x}{32 B_x^3}$ and $c_1 = \frac{3 B_x^2}{8(12 B_x^2 + \lambda_x)} \in (0, 1/32)$ such that if 
\[
n \ge \frac{B_x^4}{c_1 \lambda_x^2} \left( p \log \frac{96 r B_x^3}{\lambda_x} + \log \frac{6p}{\delta} \right),
\]
we have with probability at least $1 - \delta/3$,
\[
\sup_{\beta \in B(\beta_\theta, r)} \|H_n(\beta) - H(\beta)\|_{\mathrm{op}} \le \varepsilon_H = \frac{\lambda_x}{2}.
\]
\end{lemma}

\begin{proof}
Let $\mathcal{N}_2$ be an $\varepsilon_2$-covering of the ball $B(\beta_\theta, r)$ with
\[
\log |\mathcal{N}_2| \le p \log \frac{3r}{\varepsilon_2}.
\]

For each grid point $\beta_j \in \mathcal{N}_2$, define
\[
Y_i^{(j)} = H(\beta_j; Z_i) - H(\beta_j), \quad i = 1, \dots, N.
\]
Then $\{Y_i^{(j)}\}_i$ are independent, mean-zero, self-adjoint matrices, and
\[
\|Y_i^{(j)}\|_{\mathrm{op}} \le 2 B_x^2, \quad
V_j := \Big\| \sum_{i=1}^n \mathbb{E}[(Y_i^{(j)})^2] \Big\|_{\mathrm{op}} \le 4 n B_x^4.
\]

By the matrix Bernstein inequality \citep{Mackey_2014}, for any $t>0$,
\[
\mathbb{P}\Big( \big\| \sum_{i=1}^n Y_i^{(j)} \big\|_{\mathrm{op}} \ge t \Big)
\le 2p \cdot \exp\Bigg( 
  - \frac{t^2/2}{V_j + (2 B_x^2 t)/3} 
\Bigg)
\le 2p \cdot \exp\Bigg( 
  - \frac{t^2}{8 n B_x^4 + \frac{4}{3} B_x^2 t} 
\Bigg).
\]

Dividing by $n$, for $s = t/n$,
\[
\mathbb{P}\Big( \| H_n(\beta_j) - H(\beta_j) \|_{\mathrm{op}} \ge s \Big)
\le 2p \cdot \exp\Bigg( 
  - \frac{N s^2}{8 B_x^4 + \frac{4}{3} B_x^2 s} 
\Bigg).
\]

Set $s = \varepsilon_H = \lambda_x / 2$, then for a fixed $\beta_j$,
\[
\mathbb{P}\Big( \| H_n(\beta_j) - H(\beta_j) \|_{\mathrm{op}} \ge \varepsilon_H \Big)
\le 2p \cdot \exp\Big( - \frac{c_1  n \lambda_x^2}{B_x^4} \Big),
\]
with 
\[
c_1 = \frac{3 B_x^2}{8(12 B_x^2 + \lambda_x)} \in (0, 1/32).
\]

Applying the union bound over all points in $\mathcal{N}_2$ and requiring the probability $\le \delta/3$ yields
\[
n \ge \frac{B_x^4}{c_1 \lambda_x^2} 
      \Big( p \log \frac{3r}{\varepsilon_2} + \log \frac{6p}{\delta} \Big).
\]

Finally, for any $\beta \in B(\beta_\theta, r)$, there exists $\beta_j \in \mathcal{N}_2$ with 
$\|\beta - \beta_j\|_2 \le \varepsilon_2$. Using the Hessian Lipschitz property (Lemma~\ref{lem:hessian-lip}),
\begin{align*}
\|H_n(\beta) - H(\beta)\|_{\mathrm{op}} 
&\le \|H_n(\beta) - H_n(\beta_j)\|_{\mathrm{op}} 
   + \|H_n(\beta_j) - H(\beta_j)\|_{\mathrm{op}} 
   + \|H(\beta_j) - H(\beta)\|_{\mathrm{op}} \\
&\le 8 B_x^3 \varepsilon_2 
   + \varepsilon_H 
   + 8 B_x^3 \varepsilon_2 \\
&= \varepsilon_H + 16 B_x^3 \varepsilon_2.
\end{align*}

Choosing $\varepsilon_2 \le \lambda_x / (32 B_x^3)$ ensures
\[
\sup_{\beta \in B(\beta_\theta, r)} \|H_n(\beta) - H(\beta)\|_{\mathrm{op}} \le \varepsilon_H
\]
with probability at least $1 - \delta/3$.
\end{proof}

\begin{lemma}[Empirical Hessian minimal eigenvalue lower bound]
\label{lem:hessian-min-eig}
Under the conditions of Lemma~\ref{lem:local-hessian}, for all $\beta \in B(\beta_\theta, r)$ with $r = r_n(\varepsilon_1)$, we have
\[
\lambda_{\min}(H_n(\beta)) 
\ge \lambda_{\min}(H(\beta)) - \|H_n(\beta) - H(\beta)\|_{\mathrm{op}}
\ge \lambda_x - \varepsilon_H
= \frac{\lambda_x}{2} > 0.
\]
Hence $H_N(\beta)$ is invertible in $B(\beta_\theta, r)$, and
\[
\|H_n^{-1}(\beta)\|_{\mathrm{op}} \le \frac{2}{\lambda_x}.
\]
\end{lemma}

\begin{lemma}[Parameter error for MNL MLE]\label{lem:beta-rate}
Under Assumptions~\ref{as:design}--\ref{as:sub-gaussian-gradient}, there exists a universal constant $C_1>0$ such that, for any $\delta\in(0,1)$, with probability at least $1-\delta$,
\[
\|\hat\beta-\beta^*\|_2
\;\le\;
\frac{C_1\,B_x}{\lambda_x}\,
\sqrt{\frac{p+\log(1/\delta)}{n}}.
\]
\end{lemma}

\noindent\emph{Remark.} Lemma~\ref{lem:beta-rate} provides a more explicit bound,
\(
\|\hat\beta-\beta^*\|_2
\le
\frac{8\sqrt{2}\,B_x}{\lambda_x}
\sqrt{\frac{p\log 5+\log(3/\delta)}{n}},
\)

\begin{proof}
Let $E_1 = \{\hat\beta \in B(\beta_\theta, r)\}$, $E_2 = \{\inf_{\beta\in B(\beta_\theta, r)} \lambda_{\min}(H_n(\beta)) \ge \lambda_x - \varepsilon_H \}$, and
\[
S_n := \frac{1}{n}\sum_{i=1}^n (g_i - \mathbb{E} g_i), \quad g_i = g(\beta_\theta; Z_i), 
\quad E_3 = \{ \|S_n\|_2 \le 4\sqrt{2} B_x \sqrt{\frac{p \log 5 + \log (3/\delta)}{n}} \}.
\]

On $E_1 \cap E_2$, the first-order optimality condition gives
\[
0 = \nabla_\beta \mathcal{L}_n(\hat\beta) 
= \nabla_\beta \mathcal{L}_n(\beta_\theta) + H_n(\tilde\beta)(\hat\beta - \beta_\theta),
\]
for some $\tilde\beta$ on the line segment $[\beta_\theta, \hat\beta] \subset B(\beta_\theta, r)$. Hence
\[
\hat\beta - \beta_\theta = - H_n^{-1}(\tilde\beta) \, \nabla_\beta \mathcal{L}_n(\beta_\theta),
\]
so that
\[
\|\hat\beta - \beta_\theta\|_2 \le \|H_n^{-1}(\tilde\beta)\|_{\mathrm{op}} \, \|\nabla_\beta \mathcal{L}_n(\beta_\theta)\|_2
\le \frac{2}{\lambda_x} \, \|\nabla_\beta \mathcal{L}_n(\beta_\theta)\|_2.
\]

To bound $\|\nabla_\beta \mathcal{L}_n(\beta_\theta)\|_2$, let $u \in S^{p-1}$ and consider the $\varepsilon$-net $\mathcal{N}_3$ with $\varepsilon = 1/2$, $|\mathcal{N}_3| \le 5^p$. Using sub-Gaussianity (Assumption~\ref{as:sub-gaussian-gradient}) and union bound, we get with probability at least $1-\delta/3$:
\[
\sup_{u \in \mathcal{N}_3} u^\top S_n \le \sqrt{\frac{8 \sigma^2 (p \log 5 + \log(3/\delta))}{n}}.
\]

By the standard $\varepsilon$-net argument,
\[
\|S_n\|_2 \le \frac{1}{1-\varepsilon} \sup_{v \in \mathcal{N}_3} v^\top S_n \le 2 \sup_{v \in \mathcal{N}_3} v^\top S_n
\le 4 \sqrt{2} B_x \sqrt{\frac{p \log 5 + \log(3/\delta)}{n}}.
\]

Finally, on $E_1 \cap E_2 \cap E_3$, we obtain the stated $\ell_2$ bound:
\[
\|\hat\beta - \beta_\theta\|_2 \le \frac{8 \sqrt{2} B_x}{\lambda_x} \sqrt{\frac{p \log 5 + \log(3/\delta)}{n}}.
\]

The sample size $n$ should satisfy the lower bound inherited from Lemma~\ref{lem:local-hessian} to ensure $E_1 \cap E_2$ holds.
\end{proof}

\paragraph{Proof for Theorem \ref{thm:bartau}}
\begin{proof}
Recall that for any offered set $S$, we define
\[
s(S;\beta)
=\log\!\left(\sum_{k\in S}\exp(\beta^{\top}x_k)\right).
\]
The gradient of $s(S;\beta)$ with respect to $\beta$ is
\[
\nabla_{\beta}s(S;\beta)
=\sum_{k\in S}\pi_k(S,\beta)\,x_k,
\quad
\pi_k(S,\beta)
:=\frac{\exp(\beta^{\top}x_k)}{\sum_{j\in S}\exp(\beta^{\top}x_j)}.
\]
Under Assumption~\ref{as:design}, we have
$\|\nabla_{\beta}s(S;\beta)\|_2
\le \max_{k\in S}\|x_k\|_2\le B_x$,
which implies that $s(S;\beta)$ is $B_x$-Lipschitz in~$\beta$.
Therefore, for any $S$,
\begin{equation}\label{eq:s-lip}
|s(S;\hat\beta)-s(S;\beta^{*})|
\le B_x\|\hat\beta-\beta^{*}\|_2.
\end{equation}

Squaring both sides of~\eqref{eq:s-lip} and averaging over $k=1,\ldots,n$
gives the deterministic inequality
\[
\tau
=\frac{1}{n}\sum_{k=1}^{n}(\hat s_k-s_k)^2
\le B_x^2\|\hat\beta-\beta^{*}\|_2^2.
\]
Applying Theorem~\ref{thm:crude-beta-error}, which provides that with
probability at least $1-\delta$,
\[
\|\hat\beta-\beta^{*}\|_2
\le
\frac{8\sqrt{2}\,B_x}{\lambda_x}
\sqrt{\frac{p\log5+\log(3/\delta)}{n}},
\]
and substituting this bound into the previous display yields
\[
\tau
\le
B_x^2
\left(\frac{8\sqrt{2}\,B_x}{\lambda_x}\right)^2
\frac{p\log5+\log(3/\delta)}{n}
=
\frac{128\,B_x^4}{\lambda_x^2}
\frac{p\log5+\log(3/\delta)}{n}.
\]
This establishes the desired finite-sample bound.
\end{proof}

\section{Proof of Lemma~\ref{lemma:real_index}}
\label{appx:proof_real_index}

\begin{proof}
Fix any $(X,\mathcal S)$. Under the choice model \eqref{eq:prob} with $i=0$,
\[
p_0(X,\mathcal S)
=\PP(I=0\mid X,\mathcal S)
=
\frac{\exp(u_0(X))}{\sum_{i\in\mathcal S}\exp(u_i(X))+\exp(u_0(X))}.
\]
Hence,
\[
1-p_0(X,\mathcal S)
=
\frac{\sum_{i\in\mathcal S}\exp(u_i(X))}{\sum_{i\in\mathcal S}\exp(u_i(X))+\exp(u_0(X))},
\]
so the odds ratio satisfies
\[
\frac{p_0(X,\mathcal S)}{1-p_0(X,\mathcal S)}
=
\frac{\exp(u_0(X))}{\sum_{i\in\mathcal S}\exp(u_i(X))}.
\]
Taking logarithms yields
\[
\eta(X,\mathcal S)
=\log\!\left(\frac{p_0(X,\mathcal S)}{1-p_0(X,\mathcal S)}\right)
=
u_0(X)-\log\!\left(\sum_{i\in\mathcal S}\exp(u_i(X))\right).
\]
Using the outside-option specification $u_0(X)=\gamma^{*\top}z(X)$ and the definition
$s(X,\mathcal S)=\log\!\left(\sum_{i\in\mathcal S}\exp(u_i(X))\right)$ proves
\[
\eta(X,\mathcal S)=\gamma^{*\top}z(X)-s(X,\mathcal S).
\]
\end{proof}

\section{Proofs for Section \ref{sec:lin_pred}}
\label{appx:proofs-lin-pred}

The concentration steps in this section follow standard finite-sample regression arguments for bounded/sub-Gaussian designs
(e.g., \citealp{vershynin2018,wainwright2019high}).
The paper-specific component is the explicit decomposition that separates the oracle OLS fluctuation from the perturbation
introduced by learning the inclusive value (i.e., using $\hat s_k$ in place of $s_k$) and the bias of the predictor, and then propagates these errors
through the coefficient-ratio map $\hat\gamma=\hat\theta_z/(-\hat\theta_s)$.

\paragraph{Notation.}
For each observation $(X_k,\mathcal{S}_k)$ recall
\[
y_k \;=\; \logit\!\big(\tilde p_0(X_k,\mathcal{S}_k)\big),\qquad
z_k \;=\; z(X_k)\in\R^d,\qquad
s_k \;=\; s(X_k,\mathcal{S}_k),\qquad
\hat s_k \;=\; \hat s(X_k,\mathcal{S}_k).
\]
We write the oracle regressor $w^*_k = \begin{bmatrix} z_k \\ s_k \end{bmatrix}\!\in\R^{d+1}$ and the working regressor
$w_k = \begin{bmatrix} z_k \\ \hat s_k \end{bmatrix}\!\in\R^{d+1}$.  Denote sample means
$\bar y=\frac1n\sum_{k=1}^n y_k$, $\bar w=\frac1n\sum_{k=1}^n w_k$, and centered variables
$\tilde y_k := y_k-\bar y$, $\tilde w_k := w_k-\bar w$.  And  $\tilde w_k^* := w_k^*-\bar w^*$ Let
\[
\Sigma_n \;=\; \frac1n\sum_{k=1}^n \tilde w_k\tilde w_k^\top
\quad\text{and}\quad
c_n \;:=\; \frac1n\sum_{k=1}^n \tilde w_k\tilde y_k,
\qquad
\hat\theta \;=\; \Sigma^{-1}c_n.
\]
Under Assumption~\ref{as:linear} the population model satisfies
\[
y_k \;=\; a^* \;+\; (b^*\gamma^{*})^\top z_k \;-\; b^* s_k \;+\; \epsilon_k
\;=\; a^* \;+\; (\theta_z^*)^\top z_k \;+\; \theta_s^* s_k \;+\; \epsilon_k,
\]
with $\theta_z^*:=b^*\gamma^*\in\R^d$ and $\theta_s^*:=-\,b^*\in\R$.  Throughout, we use
Assumption~\ref{as:w_k}(i) (bounded regressors, denoted by $B_w$) and
Assumption~\ref{as:w_k}(ii) (nondegenerate covariance, with $\lambda_{\min}(\Sigma)\ge \lambda_0>0$), as well as
Assumptions~\ref{as:non_zero} ($b^*\neq 0$) and~\ref{as:noise_sub} (sub-Gaussian noise with proxy $\sigma^2$).

\subsection{Proof of Theorem \ref{thm:ident-consistency}}
\begin{proof}
\emph{(i) Identification (oracle population problem).}
Consider the population least-squares objective
$L(\theta)=\E[(y-w^{*\top}\theta)^2]$, where $(w^*,y)$ has the joint distribution of
$(w^*_k,y_k)$. The normal equation is
$\E[w^* y]=\E[w^*w^{*\top}]\theta$. Using
$y=a^*+(\theta^*)^\top w^*+\epsilon$ and $\E[\epsilon\mid w^*]=0$,
\[
\E[w^* y] \;=\; \E[w^*w^{*\top}]\theta^*.
\]
Let $\Sigma_{\mathrm{pop}}:=\E[(w^*-\E w^*)(w^*-\E w^*)^\top]$.
By Assumption~\ref{as:w_k}(ii) (applied in the population limit), $\Sigma_{\mathrm{pop}}$ is positive definite, hence
$\theta^*$ is the unique minimizer. Since $(\theta_z^*,\theta_s^*)=(b^*\gamma^*,-\,b^*)$ with $b^*\neq 0$,
\[
\gamma^* \;=\; \frac{\theta_z^*}{-\theta_s^*}.
\]

\smallskip
\emph{(ii) Consistency with estimated $\hat s_k$.}
        Let $\hat{\bm{\beta}} = [\hat{a}, \hat{\theta}_z, \hat{\theta}_s]^\top$ denote the OLS estimator derived from the feasible regression. It can be written in matrix form as:
    \begin{equation}
        \hat{\bm{\beta}} = (\hat{\mathbf{X}}^\top \hat{\mathbf{X}})^{-1} \hat{\mathbf{X}}^\top \mathbf{y}
    \end{equation}
    Here, $\mathbf{y} = [y_1, \dots, y_n]^\top$ is the $n \times 1$ vector of observed outcomes. The matrix $\hat{\mathbf{X}}$ is the $n \times 3$ feasible design matrix, constructed by stacking the row vectors $\hat{\mathbf{x}}_k^\top$:
    \begin{equation}
        \hat{\mathbf{X}} = \begin{bmatrix} \hat{\mathbf{x}}_1^\top \\ \vdots \\ \hat{\mathbf{x}}_n^\top \end{bmatrix}, \quad \text{where } \hat{\mathbf{x}}_k = \begin{bmatrix} 1 \\ z_k \\ \hat{s}_k \end{bmatrix}
    \end{equation}
    is the regressor vector for the $k$-th observation containing the estimated term $\hat{s}_k$.

    \paragraph{Step 1: Convergence of the Design Matrix Moments.}
    Let $\mathbf{X}$ denote the \textit{true} design matrix, composed similarly of true regressor vectors $\mathbf{x}_k = [1, z_k, s_k]^\top$. Since we are given that $\hat{s}_k \xrightarrow{p} s_k$, it follows that $\hat{\mathbf{x}}_k \xrightarrow{p} \mathbf{x}_k$. Consequently, the difference between the feasible and true design matrices vanishes asymptotically:
    \begin{equation}
        \hat{\mathbf{X}} - \mathbf{X} \xrightarrow{p} \mathbf{0}
    \end{equation}
    Applying the Continuous Mapping Theorem (CMT) to the sample moment matrix, we obtain:
    \begin{equation}
        \frac{1}{n} \hat{\mathbf{X}}^\top \hat{\mathbf{X}} = \frac{1}{n} \sum_{k=1}^n \hat{\mathbf{x}}_k \hat{\mathbf{x}}_k^\top \xrightarrow{p}\mathbf{Q}\ :=\  \mathbb{E}[\mathbf{x}_k \mathbf{x}_k^\top]  
    \end{equation}
    where $\mathbf{Q}$ is the population second moment matrix of the true regressors. Note that the variance of the measurement error $\nu_k = \hat{s}_k - s_k$ implicitly vanishes in this limit because $\nu_k \xrightarrow{p} 0$. Similarly, for the cross-product term:
    \begin{equation}
        \frac{1}{n} \hat{\mathbf{X}}^\top \mathbf{y} = \frac{1}{n} \sum_{k=1}^n \hat{\mathbf{x}}_k y_k \xrightarrow{p} \mathbb{E}[\mathbf{x}_k y_k]
    \end{equation}
    
    \paragraph{Step 2: Consistency of Coefficients $\hat{\bm{\beta}}$.}
    By the Slutsky's Theorem, the OLS estimator converges to the product of the probability limits of its components:
    \begin{equation}
        \hat{\bm{\beta}} = \left( \frac{1}{n} \hat{\mathbf{X}}^\top \hat{\mathbf{X}} \right)^{-1} \left( \frac{1}{n} \hat{\mathbf{X}}^\top \mathbf{y} \right) \xrightarrow{p} \mathbf{Q}^{-1} \mathbb{E}[\mathbf{x}_k y_k]
    \end{equation}
    Substituting the true model $y_k = \mathbf{x}_k^\top \bm{\beta}_{true} + \epsilon_k$ (where $\bm{\beta}_{true} = [\alpha, \theta_z, \theta_s]^\top$) into the expectation:
    \begin{equation}
        \mathbf{Q}^{-1} \mathbb{E}[\mathbf{x}_k (\mathbf{x}_k^\top \bm{\beta}_{true} + \epsilon_k)] = \mathbf{Q}^{-1} (\mathbf{Q}\bm{\beta}_{true} + \mathbf{0}) = \bm{\beta}_{true}
    \end{equation}
    Thus, we establish the joint convergence of the coefficients:
    \begin{equation}
        \begin{pmatrix} \hat{\theta}_z \\ \hat{\theta}_s \end{pmatrix} \xrightarrow{p} \begin{pmatrix} \theta_z \\ \theta_s \end{pmatrix}
    \end{equation}
    
    \paragraph{Step 3: Consistency of $\hat{\gamma}$ via CMT.}
    The estimator of interest, $\hat{\gamma}$, is defined as a continuous function $g(\cdot)$ of the estimated coefficients:
    \begin{equation}
        \hat{\gamma} = g(\hat{\theta}_z, \hat{\theta}_s) = -\frac{\hat{\theta}_z}{\hat{\theta}_s}
    \end{equation}
    The function $g(a, b) = -b/a$ is continuous at all points where $a \neq 0$. Since $\theta_z \neq 0$ by assumption, and we have established that $\hat{\theta}_z \xrightarrow{p} \theta_z$ and $\hat{\theta}_s \xrightarrow{p} \theta_s$, the Continuous Mapping Theorem implies:
    \begin{equation}
        \hat{\gamma} = -\frac{\hat{\theta}_z}{\hat{\theta}_s} \xrightarrow{p} -\frac{\theta_z}{\theta_s} = \gamma
    \end{equation}
    This completes the proof.
\end{proof}

\subsection{Proof of Theorem \ref{thm:finite-sample}}
\begin{proof}
Start from the centered normal equation $c_n=\Sigma\hat\theta$ and decompose $c_n$ using the oracle model:
\[
\tilde y_k\;=\;(\theta^*)^\top\tilde w^*_k+\tilde\epsilon_k
\quad\Longrightarrow\quad
c_n \;=\; \frac1n\sum_{k=1}^n \tilde w_k\tilde y_k
\;=\; \frac1n\sum_{k=1}^n \tilde w_k\tilde w_k^{*\top}\theta^* \;+\; \underbrace{\frac1n\sum_{k=1}^n \tilde w_k\tilde\epsilon_k}_{=:~A}.
\]
Since $\tilde w^*_k-\tilde w_k=\bigl[0;\,s_k-\hat s_k\bigr]-\bigl(\bar w^*-\bar w\bigr)$ and
$\sum_k \tilde w_k=0$, we have
\[
\frac1n\sum_{k=1}^n \tilde w_k\tilde w_k^{*\top}\theta^*
\;=\; \frac1n\sum_{k=1}^n \tilde w_k\tilde w_k^\top\theta^* \;+\; \theta_s^* \underbrace{\frac1n\sum_{k=1}^n \tilde w_k\,(s_k-\hat s_k)}_{=:~B}.
\]
Hence
\begin{equation}\label{eq:basic-decomp}
\hat\theta-\theta^* \;=\; \Sigma^{-1}\!\Big(A + \theta_s^* B\Big).
\end{equation}

\textbf{Concentration for \(A\) and deterministic bound for \(B\).}
Conditional on $\{w_k\}$, the vector summands in $A$ are mean-zero and $\sigma^2$-sub-Gaussian with
$\psi_2$-norm at most $\sigma\|\tilde w_k\|_2/n$.  A standard vector sub-Gaussian tail bound (e.g., \citet{vershynin2018}, Thm.~6.21)
gives, for any $\delta\in(0,1)$, with probability at least $1-\delta$,
\begin{equation}\label{eq:A-bound}
\|A\|_2
\;\le\;
\sqrt{2}\,\sigma\,
\Big(\frac1n\sum_{k=1}^n \|\tilde w_k\|_2^2\Big)^{\!1/2}\,
\sqrt{\frac{\log\!\big(2(d+1)/\delta\big)}{n}}.
\end{equation}
By the variance-reduction identity
$\sum_k\|\tilde w_k\|_2^2=\sum_k\|w_k\|_2^2-n\|\bar w\|_2^2\le \sum_k\|w_k\|_2^2$
and Assumption~\ref{as:w_k}(i), we obtain
\begin{equation}\label{eq:A-bound-Bw}
\|A\|_2
\;\le\;
\sqrt{2}\,\sigma\,B_w\,\sqrt{\frac{\log\!\big(2(d+1)/\delta\big)}{n}}.
\end{equation}
For $B$, Cauchy-Schwarz and the same identity yield the deterministic bound
\begin{equation}\label{eq:B-bound}
\|B\|_2
\;\le\;
\Big(\frac1n\sum_{k=1}^n \|\tilde w_k\|_2^2\Big)^{\!1/2}
\Big(\frac1n\sum_{k=1}^n (s_k-\hat s_k)^2\Big)^{\!1/2}
\;\le\; B_w\,\sqrt{\bar\tau},
\end{equation}
where $\bar\tau:=\frac1n\sum_{k=1}^n (s_k-\hat s_k)^2$ (see the main text).

\textbf{Lower-bounding the denominator and controlling \(\|\hat\theta-\theta^*\|_2\).}
Assumption~\ref{as:w_k}(ii) gives $\|\Sigma^{-1}\|_{\op}\le 1/\lambda_0$.  Combining
\eqref{eq:basic-decomp}-\eqref{eq:B-bound} we have
\[
\|\hat\theta-\theta^*\|_2
\;\le\;
\frac1{\lambda_0}\Big(\|A\|_2 + |b^*|\,\|B\|_2\Big)
\;\le\;
\frac{B_w}{\lambda_0}\left(
\sqrt{2}\,\sigma\sqrt{\frac{\log\!\big(2(d+1)/\delta\big)}{n}}
\;+\;
|b^*|\,\sqrt{\bar\tau}
\right)
\]
with probability at least $1-\delta$.
Moreover, if
\begin{equation}
\sqrt{\bar\tau}\ \le\ \frac{\lambda_0}{8B_w}
\quad\text{and}\quad
n \ \ge\ \frac{128\,B_w^2\,\sigma^2}{\lambda_0^2\,{b^*}^2}\,
\log\!\Big(\frac{2(d+1)}{\delta}\Big),
\end{equation}
then $\|\hat\theta-\theta^*\|_2 \le \tfrac12 |\,\theta_s^*\,|$ and hence $|\hat\theta_s|\ge \tfrac12|b^*|$.

\textbf{From \(\hat\theta\) to \(\hat\gamma\).}
Since $\hat\gamma=\hat\theta_z/(-\hat\theta_s)$ and
$\gamma^*=\theta_z^*/(-\theta_s^*)$, the standard ratio bound yields (on $|\hat\theta_s|\ge \frac12|b^*|$)
\[
\|\hat\gamma-\gamma^*\|_2
\ \le\ \frac{2}{|b^*|}\Big(\|\hat\theta_z-\theta_z^*\|_2 + \|\gamma^*\|_2\,|\hat\theta_s-\theta_s^*|\Big)
\ \le\ \frac{2(1+\|\gamma^*\|_2)}{|b^*|}\,\|\hat\theta-\theta^*\|_2.
\]
Combining with the previous display and \eqref{eq:A-bound-Bw}-\eqref{eq:B-bound} proves
\begin{equation}
\|\hat\gamma-\gamma^*\|_2
\;\le\;
\frac{2\sqrt{2}\,\big(1+\|\gamma^*\|_2\big)\,B_w}{|b^*|\,\lambda_0}\,
\sigma\sqrt{\frac{\log\!\big(2(d+1)/\delta\big)}{n}}
\;+\;
\frac{2\big(1+\|\gamma^*\|_2\big)\,B_w}{\lambda_0}\,\sqrt{\bar\tau},
\end{equation}
as stated.
\end{proof}

\subsection{Proof of Corollary \ref{cor:p0_error}}
We first recall a basic fact.

\begin{lemma}[Logistic is $1/4$-Lipschitz]\label{lem:logistic-lip}
Let $\sigma(t):=(1+e^{-t})^{-1}$.  Then $|\sigma(u)-\sigma(v)|\le \tfrac14|u-v|$ for all $u,v\in\R$.
\end{lemma}

\begin{proof}
$\sigma'(t)=\sigma(t)(1-\sigma(t))\in(0,\tfrac14]$ with maximum $\tfrac14$ at $t=0$, so the mean-value theorem yields the claim.
\end{proof}

\begin{proof}[Proof of Corollary \ref{cor:p0_error}]
Write $t(X,\mathcal{S}):=\gamma^{*\top}z(X)-s(X,\mathcal{S})$ and
$\hat t(X,\mathcal{S}):=\hat\gamma^{\top}z(X)-\hat s(X,\mathcal{S})$.  Then
\[
\big|\hat p_0(X,\mathcal{S})-p_0(X,\mathcal{S})\big|
\;\le\; \tfrac14\,\big|\hat t(X,\mathcal{S})-t(X,\mathcal{S})\big|
\;=\; \tfrac14\,\big|(\hat\gamma-\gamma^*)^\top z(X)-(\hat s-s)\big|.
\]
The triangle and Cauchy-Schwarz inequalities give
\[
\big|\hat p_0(X,\mathcal{S})-p_0(X,\mathcal{S})\big|
\;\le\; \frac14\Big(\|\hat\gamma-\gamma^*\|_2\,\|z(X)\|_2+|\hat s(X,\mathcal{S})-s(X,\mathcal{S})|\Big),
\]
which is \eqref{eq:p0-Lip-explicit}; invoking \eqref{eq:gamma-fs-Bw} yields the stated high-probability bound in the corollary.
\end{proof}

\subsection{Refined Bound for Theorem \ref{thm:finite-sample} }
\begin{theorem}[High-Probability Bound for $\hat{\theta}$]
\label{thm:theta_hat_error}
Consider the regression model
\[
y_k = a + (\theta^*)^T \omega_k + \varepsilon_k, \quad k=1,\dots,n.
\]

With probability at least $1-\delta$,
\begin{equation}
\begin{aligned}
\|\hat{\theta} - \theta^*\|_2
\le\;&
\underbrace{\frac{2 B_w}{\lambda_0} \sigma 
\sqrt{\frac{1 + C \log (4/\delta) + C \sqrt{\log (4/\delta)}}{n}}}_{\text{oracle OLS term}} \\
&+
\underbrace{\frac{(4B_w \sqrt{\tau} + \tau)(\sqrt{\tau}+2B_w) + \lambda_0 \sqrt{\tau}}
{\lambda_0 (\lambda_0 - 4B_w \sqrt{\tau} - \tau) \sqrt{n}}
\, \|\tilde{y}\|_2}_{\text{perturbation due to }\hat{s}_k}.
\end{aligned}
\end{equation}
where $\tilde{y} = (y_1 - \bar{y}, \dots, y_n - \bar{y})^T$, and
\begin{equation}\label{eq:ytilde_bound}
\begin{aligned}
\|\tilde{y}\|_2^2 \le\;&
4 n B_w^2 \|\theta^*\|_2^2
+ \frac{4 \sqrt{n} B_w \|\theta^*\|_2}{\lambda_0} \sigma
\sqrt{1 + C \log(4/\delta) + C \sqrt{\log(4/\delta)}} \\
&+ \sigma^2 \bigl( n + C' \sqrt{n \log(4/\delta)} + C' \log(4/\delta) \bigr).
\end{aligned}
\end{equation}
holds with probability at least $1-\delta/2$.
\end{theorem}
The argument refines the decomposition in \eqref{eq:basic-decomp} and tracks the dependence on the
centering step and the perturbation $s_k-\hat s_k$ more sharply.

\begin{lemma}[Oracle OLS error]\label{lemma:oracle_OLS_appx}
Let $\theta^\circ := \big(\frac1n\sum_{k=1}^n \tilde w^*_k\tilde w_k^{*\top}\big)^{-1}\!
\big(\frac1n\sum_{k=1}^n \tilde w^*_k\tilde y_k\big)$ be the infeasible OLS estimator that uses the oracle regressors $w^*_k$.
There exist absolute constants $C,C'>0$ such that, for any $\delta\in(0,1)$, with probability at least $1-\delta/2$,
\[
\|\theta^\circ-\theta^*\|_2
\;\le\; \frac{2 B_w}{\lambda_0}\,\sigma\,
\sqrt{\frac{1 + C \log(4/\delta) + C \sqrt{\log(4/\delta)}}{n}}.
\]
\end{lemma}

\begin{proof}
We have $\tilde{y}_k = \theta^{*T} \tilde{\omega}_k + \tilde{\varepsilon}_k$, so
\[
c_n = \sum_{k=1}^n \tilde{\omega}_k \tilde{y}_k = \Sigma_n \theta^* + \sum_{k=1}^n \tilde{\omega}_k \tilde{\varepsilon}_k,
\]
and thus
\[
\theta^\circ - \theta^* = \Sigma_n^{-1} \sum_{k=1}^n \tilde{\omega}_k \tilde{\varepsilon}_k = \left( \frac{1}{n} \Sigma_n \right)^{-1} \left( \frac{1}{n} \sum_{k=1}^n \tilde{\omega}_k \tilde{\varepsilon}_k \right).
\]

Let $S = \frac{1}{n} \sum_{k=1}^n \tilde{\omega}_k \tilde{\varepsilon}_k$. Then
\[
\|\theta^\circ - \theta^*\|_2 \le \|\Sigma^{-1}\|_{op} \|S\|_2 \le \frac{1}{\lambda_0} \|S\|_2.
\]

To bound $S$, write $S = \frac{1}{n} W^T \tilde{\varepsilon}$, where $W$ has rows $\tilde{\omega}_k^T$ and $\tilde{\varepsilon} = (\tilde{\varepsilon}_1, \dots, \tilde{\varepsilon}_n)^T$. Let $P = I_n - \frac{1}{n} \mathbf{1}\mathbf{1}^T$ be the projection matrix, then $\tilde{\varepsilon} = P \varepsilon$ and
\[
\|S\|_2^2 = \frac{1}{n^2} \varepsilon^T (P W W^T P) \varepsilon.
\]

Apply Hanson-Wright inequality \citep{hanson1971bound} with $A = \frac{1}{n^2} P W W^T P$, using $\|W\|_F^2 \le 4 n B_w^2$ and $\|A\|_F \le \frac{4 B_w^2}{n}$, $\|A\|_{op} \le \frac{4 B_w^2}{n}$. The expected value $E[\varepsilon^T A \varepsilon] \le \frac{4 \sigma^2 B_w^2}{n}$. Then there exists a constant $C$ such that with probability at least $1-\delta/2$,
\[
\|S\|_2 \le \sigma \frac{2 B_w}{\sqrt{n}} \sqrt{1 + C \log(4/\delta) + C \sqrt{\log(4/\delta)}}.
\]

Combining gives the stated bound.
\end{proof}

\begin{lemma}[Decomposition of Perturbation Error]
\label{lemma:perturbation_decomp}
Let $\Delta \Sigma = \sum_{k=1}^n (\tilde{\omega}_k \tilde{\delta}_k^T + \tilde{\delta}_k \tilde{\omega}_k^T) + \sum_{k=1}^n \tilde{\delta}_k \tilde{\delta}_k^T$ and $\Delta_c = \sum_{k=1}^n \tilde{\delta}_k \tilde{y}_k$. Then the perturbation error satisfies
\[
\|\hat{\theta} - \theta\|_2 \le \|(\Sigma_n + \Delta\Sigma)^{-1}\|_{op} \|\Delta_c\|_2 + \|\Sigma_n^{-1}\|_{op} \|\Delta\Sigma\|_{op} \|(\Sigma_n + \Delta\Sigma)^{-1}\|_{op} \|c_n\|_2.
\]
\end{lemma}

\begin{proof}
Using the matrix identity $(\Sigma_n + \Delta\Sigma)^{-1} = \Sigma_n^{-1} - \Sigma_n^{-1} \Delta \Sigma (\Sigma_n + \Delta\Sigma)^{-1}$, we have
\[
\hat{\theta} - \theta = (\Sigma_n + \Delta\Sigma)^{-1} \Delta_c - \Sigma_n^{-1} \Delta \Sigma (\Sigma_n + \Delta\Sigma)^{-1} c_n,
\]
and the stated bound follows from the triangle inequality and sub-multiplicativity of norms.
\end{proof}

\begin{lemma}[High-Probability Bound for $\|\tilde{y}\|_2$]
\label{lemma:y_tilde_appx}
With probability at least $1-\delta/2$,
\begin{equation}\tag{\ref{eq:ytilde_bound}}
\begin{aligned}
\|\tilde{y}\|_2^2 \le\;&
4 n B_w^2 \|\theta^*\|_2^2
+ \frac{4 \sqrt{n} B_w \|\theta^*\|_2}{\lambda_0} \sigma
\sqrt{1 + C \log(4/\delta) + C \sqrt{\log(4/\delta)}} \\
&+ \sigma^2 \bigl( n + C' \sqrt{n \log(4/\delta)} + C' \log(4/\delta) \bigr).
\end{aligned}
\end{equation}
\end{lemma}

\begin{proof}
We have $\tilde{y}_k = \theta^{*T} \tilde{\omega}_k + \tilde{\varepsilon}_k$, so
\[
\|\tilde{y}\|_2^2 = \sum_{k=1}^n (\theta^{*T} \tilde{\omega}_k)^2 + 2 \sum_{k=1}^n \theta^{*T} \tilde{\omega}_k \tilde{\varepsilon}_k + \sum_{k=1}^n \tilde{\varepsilon}_k^2.
\]

First term: $\sum (\theta^{*T} \tilde{\omega}_k)^2 \le 4 n B_w^2 \|\theta^*\|_2^2$.  

Second term: $2 \sum \theta^{*T} \tilde{\omega}_k \tilde{\varepsilon}_k \le 2 n \|\theta^*\|_2 \|S\|_2$, bounded using Lemma~\ref{lemma:oracle_OLS_appx}.  

Third term: 
Since $\tilde{\varepsilon}_k = \varepsilon_k - \bar{\varepsilon}$ and $\mathbb{E}[\varepsilon_k]=0$ with i.i.d.\ sub-Gaussian noise,
applying the Hanson--Wright inequality with $A = I_n$ yields that, with probability at least $1-\delta/2$,
\begin{equation}
\sum_{k=1}^n \varepsilon_k^2 
\;\le\;
n\sigma^2 
+ C \sigma^2 \left( \sqrt{n \log(4/\delta)} + \log(4/\delta) \right).
\end{equation}

Combining the three terms gives the stated bound.
\end{proof}

\begin{lemma}[Norm Bounds for Perturbation Terms]
\label{lemma:norm_bounds}
The following bounds hold:
\[
\|\Sigma_n^{-1}\|_{op} \le \frac{1}{n \lambda_0}, \quad
\|(\Sigma_n + \Delta\Sigma)^{-1}\|_{op} \le \frac{1}{n(\lambda_0 - 4 B_w \sqrt{\tau} - \tau)},
\]
\[
\|\Delta_c\|_2 \le \sqrt{n\tau} \, \|\tilde{y}\|_2, \quad
\|\Delta\Sigma\|_{op} \le n (4 B_w \sqrt{\tau} + \tau), \quad
\|c_n\|_2 \le \sqrt{n} (\|\tilde{y}\|_2 + 2 B_w).
\]
\end{lemma}

\begin{proof}
All bounds follow from standard norm inequalities:

1. \(\|\Sigma_n^{-1}\|_{op} \le 1/(n \lambda_0)\) since \(\lambda_{\min}(\Sigma_n) \ge n \lambda_0\).  

2. \(\|(\Sigma_n + \Delta\Sigma)^{-1}\|_{op} \le 1/(n(\lambda_0 - 4 B_w \sqrt{\tau} - \tau))\) by Weyl's inequality, assuming \(4 B_w \sqrt{\tau} + \tau < \lambda_0\).  

3. \(\|\Delta_c\|_2 = \|\sum_{k=1}^n \tilde{\delta}_k \tilde{y}_k\|_2 \le \sum_{k=1}^n \|\tilde{\delta}_k\|_2 |\tilde{y}_k| \le \sqrt{n \sum \|\tilde{\delta}_k\|_2^2} \, \|\tilde{y}\|_2 = \sqrt{n \tau} \, \|\tilde{y}\|_2\).  

4. \(\|\Delta\Sigma\|_{op} \le \sum_{k=1}^n 2 \|\tilde{\omega}_k\|_2 \|\tilde{\delta}_k\|_2 + \sum_{k=1}^n \|\tilde{\delta}_k\|_2^2 \le n (4 B_w \sqrt{\tau} + \tau)\).  

5. \(\|c_n\|_2 = \|\sum_{k=1}^n \tilde{\omega}_k \tilde{y}_k\|_2 \le \sqrt{n} (\|\tilde{y}\|_2 + 2 B_w)\), using \(\|\tilde{\omega}_k\|_2 \le 2 B_w\).
\end{proof}

By combining the auxiliary results established in Lemmas~\ref{lemma:oracle_OLS_appx}, 
\ref{lemma:perturbation_decomp}, \ref{lemma:y_tilde_appx}, 
and \ref{lemma:norm_bounds}, we can now prove the high-probability bound for 
$\hat{\theta}$.

\begin{proof}[Proof of Theorem~\ref{thm:theta_hat_error}]

\textbf{Step 1: Decomposition.}
Decompose the total estimation error as
\[
\hat{\theta} - \theta^* = (\hat{\theta} - \theta) + (\theta - \theta^*),
\]
where $\theta$ is the oracle OLS estimator using true covariates.
Applying Lemma~\ref{lemma:oracle_OLS_appx} gives a high-probability bound on 
$\|\theta - \theta^*\|_2$.

\textbf{Step 2: Bounding the Perturbation Term.}
From Lemma~\ref{lemma:perturbation_decomp} together with 
Lemma~\ref{lemma:norm_bounds}, we can control
\[
\|\hat{\theta} - \theta\|_2 
\le 
\|(\Sigma_n + \Delta\Sigma)^{-1}\|_{op} \|\Delta_c\|_2
+ \|\Sigma_n^{-1}\|_{op} \|\Delta\Sigma\|_{op} 
\|(\Sigma_n + \Delta\Sigma)^{-1}\|_{op} \|c_n\|_2.
\]

\textbf{Step 3: Bounding $\|\tilde{y}\|_2$.}
Lemma~\ref{lemma:y_tilde_appx} provides a high-probability upper bound for $\|\tilde{y}\|_2$:
\begin{equation}\label{eq:ytilde_bound_repeat}
\begin{aligned}
\|\tilde{y}\|_2^2 \le\;&
4 n B_w^2 \|\theta^*\|_2^2
+ \frac{4 \sqrt{n} B_w \|\theta^*\|_2}{\lambda_0} \sigma
\sqrt{1 + C \log(4/\delta) + C \sqrt{\log(4/\delta)}} \\
&+ \sigma^2 \bigl( n + C' \sqrt{n \log(4/\delta)} + C' \log(4/\delta) \bigr).
\end{aligned}
\end{equation}
which holds with probability at least $1-\delta/2$.

\textbf{Step 4: Combining Bounds.}
From Lemmas~\ref{lemma:oracle_OLS_appx} and~\ref{lemma:norm_bounds}, we have
\[
\|\hat{\theta} - \theta^*\|_2 
\le 
\underbrace{\frac{2 B_w}{\lambda_0} \sigma 
\sqrt{\frac{ 1+ C \log(4/\delta) + C \sqrt{\log(4/\delta)}}{n}}}_{\text{Oracle OLS term}}
+ 
\underbrace{\frac{(4B_w \sqrt{\tau} + \tau)(\sqrt{\tau}+2B_w) + \lambda_0 \sqrt{\tau}}
{\lambda_0 (\lambda_0 - 4B_w\sqrt{\tau}-\tau) \sqrt{n}}}_{\text{Perturbation coeff.}}
\|\tilde{y}\|_2.
\]
Substituting~\eqref{eq:ytilde_bound_repeat} yields the final high-probability bound
stated in Theorem~\ref{thm:theta_hat_error}.

\end{proof}

\section{Proofs for Section~\ref{sec:nmbs}}
\label{app:nmbs}

The rank-based estimator analyzed in Section~\ref{sec:nmbs} is in the spirit of maximum rank correlation:
see \citet{han1987mrc} and subsequent analyses such as \citet{sherman1993mrc}.
In the proofs below, the uniform U-process deviation bound is a standard VC/U-process control
(\citealp{sherman1993mrc}; see also \citealp{hoeffding1963probability,vershynin2018} for the underlying concentration tools).
Our novelties are the population comparison inequality that incorporates the auxiliary predictor's mis-ordering
parameters $(\rho_0,\delta_0)$ (Theorem~\ref{thm:pop-approx}) and the plug-in stability bound that quantifies the effect of
replacing $s$ by $\hat s$ learned from purchase-only data (Proposition~\ref{prop:plugin}) and their further influence on the unobserved choice estimation.

\paragraph{Notation.}
Recall from Assumption~\ref{as:mnbs-formal} that
\[
W^*(X,\mathcal{S})=\big[z(X)^\top,\ s(X,\mathcal{S})\big]^\top\in\R^{d+1},\qquad
\theta^*=\big[\gamma^{*\top},-1\big]^\top,\qquad
\eta(X,\mathcal{S})=\theta^{*\top}W^*(X,\mathcal{S}).
\]
Let \((Y,W^*)\) and \((Y',W^{*\prime})\) be i.i.d.\ copies, and set
\[
D^*:=W^*-W^{*\prime},\qquad
\mathcal O:=\big\{\theta^{*\top}D^*\neq 0\big\},\qquad
S^*:=\operatorname{sign}\!\big(\theta^{*\top}D^*\big).
\]
For any \(\theta=[\gamma^\top,-1]^\top\), write \(S(\theta):=\operatorname{sign}(\theta^\top D^*)\).
Assumption~\ref{as:mnbs-formal} yields a measurable set of pairs \(\mathcal G\subseteq\mathcal O\) with
\(\Pr\big((W^*,W^{*\prime})\in\mathcal G\mid \mathcal O\big)\ge 1-\delta_0\) such that, for all \((W^*,W^{*\prime})\in\mathcal G\),
\[
\Pr\!\Big(\operatorname{sign}(Y-Y')\neq S^*\ \Big|\ W^*,W^{*\prime}\Big)\ \le\ \rho_0,
\qquad \rho_0\in\big[0,\tfrac12\big).
\]
When comparing \(\theta\) to \(\theta^*\), we also use \(h=[(\gamma-\gamma^*)^\top,0]^\top\) and abbreviate
\[
U:=\theta^{*\top}D^*,\qquad V_h:=h^\top D^*.
\]

\begin{lemma}[Sign flip implies small margin]\label{lem:key}
For any \(h\in\R^{d+1}\),
\[
\Big\{\operatorname{sign}(U)\neq \operatorname{sign}(U+V_h)\Big\}\ \subseteq\ \Big\{|U|\le |V_h|\Big\}.
\]
\end{lemma}

\subsection{Proof of Lemma~\ref{lem:key}}
\begin{proof}
If \(\operatorname{sign}(U)\neq \operatorname{sign}(U+V_h)\), then \(U(U+V_h)\le 0\), which implies \(U^2\le |U|\,|V_h|\) and hence \(|U|\le |V_h|\).
\end{proof}

\subsection{Proof of Theorem~\ref{thm:pop-approx}}
\begin{proof}[Proof of Theorem~\ref{thm:pop-approx}]
Let \(L=\operatorname{sign}(Y-Y')\). Conditioning on \((W^*,W^{*\prime})\),
\[
\mathrm{RC}(\theta)=\Pr\!\big(L=S(\theta)\big)=\E\!\Big[\Pr\!\big(L=S(\theta)\mid W^*,W^{*\prime}\big)\Big].
\]
Therefore
\[
\mathrm{RC}(\theta^*)-\mathrm{RC}(\theta)
=\E\!\Big[\underbrace{\Pr(L=S^*\mid W^*,W^{*\prime})-\Pr(L=S(\theta)\mid W^*,W^{*\prime})}_{\Delta(W^*,W^{*\prime})}\Big].
\]
If \(S(\theta)=S^*\), then \(\Delta=0\); if \(S(\theta)=-S^*\), then \(\Delta=1-2\,\Pr(L\neq S^*\mid W^*,W^{*\prime})\).
Split the expectation over \( (W^*,W^{*'})\in \mathcal G\) and its complement \(\mathcal G^c\):
on \(\mathcal G\), \(\Delta\ge (1-2\rho_0)\,\One\{S(\theta)\neq S^*\}\);
on \(\mathcal G^c\), \(\Delta\ge -\,\One\{S(\theta)\neq S^*\}\ge -1\).
Hence
\[
\mathrm{RC}(\theta^*)-\mathrm{RC}(\theta)
\ \ge\ (1-2\rho_0)\,\Pr\!\big(S(\theta)\neq S^* \mid \mathcal G\big)\ -\ \Pr(\mathcal G^c).
\]
By Assumption~\ref{as:mnbs-formal}, \(\Pr(\mathcal G^c)\le \delta_0\), which gives the claim.
\end{proof}

\subsection{Proof of Proposition~\ref{prop:margin}}

\begin{proof}
    \textbf{1. Definitions and Event Decomposition.}
    Let $U = \theta^{*\top}D^*$ and $V_h = (\theta - \theta^*)^\top D^*$. The event of sign change is $E_{\text{change}} = \{\operatorname{sign}(U) \neq \operatorname{sign}(U + V_h)\}$. By the central symmetry of $D^*$ given $\mathcal O$, the probability of sign change is:
    $$
    \Pr(E_{\text{change}} \mid \mathcal O) = 2 \Pr(U > 0, V_h < -U \mid \mathcal O).
    $$
    
    \textbf{2. Lower Bounding via Subset Inclusion.}
    We define the marginal slack $t_v$:
    $$
    t_v := \tau_{\mathrm{nd}} \|\gamma - \gamma^*\|_2.
    $$
    We define the subset events:
    $$
    A = \{0 < U \le t_v\}, \quad B = \{V_h \le -t_v\}.
    $$
    
    Since $\lVert \theta-\theta^*\rVert_2=\lVert[(\gamma-\gamma^*)^\top ,0]\rVert_2=\lVert\gamma-\gamma^*\rVert_2$, event $B= \{\frac{(\theta-\theta^*)^\top}{\lVert \theta-\theta^*\rVert_2}D^*\le -\frac{t_v}{\lVert \theta-\theta^*\rVert_2}\}=\{\frac{(\theta-\theta^*)^\top}{\lVert \theta-\theta^*\rVert_2}D^*\le -\tau_{nd}\}$. This leads to the lower bound:
    $$
    \Pr(E_{\text{change}} \mid \mathcal O) \ge 2 \Pr(A \cap B \mid \mathcal O).
    $$
    
    \textbf{3. Application of Association between design and truth (Assumption ~\ref{as:design-grouped}(iv)).}
    
    By assumption ~\ref{as:design-grouped} (iv), $\frac{\theta-\theta^*}{\lVert \theta-\theta^*\rVert_2}$ is a unit vector and $t_v=\tau_{nd}\lVert\gamma-\gamma^*\rVert_2\le \tau_{nd}\cdot r_0 \le t_0$, so the joint probability of the slab event $A$ and the tail event $B$ is bounded below by the product of their marginal probabilities:
    $$
    \Pr(A \cap B \mid \mathcal O) \ge c_{\text{assoc}} \cdot \Pr(A \mid \mathcal O) \cdot \Pr(B \mid \mathcal O).
    $$
    Substituting this into the inequality from Step 2:
    $$
    \Pr(E_{\text{change}} \mid \mathcal O) \ge 2 c_{\text{assoc}} \cdot \Pr(A \mid \mathcal O) \cdot \Pr(B \mid \mathcal O). \quad (*).
    $$
    
    \textbf{4. Bounding the Marginal Probabilities.}
    
    \textbf{Bounding $\Pr(A)$ (using Assumption ~\ref{as:design-grouped}(ii)):}
    By symmetry and since $t_v\le t_0$, Assumption 1(ii) yields:
    $$
    \Pr(A \mid \mathcal O) = \frac{1}{2}\Pr\left(\left\lvert\theta^{*\top}D^*\right\rvert \le t_v \mid \mathcal O\right) \ge \frac{1}{2} c_{\mathrm{lt}} \, t_v^{\alpha_{\mathrm{down}}}.
    $$
    
    \textbf{Bounding $\Pr(B)$ (using Assumption ~\ref{as:design-grouped}(i)):}
    Let $v_h$ be the unit vector in the direction of $\theta - \theta^*$. By symmetry and Assumption ~\ref{as:design-grouped}(i):
    $$
    \Pr(B \mid \mathcal O) = \frac{1}{2} \Pr\left(\left\lvert\frac{(\theta-\theta^*)^\top}{\lVert \theta-\theta^*\rVert_2}D^*\right\rvert \ge \tau_{\mathrm{nd}} \mid \mathcal O\right) \ge \frac{1}{2} p_{\mathrm{nd}}.
    $$
    
    \textbf{5. Final Assembly and Conclusion.}
    Substituting the bounds for $\Pr(A')$ and $\Pr(B'')$ back into equation $(*)$:
    $$
    \Pr(E_{\text{change}} \mid \mathcal O) \ge 2 c_{\text{assoc}} \cdot \left(\frac{1}{2} c_{\mathrm{lt}} \, t_v^{\alpha_{\mathrm{down}}}\right) \cdot \left(\frac{1}{2} p_{\mathrm{nd}}\right) = \frac{c_{\text{assoc}}}{2} c_{\mathrm{lt}} p_{\mathrm{nd}} \, t_v^{\alpha_{\mathrm{down}}}.
    $$
    Substituting the definition of $t_v$ and defining the constant $\kappa = \frac{c_{\text{assoc}}}{2} c_{\mathrm{lt}} p_{\mathrm{nd}} \tau_{nd}^{\alpha_{\mathrm{down}}} $:
    $$
    \Pr(E_{\text{change}} \mid \mathcal O) \ge \frac{c_{\text{assoc}}}{2} \cdot \left(p_{\mathrm{nd}} c_{\mathrm{lt}} \tau_{\mathrm{nd}}^{\alpha_{\mathrm{down}}}\right) \cdot \|\gamma - \gamma^*\|_2^{\alpha_{\mathrm{down}}} = \kappa \|\gamma-\gamma^*\|_2^{\,\alpha_{\mathrm{down}}}.
    $$
\end{proof}

\subsection{Proof of Proposition~\ref{prop:plugin}}

Below we first introduce a U-process bound
Theorem~\ref{thm:uprocess}, which is a standard uniform deviation bound for a bounded VC-type U-process. Related bounds appear in analyses of MRC and other U-statistic M-estimators (e.g., \citealp{sherman1993mrc}),
and can be obtained from Hoeffding-type concentration for U-statistics \citep{hoeffding1963probability}
combined with VC/Sauer counting arguments (see also \citealp{vershynin2018}).
We provide a self-contained proof for completeness.

\begin{theorem}[Uniform U-process control]\label{thm:uprocess}
There exists an absolute constant \(C_U>0\) such that, for any \(\delta\in(0,1)\), with probability at least \(1-\delta\),
\[
\sup_{\|\theta\|_2=1}\ \big|\widehat{\mathrm{RC}}_n(\theta)-\mathrm{RC}(\theta)\big|
\ \le\ C_U\,\sqrt{\frac{d+d\log(n/d)+\log\!\big(2/\delta\big)}{n}}.
\]
\end{theorem}

\begin{proof}[Proof of Theorem~\ref{thm:uprocess}]
Let us define the U-statistic
\[
U_n(\theta) := \frac{2}{n(n-1)} \sum_{1 \le k < l \le n} \mathbbm{1}_{\{ |\theta^T D_{kl}^*| \le 2\tau_s \}}, \quad D_{kl}^* = W_k^* - W_l^*,
\]
so that
\[
\sup_{\|\theta\|_2 = 1} \big| \widehat{\mathrm{RC}}_n(\theta) - \mathrm{RC}(\theta) \big| 
\le \sup_{\|\theta\|_2 = 1} \big| U_n(\theta) - \mathbb{E}[U_n(\theta)] \big| + \sup_{\|\theta\|_2 = 1} \mathbb{E}[U_n(\theta)].
\]

\textbf{Step 1: Control of the expectation.}  
By the model assumption (Assumption~\ref{as:design-grouped}(iii)), for any unit vector $\theta$,  
\[
\mathbb{E}[U_n(\theta)] = \frac{2}{n(n-1)} \sum_{k<l} P\Big( |\theta^T D_{kl}^*| \le 2\tau_s \,\big|\, \mathcal{O} \Big) \le \frac{2}{n(n-1)} \cdot \frac{n(n-1)}{2} \cdot L_{\mathrm{ac}} (2\tau_s)^{\alpha_{\mathrm{up}}} = L_{\mathrm{ac}} (2\tau_s)^{\alpha_{\mathrm{up}}}.
\]

\textbf{Step 2: Control of the centered U-process.}  
Consider the centered U-statistic
\[
X_\theta := U_n(\theta) - \mathbb{E}[U_n(\theta)].
\]

Each summand $\mathbbm{1}_{\{ |\theta^T D_{kl}^*| \le 2\tau_s \}}$ is bounded in $[0,1]$, and by the classical Hoeffding inequality for U-process \citep{hoeffding1963probability}, for any fixed $\theta$ and any $t>0$,
\[
\mathbb{P}\big( |X_\theta| \ge t \big) \le 2 \exp\Big( - 2 n_\mathrm{pairs} t^2 \Big),
\]
where $n_\mathrm{pairs} = \binom{n}{2}$ is the number of pairs $(k,l)$.  

\textbf{Step 3: Uniform control over $\theta$.}  
Let $\mathcal{F} := \{ \theta \mapsto \mathbbm{1}_{\{ |\theta^T D_{kl}^*| \le 2\tau_s \}} : \|\theta\|_2 = 1 \}$ denote the function class.  
The VC-dimension of $\mathcal{F}$ is at most $V \le 2(d+1)$, since these are thresholded linear functions in $\mathbb{R}^{d+1}$.  
By Sauer's lemma, the shattering number is
\[
N \le \left( \frac{e n(n-1)}{4(d+1)} \right)^{2(d+1)}.
\]

Applying a union bound over all possible labelings of the $n(n-1)/2$ pairs, we have for any $t>0$,
\[
\mathbb{P}\Big( \sup_{\theta} |X_\theta| \ge t \Big) \le 2 \left( \frac{e n(n-1)}{4(d+1)} \right)^{2(d+1)} \exp\Big( -2 n_\mathrm{pairs} t^2 \Big).
\]

\textbf{Step 4: Solve for $t$ given confidence level $\delta$.}  
Set the right-hand side equal to $\delta$ and solve for $t$:
\[
t \ge 2 \sqrt{ \frac{ (d+1) \log\left( \frac{e n(n-1)}{4(d+1)} \right) + \frac12 \log(2/\delta) }{ n } }.
\]

Hence, there exists an absolute constant $C_U > 0$ such that with probability at least $1-\delta$,
\[
\sup_{\|\theta\|_2 = 1} |U_n(\theta) - \mathbb{E}[U_n(\theta)]| \le C_U \sqrt{\frac{d+d\log(n/d) + \log(2/\delta)}{n}}.
\]

\end{proof}
\paragraph{Proof of Proposition~\ref{prop:plugin}}
\begin{proof}

We conduct the proof with 3 steps:

\textbf{Step 1: Control of plug-in bias.}  
Define the pairwise difference
\[
\delta_W := \theta^T (W_k - W_l) - \theta^T (W_k^* - W_l^*) = \theta_{d+1} (\hat{s}_k - s_k - \hat{s}_l + s_l),
\]
so that $|\delta_W| \le 2\tau_s$.  
Then
\[
|\mathbbm{1}_{\{(y_k-y_l)(\theta^T W_k - \theta^T W_l) > 0\}} - \mathbbm{1}_{\{(y_k-y_l)(\theta^T W_k^* - \theta^T W_l^*) > 0\}}|
\le \mathbbm{1}_{\{ |\theta^T (W_k^* - W_l^*)| \le 2 \tau_s \}}.
\]
Averaging over all pairs and taking the supremum over $\theta$, we obtain
\[
\sup_{\|\theta\|_2 = 1} |\hat{RC}_n(\theta) - \hat{RC}_n^*(\theta)| 
\le \sup_{\|\theta\|_2 = 1} \frac{2}{n(n-1)} \sum_{k<l} \mathbbm{1}_{\{ |\theta^T D_{kl}^*| \le 2\tau_s \}}.
\]

By Assumption~\ref{as:design-grouped}(iii),
\[
\sup_{\|\theta\|_2 = 1} \mathbb{E} \Big[ \frac{2}{n(n-1)} \sum_{k<l} \mathbbm{1}_{\{ |\theta^T D_{kl}^*| \le 2\tau_s \}} \Big] 
\le L_{\mathrm{ac}} (2\tau_s)^{\alpha_{\mathrm{up}}}.
\]

\textbf{Step 2: Control of the centered U-process.}  
Define the U-statistic
\[
U_n(\theta) := \frac{2}{n(n-1)} \sum_{k<l} \mathbbm{1}_{\{ |\theta^T D_{kl}^*| \le 2\tau_s \}}.
\]
Then
\[
\sup_{\|\theta\|_2=1} \Big| U_n(\theta) - \mathbb{E}[U_n(\theta)] \Big|
\]
is exactly the quantity controlled by Theorem~\ref{thm:uprocess}. Therefore, with probability at least $1-\delta$,
\[
\sup_{\|\theta\|_2=1} \Big| U_n(\theta) - \mathbb{E}[U_n(\theta)] \Big| \le C_U \sqrt{\frac{d + d\log(n/d) + \log(2/\delta)}{n}}.
\]

\textbf{Step 3: Combine the two bounds.}  
Combining Step 1 and Step 2, we have with probability at least $1-\delta$,
\[
\sup_{\|\theta\|_2=1} |\hat{RC}_n(\theta) - \hat{RC}_n^*(\theta)| 
\le L_{\mathrm{ac}} (2\tau_s)^{\alpha_{\mathrm{up}}} + C_U \sqrt{\frac{d + d\log(n/d) + \log(2/\delta)}{n}}.
\]
This completes the proof.
\end{proof}

\subsection{Proof of Theorem~\ref{thm:finite}}
\begin{proof}
We proceed this proof with Theorem~\ref{thm:uprocess} and Proposition~\ref{prop:plugin}:

\textbf{Step 1: Uniform empirical-to-population deviation.}
By Theorem~\ref{thm:uprocess} and Proposition~\ref{prop:plugin},
with probability at least $1-\delta$,
\[
\sup_{\|\theta\|_2=1}
\big|
\widehat{\mathrm{RC}}^{(\hat s)}_n(\theta)
- \mathrm{RC}(\theta)
\big|
\;\le\;
\varepsilon_n,
\]
where
\[
\varepsilon_n
:=
C_U\sqrt{\frac{
d + d\log(n/d) + \log(2/\delta)
}{n}}
\;+\;
L_{\mathrm{ac}}\,(2\tau_s)^{\alpha_{\mathrm{up}}}.
\]

\textbf{Step 2: Optimality of $\hat\theta$.}
By definition of $\hat\theta$,
\[
\widehat{\mathrm{RC}}^{(\hat s)}_n(\hat\theta)
\;\ge\;
\widehat{\mathrm{RC}}^{(\hat s)}_n(\theta^*).
\]
Using the uniform deviation bound on both sides,
\[
\mathrm{RC}(\hat\theta)
\;\ge\;
\widehat{\mathrm{RC}}^{(\hat s)}_n(\hat\theta)
-\varepsilon_n
\;\ge\;
\widehat{\mathrm{RC}}^{(\hat s)}_n(\theta^*)
-\varepsilon_n
\;\ge\;
\mathrm{RC}(\theta^*)-2\varepsilon_n.
\]
Equivalently,
\[
\mathrm{RC}(\theta^*)-\mathrm{RC}(\hat\theta)
\;\le\;
2\varepsilon_n.
\]

\textbf{Step 3: Apply population approximation + margin curvature.}
By Theorem~\ref{thm:pop-approx},
\[
\mathrm{RC}(\theta^*)-\mathrm{RC}(\theta)
\;\ge\;
(1-2\rho_0)\,
\mathrm{PC}(\theta^*,\theta)-\delta_0.
\]

By the margin condition (Proposition~\ref{prop:margin}), 
when $\|\hat\gamma-\gamma^*\|_2\le r_0$,
\[
\mathrm{PC}(\theta^*,\hat\theta)
\;\ge\;
\kappa\,\|\hat\gamma-\gamma^*\|_2^{\alpha_{\mathrm{down}}}.
\]

Combine the last two displays:
\[
\mathrm{RC}(\theta^*)-\mathrm{RC}(\hat\theta)
\;\ge\;
(1-2\rho_0)\,
\kappa\|\hat\gamma-\gamma^*\|_2^{\alpha_{\mathrm{down}}}
\;-\;\delta_0.
\]

Together with $\mathrm{RC}(\theta^*)-\mathrm{RC}(\hat\theta)\le 2\varepsilon_n$,
we obtain
\[
(1-2\rho_0)\,
\kappa\|\hat\gamma-\gamma^*\|_2^{\alpha_{\mathrm{down}}}
\;-\;
\delta_0
\;\le\;
2\varepsilon_n.
\]

Thus
\[
\|\hat\gamma-\gamma^*\|_2^{\alpha_{\mathrm{down}}}
\;\le\;
\frac{
2\varepsilon_n+\delta_0
}{
(1-2\rho_0)\,\kappa
}.
\]

Take the $\alpha_{\mathrm{down}}$-root to obtain
\[
\|\hat\gamma-\gamma^*\|_2
\;\le\;
\left(
\frac{
2\varepsilon_n+\delta_0
}{
(1-2\rho_0)\,\kappa
}
\right)^{\!1/\alpha_{\mathrm{down}}}.
\]

Substitute $\varepsilon_n$ to obtain \eqref{eq:gamma-fs}.

\textbf{Step 4: Ensuring $\|\hat\gamma-\gamma^*\|_2\le r_0$.}
We now make explicit the two feasibility conditions appearing implicitly above.

\textbf{(i) Nonnegativity of the RHS.}
We require
\[
2\varepsilon_n+\delta_0
<
(1-2\rho_0)\,\kappa r_0^{\alpha_{\mathrm{down}}},
\]
which after substituting $\varepsilon_n$ yields exactly condition
\eqref{eq:feasibility-gap} (first line).

\textbf{(ii) Enforcing RHS $\le r_0^{\alpha_{\mathrm{down}}}$.}
That is,
\[
\left(
\frac{
2\varepsilon_n+\delta_0
}{
(1-2\rho_0)\,\kappa
}
\right)
\;\le\;
r_0^{\alpha_{\mathrm{down}}}.
\]
Solving this gives exactly the same first line of \eqref{eq:feasibility-gap}.  
The remaining requirement is that the term
\[
2C_U\sqrt{\frac{d+d\log(n/d)+\log(2/\delta)}{n}}
\]
is small enough; explicitly:
\[
2C_U \sqrt{ \frac{d+d\log(n/d)+\log(2/\delta)}{n} }
\;\le\;
(1-2\rho_0)\,\kappa r_0^{\alpha_{\mathrm{down}}}
-
2L_{\mathrm{ac}}(2\tau_s)^{\alpha_{\mathrm{up}}}
-
\delta_0.
\]

Solving for $n$ produces the second inequality in 
\eqref{eq:feasibility-gap}.

If feasibility fails, the above argument still gives the bound 
\eqref{eq:gamma-fs}, conditional on the event 
$\{\|\hat\gamma-\gamma^*\|_2\le r_0\}$,
because the margin lower bound only requires this local condition.

\end{proof}

\subsection{Proof of Corollary~\ref{cor:p0_massart}}
\begin{proof}[Proof of Corollary~\ref{cor:p0_massart}]
By Corollary~\ref{cor:p0_error} (from the linear-bias section), the logistic map is \(1/4\)-Lipschitz, so
\[
\big|\hat p_0(X,\mathcal{S})-p_0(X,\mathcal{S})\big|
\ \le\ \frac14\Big(\|z(X)\|_2\,\|\hat\gamma-\gamma^*\|_2+|\hat s(X,\mathcal{S})-s(X,\mathcal{S})|\Big).
\]
Substitute the bound on \(\|\hat\gamma-\gamma^*\|_2\) furnished by Theorem~\ref{thm:finite}; the high-probability statement follows immediately (and likewise on the event \(\{\|\hat\gamma-\gamma^*\|_2\le r_0\}\) without feasibility).
\end{proof}

\subsection{Proof of Proposition~\ref{prop:perturb-epsmassart}}

\begin{proof}[Proof of Proposition~\ref{prop:perturb-epsmassart}]
For the notations,
\[
Y(X,\mathcal{S}) \;=\; (h_0+\Delta h)\big(\eta(X,\mathcal{S})\big)
   \;=\; \logit\big(\tilde p_0(X,\mathcal{S})\big),
\]
and note that
\[
\eta(X,\mathcal{S})
=\logit\big(p_0(X,\mathcal{S})\big)
   =\theta^{*\top}W^*.
\]
Thus for an independent pair \((X,\mathcal{S}),(X',\mathcal{S}')\),
\[
Y-Y'
   = (h_0+\Delta h)\big(\eta(X,\mathcal{S})\big)
     - (h_0+\Delta h)\big(\eta(X',\mathcal{S}')\big).
\]
Our goal is to compare the sign of \(Y-Y'\) with the sign of
\(\eta(X,\mathcal{S})-\eta(X',\mathcal{S}')=\theta^{*\top}(W^*-W^{*\prime})\).

\medskip\noindent
\textbf{Step 1: Control of the perturbation.}
Since \(\|\Delta h\|_\infty\le \Delta\), we have
\[
\big|\Delta h\big(\eta(X,\mathcal{S})\big)
 \;-\;
\Delta h\big(\eta(X',\mathcal{S}')\big)
\big|\;\le\; 2\Delta.
\]

\medskip\noindent
\textbf{Step 2: Using the definition of \(t_\Delta\).}
Assume now that
\[
\big|\theta^{*\top}D^*\big|
   = \big|\eta(X,\mathcal{S})-\eta(X',\mathcal{S}')\big|
   > t_\Delta.
\]
By the definition of the modulus of continuity and of \(t_\Delta\),
\[
\omega_{h_0}\big(|\theta^{*\top}D^*|\big)
   \;\ge\; \omega_{h_0}(t_\Delta)
   \;\ge\; 2\Delta
   \;\ge\; \Delta h(\eta)-\Delta h(\eta').
\]

\medskip\noindent
\textbf{Step 3: Comparing signs.}
We can rewrite
\[
Y-Y'
   = \big(h_0(\eta)-h_0(\eta')\big)
     \;-\;
     \big(\Delta h(\eta)-\Delta h(\eta')\big).
\]
Since the difference \(\Delta h(\eta)-\Delta h(\eta')\) is bounded in absolute value by \(2\Delta\), while
\[
h_0(\eta)-h_0(\eta')
   = h_0\big(\eta(X,\mathcal{S})\big)
     - h_0\big(\eta(X',\mathcal{S}')\big)
\]
has magnitude at least \(2\Delta\), the difference
\[
\big(h_0(\eta)-h_0(\eta')\big)
 -
\big(\Delta h(\eta)-\Delta h(\eta')\big)
\]
must have the same sign as \(h_0(\eta)-h_0(\eta')\), i.e., the perturbation is too small to flip the sign.

Since \(h_0\) is strictly increasing,  
\[
\operatorname{sign}\!\big(h_0(\eta)-h_0(\eta')\big)
   = \operatorname{sign}(\eta-\eta')
   = \operatorname{sign}\big(\theta^{*\top}(W^*-W^{*\prime})\big).
\]
By transitivity of sign equality, we conclude that
\[
\operatorname{sign}(Y-Y')
   = \operatorname{sign}\big(\theta^{*\top}(W^*-W^{*\prime})\big)
\quad\text{whenever}\quad
|\theta^{*\top}D^*|>t_\Delta.
\]
Hence,
\[
\Pr\!\Big( \operatorname{sign}(Y-Y')
      \neq \operatorname{sign}(\theta^{*\top}D^*)
      \,\Big|\, W^*,W^{*\prime}
   \Big)=0
\]
for all pairs satisfying \(|\theta^{*\top}D^*|>t_\Delta\).

\medskip\noindent
\textbf{Step 4: Mass bound.}
Let
\[
\mathcal G \;=\;
\big\{(W^*,W^{*\prime})\in\mathcal O:
      |\theta^{*\top}D^*|>t_\Delta
 \big\}.
\]
Then Assumption~\ref{as:mnbs-formal} holds with \(\rho_0=0\) and
\[
\delta_0
   = \Pr\big(|\theta^{*\top}D^*|\le t_\Delta\mid\mathcal O\big).
\]
By Assumption~\ref{as:design-grouped}(iii),
\[
\delta_0
 \;\le\;
 L_{\mathrm{ac}}\,
 t_\Delta^{\alpha_{\mathrm{up}}}.
\]

\medskip\noindent
\textbf{Step 5: Linear modulus case.}
If \(h_0'(\cdot)\ge m_0>0\) almost everywhere, then
\(
\omega_{h_0}(t)\ge m_0 t
\).
Thus from the definition of \(t_\Delta\),
\[
m_0 t \ge 2\Delta
\quad\Longrightarrow\quad
t_\Delta \le \frac{2\Delta}{m_0}.
\]

This completes the proof.
\end{proof}

\subsection{Proof of Corollary~\ref{cor:perturb-fs}}
\begin{proof}[Proof of Corollary~\ref{cor:perturb-fs}]
Combine Proposition~\ref{prop:perturb-epsmassart} with Theorem~\ref{thm:finite} (which applies with \(\rho_0=0\) and the above \(\delta_0\)).
\end{proof}

\section{Proof of Theorem \ref{thm:assortment-regret}}
\begin{proof}
Fix $X$ and an assortment $\mathcal{S}$.  Let
$\bar r(X,\mathcal{S}):=\sum_{i\in\mathcal{S}} r_i\,q_i(X,\mathcal{S})$ and
$\hat{\bar r}(X,\mathcal{S}):=\sum_{i\in\mathcal{S}} r_i\,\hat q_i(X,\mathcal{S})$.
By definition,
$R(X,\mathcal{S})=(1-p_0(X,\mathcal{S}))\,\bar r(X,\mathcal{S})$ and
$\hat R(X,\mathcal{S})=(1-\hat p_0(X,\mathcal{S}))\,\hat{\bar r}(X,\mathcal{S})$.
Thus
\[
|R-\hat R|
\;=\; \big|(1-p_0)\bar r-(1-\hat p_0)\hat{\bar r}\big|
\;\le\; \underbrace{|p_0-\hat p_0|\,\bar r}_{\text{calibration}}
\;+\; \underbrace{(1-\hat p_0)\,|\bar r-\hat{\bar r}|}_{\text{within-set shares}}
.
\]
Using $\bar r\le R_{\max}$ and $1-\hat p_0\le 1$, we obtain
\[
|R-\hat R| \;\le\; R_{\max}\,|p_0-\hat p_0| \;+\; \sum_{i\in\mathcal{S}} r_i\,|q_i-\hat q_i|
\;\le\; R_{\max}\,|p_0-\hat p_0| \;+\; R_{\max}\sum_{i\in\mathcal{S}} |q_i-\hat q_i|.
\]
Taking the maximum over $\mathcal{S}$ and applying the definition of $\varepsilon_s(X)$ and $\varepsilon_q(X)$ together with
Corollary~\ref{cor:p0_massart} (via its high-probability constant $C^{\mathrm{pert}}_n(\delta)$ from the main text) yields
\[
\max_{\mathcal{S}}\,|R(X,\mathcal{S})-\hat R(X,\mathcal{S})|
\ \le\
\frac{R_{\max}}{4}\Big(C^{\mathrm{pert}}_{n}(\delta)\,\|z(X)\|_2+\varepsilon_s(X)\Big) \;+\;  R_{\max}\,\varepsilon_q(X).
\]
Finally, optimality of $\mathcal{S}^*(X)\in\arg\max_{\mathcal{S}}R(X,\mathcal{S})$ and $\hat{\mathcal{S}}(X)\in\arg\max_{\mathcal{S}}\hat R(X,\mathcal{S})$ implies
\[
R\big(X,\mathcal{S}^*(X)\big)-R\big(X,\hat{\mathcal{S}}(X)\big)
\ \le\ 2\,\max_{\mathcal{S}}|R(X,\mathcal{S})-\hat R(X,\mathcal{S})|,
\]
which is bounded by \eqref{eq:rev-gap-bound-pert}.
\end{proof}
\section{Experiments' Details}
\label{sec:appendix_exp_setup}

This appendix provides the full experimental specifications for both the synthetic studies (Section~\ref{subsec:syn_experiments}) and the Expedia case study (Section~\ref{subsec:real_data_app}). 
We report default hyperparameters, exact data-generation steps, implementation details for the optimization routines, and the evaluation protocols.

\subsection{Synthetic Data Experiments}
\label{subsec:appendix_synthetic_exp}

\subsubsection{Synthetic Data Generation}
\label{subsec:synthetic_generation}

We generate i.i.d.\ samples indexed by $k=1,\dots,n$. Each sample consists of a context $X_k$, an assortment $\mathcal{S}_k$, the true inclusive value $s_k=s(X_k,\mathcal{S}_k)$, an estimated inclusive value $\hat s_k$, and a predictor-provided outside-option logit $y_k$. The full pipeline is illustrated in Figure~\ref{fig:synthetic_flow}.

\paragraph{Default parameters.}
Unless otherwise stated, we use the defaults in Table~\ref{tab:synthetic_params}. In particular, we set the dimension of the outside-option features to $d=24$, the item-feature dimension to $p=3$, and the default sample size to $n=2000$.

\begin{table}[htbp]
    \centering
    \caption{Default hyperparameters for synthetic data generation.}
    \label{tab:synthetic_params}
    \renewcommand{\arraystretch}{1.1}
    \small
    \begin{tabular}{@{}llp{8cm}@{}}
        \toprule
        \textbf{Category} & \textbf{Parameter} & \textbf{Value / Description} \\ 
        \midrule
        \textbf{Dimensions}     
            & $n$ (samples) & 2000 \\
            & Assortment size & $|\mathcal{S}_k|\sim\text{Unif}\{5,\ldots,15\}$ \\
            & Item pool size & 1000 \\
            & Context dim. & $d_{\text{ctx}}=24$ \\
            & Outside-feature dim. & $d=24$ \\
            & Item-feature dim. & $p=3$ \\
            & Item correlation & equicorrelated Gaussian with $\rho=0.5$ \\
        \midrule
        \textbf{Utility learning} 
            & Error mode & \texttt{structural} (perturb $\beta$) or \texttt{additive} (perturb $s$) \\
            & Noise scale & $\sigma_{\text{est}}=1.5$ \\
            & Noise distribution & Gaussian (default) or Uniform (bounded) \\
        \midrule
        \textbf{Predictor} 
            & Bias shift & $a^*=1.0$ \\
            & Bias scale & $b^*=2.0$ \\
            & Logit noise s.d. & $\sigma_{\epsilon}=0.2$ \\
        \bottomrule
    \end{tabular}
\end{table}

\paragraph{Step 1: Contexts, items, and assortments.}
We draw contexts and item covariates as follows.
\begin{itemize}
\item \textbf{Item pool.} We first generate a pool of 1000 item-feature vectors $\{x_i\in\mathbb{R}^p\}$ from a multivariate Gaussian with zero mean and equicorrelated covariance (pairwise correlation $\rho=0.5$). 
\item \textbf{Contexts.} For each sample $k$, we draw a context vector $X_k\in\mathbb{R}^{d_{\text{ctx}}}$ i.i.d.\ from a standard normal distribution.
\item \textbf{Assortments.} For each sample $k$, we draw an assortment size $|\mathcal{S}_k|\sim \text{Unif}\{5,\dots,15\}$ and sample $\mathcal{S}_k$ uniformly without replacement from the item pool.
\item \textbf{Truncation (boundedness).} To enforce bounded covariates consistent with our assumptions, we truncate the generated Gaussian features coordinate-wise to $[-3,3]$.
\end{itemize}

\paragraph{Step 2: Ground-truth utilities and logits.}
We next construct the true inside utilities, inclusive values, and outside logits.
\begin{itemize}
\item \textbf{Inside utilities and inclusive value.} We draw $\beta^*\sim\mathcal{N}(0,I_p)$ and set
\[
u_i(X_k) = \beta^{*\top}x_i,\qquad 
s_k = s(X_k,\mathcal{S}_k)=\log\!\sum_{i\in\mathcal{S}_k}\exp\!\big(u_i(X_k)\big).
\]
(Thus, inside utilities depend on item attributes but not on $X_k$; this choice isolates outside-option effects in the experiments.)
\item \textbf{Outside-option features.} We define $z(X_k)=W X_k\in\mathbb{R}^{d}$ where $W\in\mathbb{R}^{d\times d_{\text{ctx}}}$ is a random orthogonal matrix (rows orthonormal). 
\item \textbf{Outside parameters and true outside logit.} We draw $\gamma^{\mathrm{raw}}\sim\mathcal{N}(0,I_d)$ and scale
\[
\gamma^*=\gamma^{\mathrm{raw}}/\sqrt{d},
\qquad 
\eta_k=\gamma^{*\top}z(X_k)-s_k.
\]
The $\sqrt{d}$ scaling keeps $\mathrm{Var}(\gamma^{*\top}z(X_k))=O(1)$, avoiding degenerate logits concentrated near $\pm\infty$.
\end{itemize}

\paragraph{Step 3: Utility-estimation error and predictor output.}

\subparagraph{Inclusive-value estimation error.}
We construct $\hat s_k$ from $s_k$ using one of the following mechanisms.
\begin{itemize}
\item \textbf{Structural error (perturb $\beta$).} We draw $\Delta\beta\in\mathbb{R}^p$ with i.i.d.\ entries either $\mathcal{N}(0,\sigma_{\text{est}}^2)$ (Gaussian) or $\text{Unif}[-\sigma_{\text{est}}/\sqrt{p},\,\sigma_{\text{est}}/\sqrt{p}]$ (bounded). 
We set $\hat u_i(X_k)=(\beta^*+\Delta\beta)^\top x_i$ and compute
\[
\hat s_k=\log\!\sum_{i\in\mathcal{S}_k}\exp\!\big(\hat u_i(X_k)\big).
\]
\item \textbf{Additive error (perturb $s$).} We draw $\Delta s_k$ either from $\mathcal{N}(0,\sigma_{\text{est}}^2)$ or from $\text{Unif}[-\sigma_{\text{est}},\,\sigma_{\text{est}}]$ and set $\hat s_k=s_k+\Delta s_k$.
\end{itemize}
For reporting we compute both $\sqrt{\bar\tau}=\sqrt{\frac{1}{n}\sum_{k=1}^n(\hat s_k-s_k)^2}$ and $\tau_s=\max_k|\hat s_k-s_k|$; Figures~\ref{fig:linear_utility_noise_robustness}--\ref{fig:mrc_utility_noise_robustness} plot against the quantity relevant for the corresponding theory.

\subparagraph{Predictor logit.}
We generate the predictor-provided logit via
\[
y_k = h(\eta_k)+\epsilon_k,\qquad \epsilon_k\sim\mathcal{N}(0,\sigma_\epsilon^2),
\]
with one of the following links:
\begin{itemize}
\item \textbf{Linear predictor (Lin Sim):} $h(\eta)=a^*+b^*\eta$.
\item \textbf{Monotone non-linear predictor (Monotone Sim):}
\[
h(\eta)=a^*+b^*\Big[\tfrac{1}{20}\log\!\big(1+\exp(20\eta)\big)-\tfrac{1}{20}\log 2\Big],
\]
a centered Softplus map that is strictly increasing but saturates for negative $\eta$.
\end{itemize}
Finally, we convert back to predictor probabilities via $\tilde p_{0,k}=\Logistic(y_k)$ when needed (e.g., for multi-predictor probability averaging baselines).

\subsubsection{Performance Metrics}
\label{subsubsec:appendix_syn_metrics}

\paragraph{Empirical $p_0$ error.}
On a test set of size $N$, let $e_k=\big|p_0(X_k,\mathcal{S}_k)-\hat p_0(X_k,\mathcal{S}_k)\big|$. We report
\[
\text{Error}_{0.7} = \text{Quantile}_{0.7}\!\left(\{e_k\}_{k=1}^N\right).
\]

\paragraph{Revenue suboptimality.}
We generate $N_{\text{test}}=100$ independent decision instances. In each instance $j$, we draw a context and a pool of $m_{\text{cand}}=50$ items with revenues $\{r_i\}$ and solve (exactly, by revenue ordering for MNL; \citealp{talluri_vr_2004_ms}) the optimal assortment under the true model and under the plug-in estimate. We then compute the average relative revenue gap as in Section~\ref{subsubsec:syn_metrics}.

\begin{figure}[H] 
    \centering
    \includegraphics[width=0.95\textwidth, height=1.00\textheight, keepaspectratio]{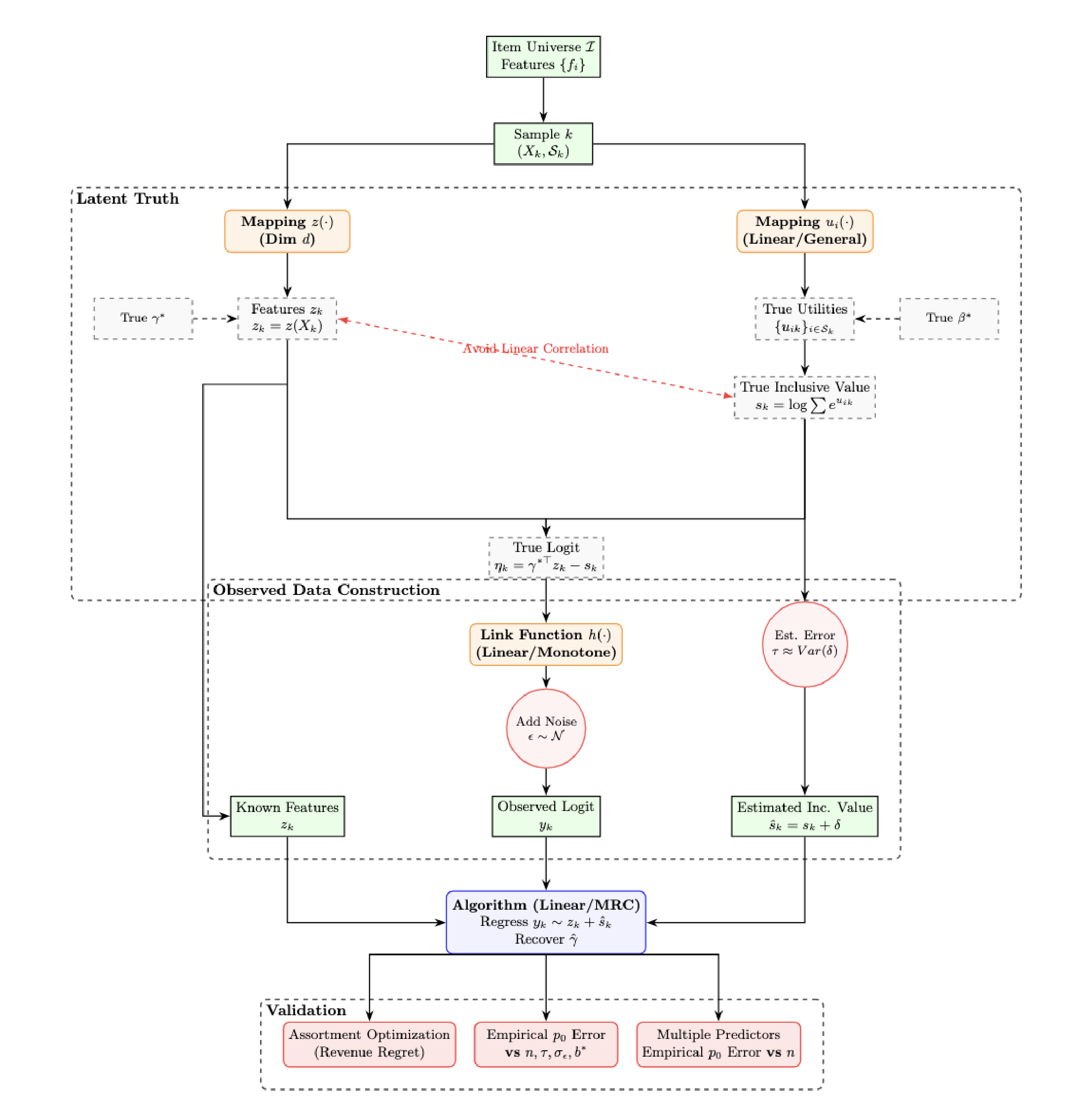}
    \caption{Synthetic data generation and evaluation protocol. We generate $(X_k,\mathcal{S}_k)$, construct $(s_k,\hat s_k)$, simulate biased logits $y_k$, and evaluate outside-option estimation accuracy and downstream decision impact.}
    \label{fig:synthetic_flow}
\end{figure}

\subsubsection{Benchmarks and Algorithms}
\label{appendix_syn_bench_algo}

\paragraph{Oracle baselines.}
\textbf{Linear\_oracle} and \textbf{MRC\_oracle} run Algorithms~\ref{alg:linear-calib} and~\ref{alg:mrc-calib} using the true inclusive values $\{s_k\}$ (i.e., $\hat s_k=s_k$). This isolates error due to finite samples and predictor noise.

\paragraph{Naïve multi-predictor baseline (Exp~6).}
\textbf{Logit Mean} averages predictor probabilities across predictors, $\bar p_{0,k}=\frac{1}{M}\sum_{m=1}^M \tilde p_{0,k}^{(m)}$, converts to logits $\bar y_k=\operatorname{logit}(\bar p_{0,k})$, and then applies Algorithm~\ref{alg:mrc-calib} using $\hat s_k$.

\subsubsection{Algorithm Implementation Details}

\paragraph{Algorithm~\ref{alg:linear-calib} (Linear calibration).}
We form the design matrix $A\in\mathbb{R}^{n\times(d+2)}$ with rows $[\,1,\ z_k^\top,\ \hat s_k\,]$ and solve the (ridge-stabilized) least-squares problem in double precision:
\[
\hat\theta \in \arg\min_\theta \|A\theta-y\|_2^2 + \lambda \|\theta_{1:}\|_2^2,
\]
where $\theta_{1:}$ excludes the intercept and $\lambda=10^{-9}$. We solve the normal equations using a Cholesky factorization of $A^\top A+\lambda I$.
We then recover $\hat\gamma=-\hat\theta_z/\hat\theta_s$ and apply a small threshold on $|\hat\theta_s|$ to avoid division by numerical zero.

\paragraph{Algorithm~\ref{alg:mrc-calib} (MRC calibration).}
Let $w_k=[\,z_k^\top,\ \hat s_k\,]^\top$ and define pairwise signs $\sigma_{ij}=\mathrm{sign}(y_i-y_j)$. We maximize sample rank correlation via a smooth surrogate objective:
\[
\min_{\theta\in\mathbb{R}^{d+1}:\ \|\theta\|_2=1}\ 
\mathcal{L}(\theta)
=
\frac{2}{n(n-1)}\sum_{i<j} \xi\!\Big(-\sigma_{ij}\,\theta^\top(w_i-w_j)\Big),
\qquad 
\xi(x)=\log(1+e^x),
\]
optimized using L-BFGS \citep{liu1989limited}. For large $n$ we subsample $2\times 10^5$ random pairs per gradient evaluation. We warm-start L-BFGS at the normalized OLS solution from Algorithm~\ref{alg:linear-calib}. Finally, we recover $\hat\gamma=-\hat\theta_z/\hat\theta_s$ with the same numerical safeguard on $\hat\theta_s$.

\subsubsection{Detailed Experimental Configurations (Exp~1--6)}

Table~\ref{tab:synthetic_exp_config} summarizes the experimental grid and the mapping to figures. Full parameter sequences are listed below the table.

\begin{table}[htbp]
    \centering
    \caption{Configuration of synthetic experiments.}
    \label{tab:synthetic_exp_config}
    \resizebox{\textwidth}{!}{
    \begin{tabular}{@{}clllll@{}}
        \toprule
        \textbf{Exp} & \textbf{Goal} & \textbf{Varied quantity} & \textbf{Fixed (default)} & \textbf{Seeds} & \textbf{Figures} \\ 
        \midrule
        \textbf{1} & Sample complexity & $n$ & $\sigma_{\text{est}}=1.5,\ \sigma_{\epsilon}=0.2$ & 10 (OLS), 30 (MRC) 
            & Fig.~\ref{fig:linear_convergence_rate}, Fig.~\ref{fig:mrc_convergence_rate} \\
        \textbf{2} & Utility-error sensitivity & $\sigma_{\text{est}}$ (reported as $\sqrt{\bar\tau}$ or $\tau_s$) & $n=2000,\ \sigma_{\epsilon}=0.2$ & 10 (OLS), 50 (MRC)
            & Fig.~\ref{fig:linear_utility_noise_robustness}, Fig.~\ref{fig:mrc_utility_noise_robustness} \\
        \textbf{3} & Predictor-noise robustness & $\sigma_{\epsilon}$ & $n=2000,\ \sigma_{\text{est}}=1.5$ & 10 
            & Fig.~\ref{fig:linear_robustness_to_simulator_noise}, Fig.~\ref{fig:mrc_robustness_to_simulator_noise} \\
        \textbf{4} & Bias-scale sensitivity & $b^*$ & $n=2000,\ \sigma_{\text{est}}=1.5,\ \sigma_{\epsilon}=0.2$ & 10 (OLS), 50 (MRC)
            & Fig.~\ref{fig:linear_b_sensitivity}, Fig.~\ref{fig:mrc_b_sensitivity} \\
        \textbf{5} & Revenue Suboptimality & $n$ & $\sigma_{\text{est}}=1.5,\ \sigma_{\epsilon}=0.2$ & 10 (OLS), 30 (MRC)
            & Fig.~\ref{fig:linear_revenue_regret}, Fig.~\ref{fig:mrc_revenue_regret} \\
        \textbf{6} & Multi-predictor aggregation & $n$ & five heterogeneous predictors & 30
            & Fig.~\ref{fig:multi-sim_aggregation} \\
        \bottomrule
    \end{tabular}
    }
\end{table}

\paragraph{Parameter grids.}
\begin{itemize}
\item \textbf{Exp~1/5/6:} $n\in\{200, 500, 1000, 2000, 3000, 4000, 5000, 6000, 8000\}$.
\item \textbf{Exp~2:} $\sigma_{\text{est}}\in\{0.0,0.1,0.2,\ldots,1.0\}$.
\item \textbf{Exp~3:} $\sigma_{\epsilon}\in\{0.1,0.5,1.0,2.0,3.0,5.0\}$.
\item \textbf{Exp~4:} $b^*\in\{0.1,0.5,1.0,1.5,2.0,3.0,5.0\}$.
\end{itemize}

\paragraph{Multi-predictor specification (Exp~6).}
We generate five predictor logits using the centered Softplus $h(\eta)=\tfrac{1}{20}\log(1+e^{20\eta})-\tfrac{1}{20}\log 2$:
\begin{align}
    y^{(1)} &= b^* h(\eta) + a^* + \epsilon^{(1)}, & \epsilon^{(1)} &\sim \mathcal{N}(0, 0.5^2) \\
    y^{(2)} &= \tfrac{1}{2} b^* h(\eta) + a^* + 0.5 + \epsilon^{(2)}, & \epsilon^{(2)} &\sim \mathcal{N}(0, 0.5^2) \\
    y^{(3)} &= \tfrac{1}{2} b^* h(\eta) + a^* - 0.5 + \epsilon^{(3)}, & \epsilon^{(3)} &\sim \mathcal{N}(0, 0.5^2) \\
    y^{(4)} &= b^* h(\eta) + a^* + 1.0 + \epsilon^{(4)}, & \epsilon^{(4)} &\sim \mathcal{N}(0, 1.0^2) \\
    y^{(5)} &= -2.5\, h(\eta) + \epsilon^{(5)}, & \epsilon^{(5)} &\sim \mathcal{N}(0, 2.5^2)\quad\text{(adversarial)}
\end{align}
Predictor (5) is anti-monotone and high-noise, serving as a stress test for aggregation.

\subsubsection{Justification for Increased Seed Count ($n_{\text{seeds}}=30, 50$)}
\label{subsubsec:seeds_justification}

We use $n_{\text{seeds}}=10$ for OLS-based experiments, but increase to 30--50 for MRC-based curves and discrete downstream objectives. The primary reason is that rank-based losses and assortment decisions can introduce higher Monte Carlo variance:
\begin{itemize}
\item \textbf{Rank sensitivity.} When $\hat s$ is noisy (Exp~2) or the signal-to-noise ratio is small (small $b^*$ in Exp~4), small perturbations can flip pairwise orderings, causing noticeable variability across seeds.
\item \textbf{Discrete decision effects.} Revenue suboptimality (Exp~5) involves an argmax over assortments; small parameter perturbations may change the selected set and induce non-smooth changes in regret.
\item \textbf{Non-convex aggregation.} In Exp~6, jointly learning orientations/weights across predictors can occasionally converge to poorer local solutions at small $n$.
\end{itemize}
Using more seeds yields smoother estimates of mean trends.

\subsection{Real Data Experiments}
\label{subsec:appendix_real_exp}

\subsubsection{Preprocessing and Feature Engineering}

We use the Expedia Personalized Sort dataset \citep{expedia-personalized-sort}. 
Each \emph{search session} (identified by \texttt{srch\_id}) presents a set of hotel impressions. 
We define the \emph{outside-option label} for a session as $b_k=1$ if the session has no booking and $b_k=0$ otherwise.

\paragraph{Temporal split and leakage prevention.}
We split the data chronologically into a Historical set (first 70\%) and a Current set (last 30\%). 
All feature transformations that use global statistics (e.g., hotel-level aggregates and scaling parameters) are fit on the Historical set and then applied to the Current set.

\paragraph{Missing values and transformations.}
We impute missing numerical values using Historical-set statistics (column means), except for \texttt{prop\_location\_score2} where missing values are set to $0$ to reflect ``no observed advantage''.
We log-transform price via $x\leftarrow \log(1+\max\{0,x\})$ to reduce skew.

\paragraph{Hotel-level aggregate features (Historical only).}
For each \texttt{prop\_id} we compute mean, standard deviation, and count over five attributes (\texttt{price\_usd}, \texttt{prop\_starrating}, \texttt{prop\_review\_score}, \texttt{prop\_location\_score1}, \texttt{prop\_location\_score2}) on the Historical set, producing 15 aggregate features. 
These aggregates are merged into the Current set; properties missing in history are imputed using Historical global means (and count $0$).

\paragraph{Scaling and encoding.}
We z-score all numerical features using the Historical-set mean and variance. 
Categorical/high-cardinality identifiers are integer-encoded for CatBoost.

\paragraph{Context vs.\ item features.}
We partition covariates into:
\begin{itemize}
\item \textbf{Context features $z(X)$} (session-level): user and search intent variables such as \texttt{site\_id}, \texttt{visitor\_location\_country\_id}, \texttt{length\_of\_stay}, \texttt{booking\_window}, \texttt{adults\_count}, \texttt{children\_count}, \texttt{room\_count}, \texttt{saturday\_night\_bool}, and \texttt{random\_bool}.
\item \textbf{Item features} (hotel-level): raw hotel attributes (e.g., \texttt{prop\_starrating}, \texttt{prop\_review\_score}, location scores, (log) price, promotion flags) together with the 15 Historical aggregates above.
\end{itemize}

\subsubsection{Hyperparameter Configuration and Training Protocols}
\label{subsec:appendix_real_hyperparams}

\paragraph{Biased predictor (CatBoost).}
We train a CatBoost classifier \citep{prokhorenkova2018catboost} on the Historical set to predict item-level bookings. We use:
\begin{itemize}
\item 500 trees, max depth 6, learning rate 0.1;
\item early stopping with a 20\% Historical validation split and patience 50.
\end{itemize}

\paragraph{Utility learning (Current set).}
We fit inside-utility models on the Current set using only sessions with an observed booking (conditional MNL likelihood). 
We use the Adam optimizer \citep{kingma2014adam} with batch size 4096.
\begin{itemize}
\item \textbf{Linear utility:} learning rate 0.05, 30 epochs.
\item \textbf{Neural utility (MLP):} 2-layer MLP with hidden size 64 and ReLU activations, learning rate 0.001, 50 epochs.
\end{itemize}

\begin{table}[htbp]
    \centering
    \caption{Hyperparameter settings for the Expedia experiments.}
    \label{tab:hyperparams}
    \begin{tabular}{l l r}
        \toprule
        \textbf{Model Component} & \textbf{Parameter} & \textbf{Value} \\
        \midrule
        \textit{Biased Predictor (CatBoost)} & Iterations & 500 \\
                                             & Tree Depth & 6 \\
                                             & Learning Rate & 0.1 \\
                                             & Early Stopping Rounds & 50 \\
        \midrule
        \textit{Utility Learner (Common)}    & Optimizer & Adam \\
                                             & Batch Size & 4096 \\
        \midrule
        \textit{Linear Utility Variant}      & Learning Rate & 0.05 \\
                                             & Training Epochs & 30 \\
        \midrule
        \textit{Neural Utility Variant}      & Architecture & 2-Layer MLP (Hidden=64) \\
                                             & Learning Rate & 0.001 \\
                                             & Training Epochs & 50 \\
        \bottomrule
    \end{tabular}
\end{table}

\subsubsection{Detailed Experimental Workflow}
\label{subsubsec:appendix_real_workflow}

We follow a three-phase pipeline that mirrors a realistic deployment under covariate shift.

\paragraph{Phase 1: Predictor construction (Historical).}
Train CatBoost on Historical data to obtain item-level booking predictions $\tilde p_i(X_k)$ for items $i\in\mathcal{S}_k$. 
We form a session-level prediction for the no-purchase probability by naïve aggregation:
\[
\tilde p_0(X_k,\mathcal{S}_k)
=
\mathrm{clip}\!\left(1-\sum_{i\in\mathcal{S}_k}\tilde p_i(X_k),\ \varepsilon,\ 1-\varepsilon\right),
\qquad \varepsilon=10^{-7},
\]
and set $y_k=\operatorname{logit}(\tilde p_0(X_k,\mathcal{S}_k))$.

\paragraph{Phase 2: Utility learning and calibration (Current).}
On the Current set, we emulate a transaction-only environment by restricting estimation to \emph{booked sessions}:
\begin{enumerate}
\item \textbf{Inside utilities.} Fit a linear model or MLP to maximize the conditional likelihood of the booked item given the session, using only booked sessions. Using the fitted model, compute $\hat s_k=\log\sum_{i\in\mathcal{S}_k}\exp(\hat u_i(X_k))$ (we compute $\hat s_k$ for all sessions so that we can later produce $\hat p_{0,k}$ for every session).
\item \textbf{Outside calibration.} Run Algorithm~\ref{alg:linear-calib} or Algorithm~\ref{alg:mrc-calib} on $\{(z_k,\hat s_k,y_k)\}$ \emph{restricted to booked sessions} to recover $\hat\gamma$, and then form $\hat p_0(X,\mathcal{S})=\Logistic(\hat\gamma^\top z(X)-\hat s(X,\mathcal{S}))$ for all sessions.
\end{enumerate}
Importantly, Phase 2 does \emph{not} use the Current-period outside labels; sessions with no booking are treated as unobserved during estimation and are used only in Phase 3 for evaluation.

\paragraph{Phase 3: Evaluation (Current).}
Let $b_k=1$ indicate that session $k$ has no booking (outside-option chosen). We evaluate predictions $\hat p_{0,k}$ using:

\begin{itemize}
\item \textbf{Negative log-likelihood (NLL).} With clipping to $[\varepsilon,1-\varepsilon]$,
\[
\mathrm{NLL}
=
-\frac{1}{n}\sum_{k=1}^n\Big[b_k\log(\hat p_{0,k})+(1-b_k)\log(1-\hat p_{0,k})\Big].
\]
\item \textbf{Expected calibration error (ECE).} Following \citet{guo2017calibration}, we partition $[0,1]$ into $M_{\text{bin}}=10$ equal-width bins $B_1,\dots,B_{M_{\text{bin}}}$ and compute
\[
\mathrm{ECE}
=
\sum_{m=1}^{M_{\text{bin}}} \frac{|B_m|}{n}\,\Big|\mathrm{acc}(B_m)-\mathrm{conf}(B_m)\Big|,
\]
where $\mathrm{conf}(B_m)=\frac{1}{|B_m|}\sum_{k\in B_m}\hat p_{0,k}$ and $\mathrm{acc}(B_m)=\frac{1}{|B_m|}\sum_{k\in B_m} b_k$.
\item \textbf{Reliability diagram.} We plot $\mathrm{conf}(B_m)$ vs.\ $\mathrm{acc}(B_m)$; the diagonal corresponds to perfect calibration.
\end{itemize}

\subsubsection{Detailed Comparison Methods}
\label{subsubsec:appendix_real_methods}

We compare the following five methods.

\begin{itemize}
\item \textbf{Predictor (baseline).} Use the Historical CatBoost predictor and naïve aggregation to obtain $\tilde p_0(X,\mathcal{S})$ as above. This ignores substitution and is vulnerable to covariate shift.

\item \textbf{Linear (Lin-Util).} Learn inside utilities with a linear model on the Current set; calibrate $\gamma$ with Algorithm~\ref{alg:linear-calib}.

\item \textbf{Linear (NN-Util).} Learn inside utilities with an MLP; calibrate with Algorithm~\ref{alg:linear-calib}.

\item \textbf{MRC (Lin-Util).} Learn inside utilities with a linear model; calibrate with Algorithm~\ref{alg:mrc-calib}.

\item \textbf{MRC (NN-Util) (Ours).} Learn inside utilities with an MLP; calibrate with Algorithm~\ref{alg:mrc-calib}.
\end{itemize}

\end{document}